\documentclass{article}

\usepackage{amsmath,amsfonts,bm}

\def\Figref#1{Fig.~\ref{#1}}

\def\Secref#1{Sec.~\ref{#1}}

\def\eqref#1{equation~\ref{#1}}
\def\Eqref#1{Eq.~\ref{#1}}
\def\Tabref#1{Tab.~\ref{#1}}

\def\1{\bm{1}}

\DeclareMathAlphabet{\mathsfit}{\encodingdefault}{\sfdefault}{m}{sl}
\SetMathAlphabet{\mathsfit}{bold}{\encodingdefault}{\sfdefault}{bx}{n}

\usepackage[sort,authoryear,round]{natbib}
\usepackage{fullpage}
\usepackage[utf8]{inputenc} 
\usepackage[T1]{fontenc}    
\usepackage[table]{xcolor}

\usepackage{hyperref}       
\definecolor{myblue}{rgb}{0,0.2,0.8}
\hypersetup{ %
pdftitle={},
pdfauthor={},
pdfsubject={},
pdfkeywords={},
pdfborder=0 0 0,
pdfpagemode=UseNone,
colorlinks=true,
linkcolor=myblue,
citecolor=myblue,
filecolor=myblue,
urlcolor=myblue,
pdfview=FitH}
\usepackage{authblk}
\usepackage{url}            
\usepackage{booktabs}       
\usepackage{amsfonts}       
\usepackage{amsthm}
\usepackage{nicefrac}       
\usepackage{microtype}      

\usepackage{graphicx}
\usepackage{xspace}
\usepackage[varqu,varl,var0,scaled=0.97]{inconsolata}
\usepackage{caption}
\usepackage{overpic}
\usepackage{wrapfig}
\usepackage{tcolorbox}
\usepackage{enumitem}
\usepackage{threeparttable}
\usepackage{subfig}

\usepackage{tabularx}  
\usepackage{multirow} 
\usepackage{listings}
\usepackage{float}
\usepackage{soul}
\usepackage{algorithm}
\usepackage{algpseudocode}
\usepackage{multicol}

\usepackage{adjustbox}

\definecolor{lightblue}{RGB}{220,230,241}
\definecolor{dkgreen}{rgb}{0,0.6,0}
\definecolor{gray}{rgb}{0.5,0.5,0.5}
\definecolor{mauve}{rgb}{0.58,0,0.82}

\newcommand{\eg}{\emph{e.g.},\xspace}
\newcommand{\ie}{\emph{i.e.},\xspace}

\newcommand{\fold}{\texttt{FOLD}\xspace}
\newcommand{\magprune}{\texttt{MAG}\xspace}
\newcommand{\magone}{\texttt{MAG1}\xspace}
\newcommand{\magtwo}{\texttt{MAG2}\xspace}

\usepackage[capitalize,noabbrev]{cleveref}

\newtheorem{theorem}{Theorem}[section]

\title{\bf{Cut Less, Fold More: Model Compression\\ through the Lens of Projection Geometry}}

\author{Olga Saukh$^{\circ\text{,§}}$, Dong Wang$^\circ$, Haris Šikić$^\circ$, Yun Cheng$^\ast$, Lothar Thiele$^\star$}
\affil[]{$^\circ$Graz University of Technology, $^{\text{§}}$Complexity Science Hub, Austria \\
$^\ast$Swiss Data Science Center, $^\star$ETH Zurich, Switzerland
}
\affil[]{\texttt{\{saukh@, dong.wang@, haris.sikic@student.\}tugraz.at}, \\ \texttt{yun.cheng@sdsc.ethz.ch}, 
\texttt{thiele@tik.ee.ethz.ch} \\
}

\begin{document}
\date{}
\maketitle

\begin{abstract}
Compressing neural networks without retraining is vital for deployment at scale. We study calibration-free compression through the lens of projection geometry: structured pruning is an axis-aligned projection, whereas model folding performs a low-rank projection via weight clustering. We formalize both as orthogonal operators and show that, within a rank distance of one, folding provably yields smaller parameter reconstruction error, and under mild smoothness assumptions, smaller functional perturbations than pruning. At scale, we evaluate $>$1'000 checkpoints spanning ResNet18, PreActResNet18, ViT-B/32, and CLIP ViT-B/32 on CIFAR-10 and ImageNet-1K, covering diverse training hyperparameters (optimizers, learning rates, augmentations, regularization, sharpness-aware training), as well as multiple LLaMA-family 60M and 130M parameter models trained on C4. We show that folding typically achieves higher post-compression accuracy, with the largest gains at moderate–high compression. The gap narrows and occasionally reverses at specific training setups. Our results position folding as a geometry-aware, calibration-free alternative to pruning that is often superior in practice and principled in theory.
\end{abstract}

\section{Introduction}
Neural network compression is critical for deploying models in resource-constrained environments. Common approaches include quantization, which reduces the precision of weights and activations, and knowledge distillation, which transfers information from a large teacher model to a smaller student model. In this work, we focus on the class of calibration-free post-training structured compression methods that optimize the model architecture itself without access to training data. Among these, the most widely used is \emph{magnitude-based pruning}, which prunes tensor elements according to their magnitudes, using them as a proxy for their contribution to model accuracy~\citep{han2015learningweightsconnectionsefficient, mishra2021acceleratingsparsedeepneural, lu2023steplearningnmstructured, ding2024usmlitequantizationsparsityaware, bambhaniya2024progressivegradientflowrobust}. When combined with fine-tuning or a lightweight BatchNorm reset~\citep{saikumar2025signalcollapseoneshotpruning}, this approach achieves significant compression rates with negligible accuracy loss~\citep{kurtic2022optimalbertsurgeonscalable, sanh2020movementpruningadaptivesparsity}.  
In contrast, the recently introduced \emph{model folding} clusters similar weights and ties them together, providing an approximation of the original network~\citep{wang2025forget}. Both pruning and folding reduce parameter count but differ fundamentally: pruning removes weights entirely, while folding preserves them in merged representations.

In this work, we develop a unified theoretical and empirical framework to compare pruning and folding through the lens of \emph{orthogonal projections} in parameter space. We show that both compression methods can be viewed as projections onto lower-dimensional subspaces, but with crucial differences in geometry: pruning corresponds to axis-aligned coordinate projections, while folding projects onto cluster-structured subspaces that retain directional information.

At a high level, both pruning and folding compress the weights of a model. 
We show that for any pruned solution there exists a folded alternative that is \emph{almost} as small—using one extra component in the compressed representation—yet is strictly closer to the original weights (smaller Frobenius norm), which in turn bounds the change in the network function.
Intuitively, folding merges weight vectors with similar directions rather than zeroing coordinates, so the compressed model stays closer in behavior to the initial network. 

Empirically, we perform a comprehensive calibration-free study over $>$1'000 checkpoints spanning CNNs and ViTs on CIFAR-10 and ImageNet-1K, trained under diverse hyperparameter choices (optimizers, learning rates, augmentation, regularization, sharpness-aware training). We also train and process 18 LLaMA-family models with 60M and 130M parameters on C4, by varying learning rates, warmup lengths, and weight decay strength. After compression and also followed by lightweight and full fine-tuning, folding typically attains higher post-compression accuracy, with the largest gains at moderate to high compression. The margin narrows, and can occasionally reverse, at very low compression or under specific training setups, but the overall trend is consistent with our theoretical analysis. Our projection-based perspective opens new directions for designing compression methods that explicitly optimize for functional closeness. This paper makes the following contributions:
\begin{itemize}
\item We introduce a unified projection framework that casts pruning and folding as orthogonal projections onto, respectively, axis-aligned and cluster-structured subspaces. We prove that at a compression rank difference of one, folding achieves smaller parameter reconstruction error and tighter function-perturbation bounds under mild smoothness assumptions.
\item A large-scale evaluation across $>$1'000 checkpoints and diverse hyperparameters, covering CNNs and ViTs on CIFAR-10 and ImageNet-1K, as well as LLaMA-60M and LLaMA-130M on C4. In addition, we use post-compression lightweight LayerNorm reset for ViTs, or full-fine-tuning to show that the strong performance of folding is preserved in these settings.
\item We show that folding is a geometry-aware alternative that is often superior in practice, with clearly identified regimes (\eg moderate–high compression) where its advantage is most pronounced, and corner cases where the gap narrows.
\end{itemize}
We discuss related work in Appendix~\ref{appx:related}, however, the main text already positions pruning and folding within our projection framework and clarifies the novelty of our approach.

\section{Unified Framework for Pruning and Folding}
\label{sec:theory}
\subsection{Preliminaries and Definitions}

We consider a neural network with input \( x \in \mathbb{R}^d \). We assume ReLU activations and normalization layers (\eg BatchNorm or LayerNorm) are present. 

To develop the theoretical framework, we focus on compressing a single layer at a time. This layer has \( p \) inputs and \( m \) outputs with its parameters collected in matrix \( \mathbf{W} \in \mathbb{R}^{m \times p} \). A row \( w(i) \) of \( \mathbf{W} \) is denoted as the \(i\)th parameter vector with individual weights \( w(i,j) \). Since all other network parameters are treated as fixed, the network function can be expressed as \( f(x; \boldsymbol{W}) \), which is trained to minimize a loss function \( L(\mathbf{W}) \). 

We assume that the loss function\,\(L\) is Lipschitz continuous, \ie there exists a constant \(\kappa > 0\)\,such\,that
\begin{equation} \label{eq:lipschitz}
|L(\mathbf{W}_1) - L(\mathbf{W}_2)| 
\leq \kappa \, 
\lVert \mathbf{W}_1 - \mathbf{W}_2 \rVert_F
\end{equation}
for all admissible parameter matrices \(\mathbf{W}_1\) and \(\mathbf{W}_2\). The Frobenius norm of a matrix is defined as \( \lVert A \rVert_F = \sqrt{\sum_{i,j} |a_{ij}|^2} \), that is, the square root of the sum of the squares of its entries, or equivalently, the $\ell_2$-norm of the vectorized matrix. 
This Lipschitz condition controls the change in loss with respect to parameter perturbations.
 
\paragraph{Orthogonal Projection.} We formalize structured pruning and model folding as orthogonal projections in parameter space. A matrix \( \mathbf{C} \in \mathbb{R}^{m\times m} \) is an orthogonal projection if \( \mathbf{C} = \mathbf{C}^\top = \mathbf{C}^2 \), \ie it is symmetric and idempotent. Such projections map any parameter vector to its closest point (in the Euclidean norm) within a lower-dimensional subspace.

If the columns of \( \mathbf{U} \in \mathbb{R}^{m \times k} \) form a basis of a \( k \)-dimensional subspace,  the corresponding orthogonal projection is
\begin{equation}\label{eq:proj}
\mathbf{C} = \mathbf{U} (\mathbf{U}^\top \mathbf{U})^{-1} \mathbf{U}^\top.
\end{equation}
Equivalently,
\[
\mathbf{C} y = \underset{z \in \mathrm{Range}(\mathbf{U})}{\arg\min} \; \lVert y - z \rVert_2
\]
meaning \( \mathbf{C} y \) is the orthogonal projection of \( y \) onto the subspace spanned by \( \mathbf{U} \).

\subsection{Compression as Orthogonal Projection}

\paragraph{Structured pruning.} Pruning can be viewed as a projection onto a coordinate-aligned subspace at the level of neurons, filters, or channels. Assume the layer outputs are ordered so that the last \( m-k \) are pruned. The corresponding basis \( \mathbf{U}_p \) spans the \( k \)-dimensional subspace, with projection matrix \( \mathbf{C}_p \) and transformed weight matrix \( \mathbf{W}_p \):
\begin{equation}\label{eq:pruning}
\mathbf{U}_p = \begin{pmatrix} I \\ 0 \end{pmatrix}, \quad
\mathbf{C}_p = \begin{pmatrix} I & 0 \\ 0 & 0 \end{pmatrix}, \quad
\mathbf{W}_p = \mathbf{C}_p \mathbf{W}.
\end{equation}
Consequently, the last \( m-k \) rows of \( \mathbf{W}_p \) are zero, and the corresponding neurons, filters, or channels can be simply removed.

\paragraph{Model folding.} Folding groups the parameters into \( k \) clusters and replaces each cluster with its mean. Depending on the choice of clusters, a different folding results. Folding can be represented as an orthogonal projection onto the \( k \)-dimensional subspace spanned by \( \mathbf{U}_f \in \{0,1\}^{m\times k} \), where each row contains exactly one nonzero entry indicating the cluster assignment. A cluster \( S_j \) comprises all indices of parameter vectors belonging to it; thus, \( u_f(i,j) = 1 \) if and only if \( i \in S_j \).

The projection \( \mathbf{C}_f \) defined in \Eqref{eq:proj} maps each cluster to its mean~\citep{wang2025forget}. Specifically,  
\begin{equation}\label{eq:folding}
\mathbf{W}_f = \mathbf{C}_f \mathbf{W}, \quad
\forall i \in S_j : \; w_f(i) = \mu_j, \quad
\mu_j = \frac{1}{|S_j|} \sum_{i \in S_j} w(i),
\end{equation}
where \(\mu_j\) is the mean of cluster \(j\). After projection, all parameter vectors within a cluster are replaced by their mean, making them identical. As a result, the corresponding layer outputs are also identical, leaving a total of \(k\) distinct neurons, filters, or channels. Practically, the identical layer outputs can be joined while adapting the next layer appropriately, see \citep{wang2025forget}.

\subsection{Folding Dominates Pruning}
To compare pruning and folding, we first show that for any choice of pruning, there exists a folding that yields a more accurate approximation of the parameter matrix \( \mathbf{W} \).

\begin{theorem}
\label{thm:folding-existence}
Given any pruning with basis \( \mathbf{U}_p \) of rank \( 0 \leq k_p \leq m-1 \) (\ie at least one parameter vector is pruned), there exists a folding with basis \( \mathbf{U}'_f \) and rank \( k_f = k_p + 1 \) such that
\[
\lVert \mathbf{W} - \mathbf{W}_p \rVert_F^2 \geq \lVert \mathbf{W} - \mathbf{W}'_f \rVert_F^2,
\]
where \( \mathbf{W}_p = \mathbf{C}_p \mathbf{W} \) and \( \mathbf{W}'_f = \mathbf{C}'_f \mathbf{W} \), with \( \mathbf{C}_p \) and \( \mathbf{C}'_f \) denoting the orthogonal projections defined in~\Eqref{eq:proj}. 
\end{theorem}

In the above theorem, $\mathbf{U}'_f$ denotes the constructive clustering obtained by merging all pruned rows into a single additional cluster. The proof is in Appendix~\ref{appx:theory}. By the Lipschitz continuity of the loss function in \Eqref{eq:lipschitz}, the superior approximation property of folding implies a tighter bound on the loss difference compared to pruning:
\[
| L(\mathbf{W}) - L(\mathbf{W}'_f) | \leq \kappa \, \lVert \mathbf{W} - \mathbf{W}'_f \rVert_F,
\quad
| L(\mathbf{W}) - L(\mathbf{W}_p) | \leq \kappa \, \lVert \mathbf{W} - \mathbf{W}_p \rVert_F,
\]
with
\[
\lVert \mathbf{W} - \mathbf{W}'_f \rVert_F^2 \leq \lVert \mathbf{W} - \mathbf{W}_p \rVert_F^2.
\]
Furthermore, the rank difference \( k_f = k_p + 1 \) between pruning and folding is practically negligible, since in typical scenarios many parameter vectors are pruned. For instance, under a uniform \(50\%\) per-layer retention, a ResNet-18 stage with 256 channels keeps \(k_p=128\) (so folding uses \(k_f=129\)), and a ViT-B/32 block with width 768 keeps \(k_p=384\) (so \(k_f=385\)); the relative increase is just \(1/k_p \approx 0.78\%\) and \(0.26\%\), respectively—negligible in practice. Moreover, for all layers and architectures we observe that loss and accuracy vary smoothly as the rank increases from 
$k_p$ to $k_p + 1$, no jumps in loss or accuracy, and the error difference between ranks $k$ and $k + 1$ is typically much smaller than the difference between pruning and folding at the same rank (see Appendix~\ref{subsec:one_rank_slack}).

Finally, we show that folding using optimal \(k\)-means clustering never yields a less accurate approximation of the parameter matrix \( \mathbf{W} \) than pruning. 

\begin{theorem}
\label{thm:folding-optimal}
Let \( \mathbf{U}^\star_f \) be the basis obtained from an optimal \( k \)-means clustering with \( k_f \) clusters, \ie the folding clusters are determined by a \( k \)-means algorithm minimizing the accumulated within-cluster sum of squares. Then, for any pruning with basis \( \mathbf{U}_p \) of rank \( k_p = k_f - 1 \), we have
\[
\lVert \mathbf{W} - \mathbf{W}_p \rVert_F^2 \geq 
\lVert \mathbf{W} - \mathbf{W}^\star_f \rVert_F^2,
\]
where \( \mathbf{W}_p = \mathbf{C}_p \mathbf{W} \) and \( \mathbf{W}^\star_f = \mathbf{C}^\star_f \mathbf{W} \), with \( \mathbf{C}_p \) and \( \mathbf{C}^\star_f \) denoting the orthogonal projections defined in~\Eqref{eq:proj}.
\end{theorem}

The proof is given in Appendix~\ref{appx:theory}. This result demonstrates that \( k \)-means folding is not merely a heuristic, but an optimal projection under clustering constraints. Note that the special folding $\mathbf{W}'_f$ in Theorem~\ref{thm:folding-existence} is suboptimal, while Theorem~\ref{thm:folding-optimal} shows that $\mathbf{W}^\star_f$ achieves the minimum possible reconstruction error over all clusterings, producing a strictly stronger improvement as $\lVert \mathbf{W} - \mathbf{W}_p \rVert_F^2 \geq \lVert \mathbf{W} - \mathbf{W}'_f \rVert_F^2 \geq
\lVert \mathbf{W} - \mathbf{W}^\star_f \rVert_F^2$. Unlike pruning, which relies on parameter vector removal, folding generalizes the idea by enabling coordinated parameter merging. Thus, folding incurs less parameter distortion and provably smaller loss perturbation under a local parameter-Lipschitz assumption.

In addition, Theorem~\ref{thm:folding-optimal} has implications for a possible fine-tuning after compression. Matrix \( \mathbf{W} \) contains the optimized weights and \( \mathbf{W}_p \) or \( \mathbf{W}^\star_f \) contain the weights after pruning and folding the optimized network. As a result of Theorem~\ref{thm:folding-optimal}, the quadratic distance between the optimized weights and the compressed optimized weights is smaller for folding in comparison to pruning. 

Our theoretical results employ a one–rank slack comparing pruning at rank \( k_p \) to folding at \( k_f = k_p + 1 \), as a proof device to obtain a clean monotonicity guarantee on projection error. This slack does \emph{not} reflect our evaluation protocol. In all experiments we enforce matched sparsity budgets and compare methods at the \emph{same} retained size (parameters and FLOPs). Hence, empirical accuracy gaps cannot be attributed to extra capacity.

\begin{figure}[t]
     \centering
     \vskip -.4cm
     \subfloat[][ResNet18, Adam, \magone vs \fold, \textbf{no} L1 regularization]{
        \includegraphics[height=0.24\textwidth]{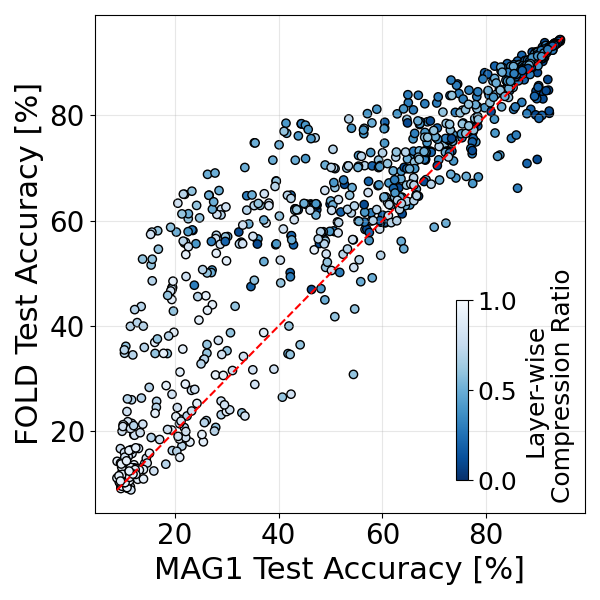}
        \includegraphics[height=0.24\textwidth]{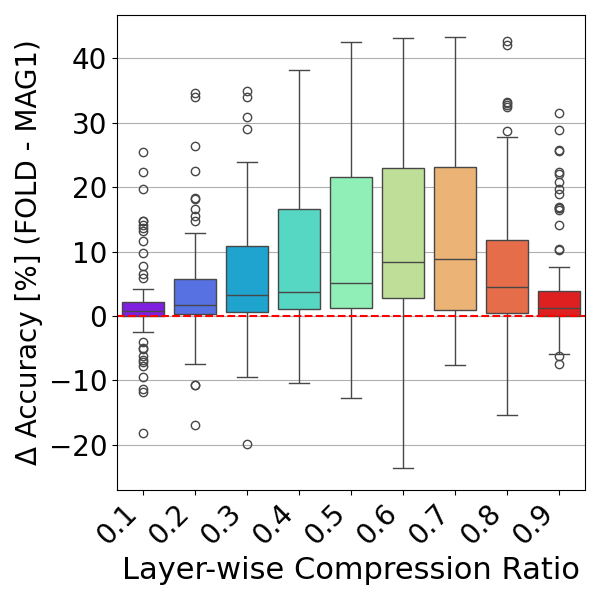}
     } 
     \subfloat[][PreActResNet18, \magone vs \fold]{
        \includegraphics[height=0.24\textwidth]{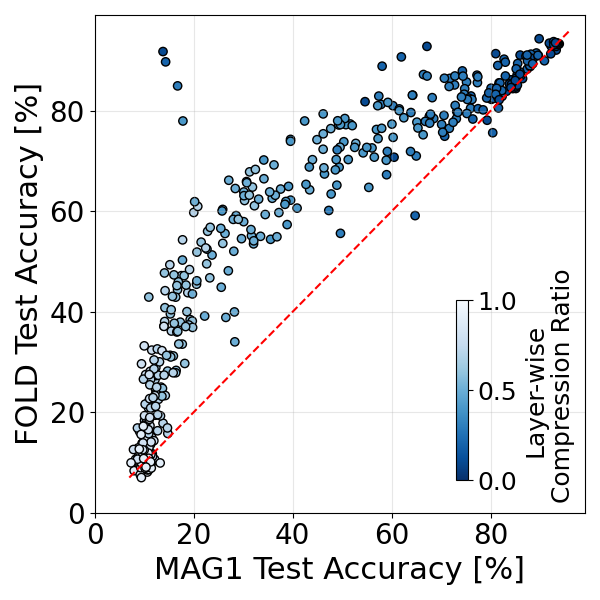}
        \includegraphics[height=0.24\textwidth]{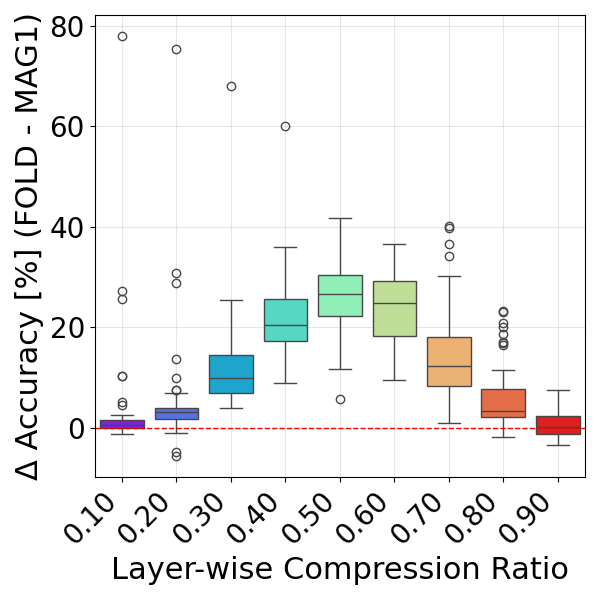}
     }
     
     \subfloat[][ViT-B/32, \magone vs \fold, base accuracy $>$75\%]{
        \includegraphics[height=0.24\textwidth]{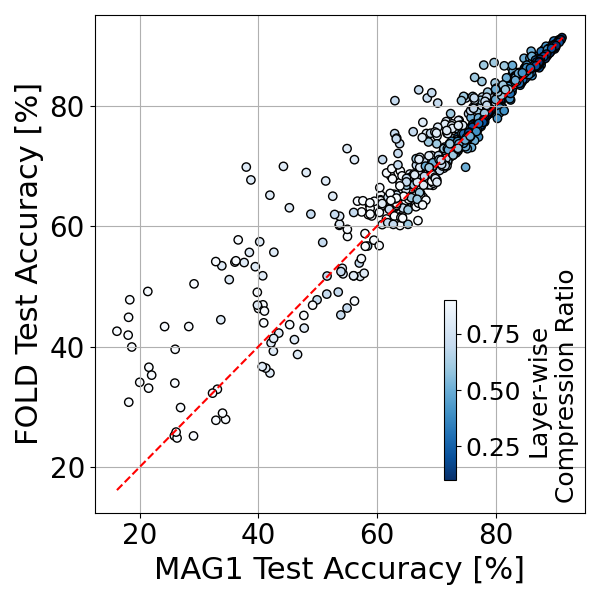}
        \includegraphics[height=0.24\textwidth]{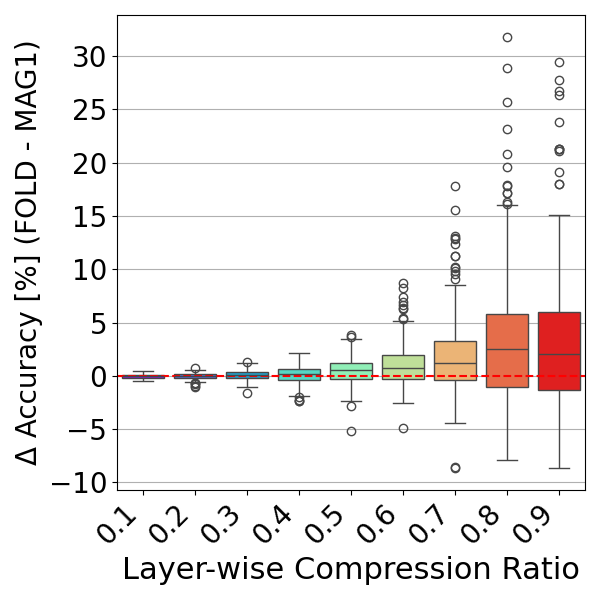}
     }
     \subfloat[][CLIP ViT-B/32, \magone vs \fold]{
        \includegraphics[height=0.24\textwidth]{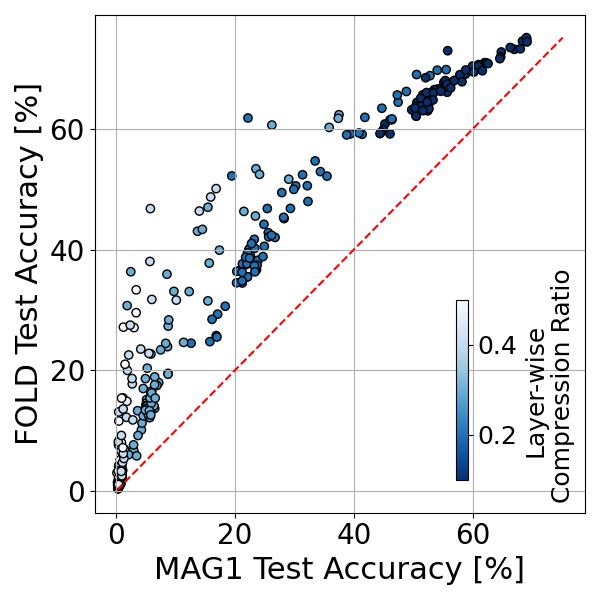}
        \includegraphics[height=0.24\textwidth]{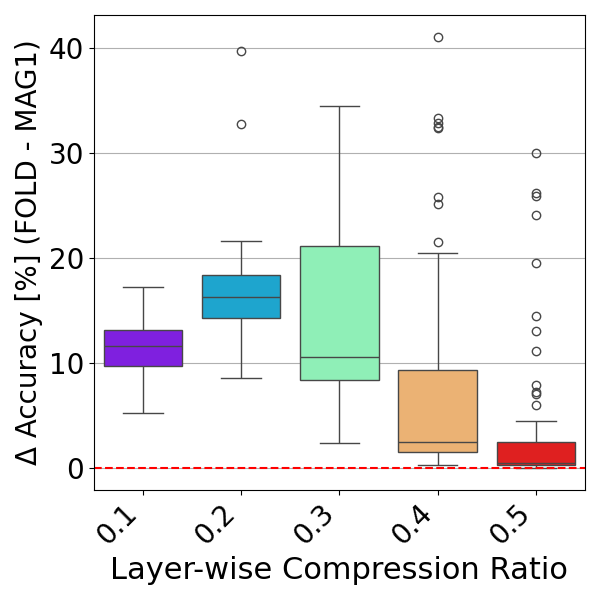}
     }
     \caption{\textbf{Folding outperforms magnitude pruning across diverse training regimes.}
     \textbf{Top row:} ResNet18 and PreActResNet18 on CIFAR-10. ResNet18 checkpoints were trained from scratch with Adam using different hyperparameter configurations. PreActResNet18 checkpoints are from \citet{andriushchenko2023modernlookrelationshipsharpness}.
     \textbf{Bottom row:} ViT-B/32 on CIFAR-10 from~\citep{andriushchenko2023modernlookrelationshipsharpness} and CLIP ViT-B/32 on ImageNet-1K from~\citep{wortsman2022modelsoupsaveragingweights}. See Appendix~\ref{appx:hyperparameters} for details. In these plots, we use checkpoints that were trained without L1 regularization. Scatter plots show post-compression accuracy for magnitude pruning (L1 criterion) versus folding at uniform per-layer compression ratios (color-coded by layer-wise compression ratio). Bar plots depict the accuracy gain by folding, computed as $\Delta=\mathrm{Acc}{(\text{\fold})}-\mathrm{Acc}{(\text{\magone})}$, as a function of layer-wise compression ratio. Folding yields the largest improvements at moderate to high compression, confirming its robustness across architectures and datasets.  \Figref{fig:accuracy_FOLD_vs_MAG_L2} shows the results for magnitude pruning with L2 criterion.
     }
     \label{fig:accuracy_FOLD_vs_MAG_L1}
\end{figure}

\section{Experimental Results}
\label{sec:exp}

Most pruning studies vary only seeds by training several checkpoints under a single hyperparameter recipe, leaving the role of upstream training underexplored. We instead benchmark $>1$'000 checkpoints spanning diverse hyperparameters (optimizers, learning rates, augmentation, regularization, SAM) to quantify how training choices interact with folding and pruning. Concretely, we train 216 ResNet18 (Adam) and 576 ResNet18 (SGD) models on CIFAR-10, include 50 PreActResNet18 and 200 ViT-B/32 checkpoints from \citep{andriushchenko2023modernlookrelationshipsharpness}, and add 72 CLIP ViT-B/32 models fine-tuned on ImageNet-1K from \citep{wortsman2022modelsoupsaveragingweights}. The two ViT families differ markedly in scale ($\sim$19M vs.\ $\sim$151M parameters). We also train 36 LLaMA-family 60M and 130M parameter models on the Colossal Clean Crawled Corpus (C4)~\citep{raffel2020exploring}. Training details are in Appendix~\ref{appx:hyperparameters}. The results for LLaMa-130M are in Appendix~\ref{appx:rebuttal}.

\begin{figure}[t]
     \centering
     \vskip -.4cm
     \subfloat[][ViT-B/32, \magone vs \fold, base accuracy $>$75\%]{
        \includegraphics[height=0.24\textwidth]{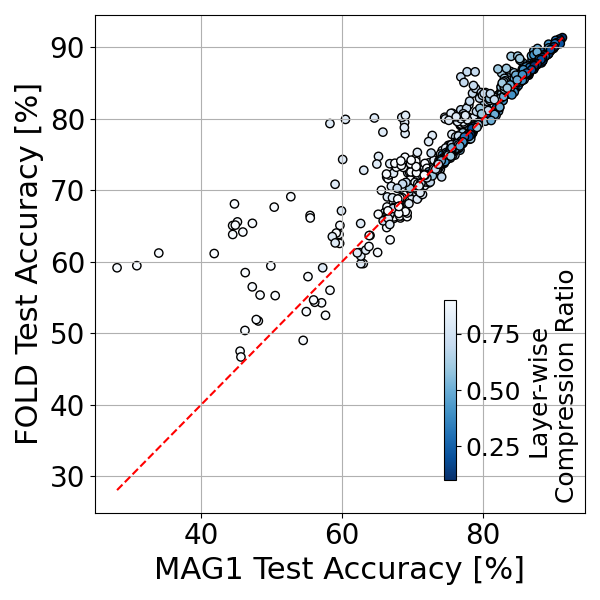}
        \includegraphics[height=0.24\textwidth]{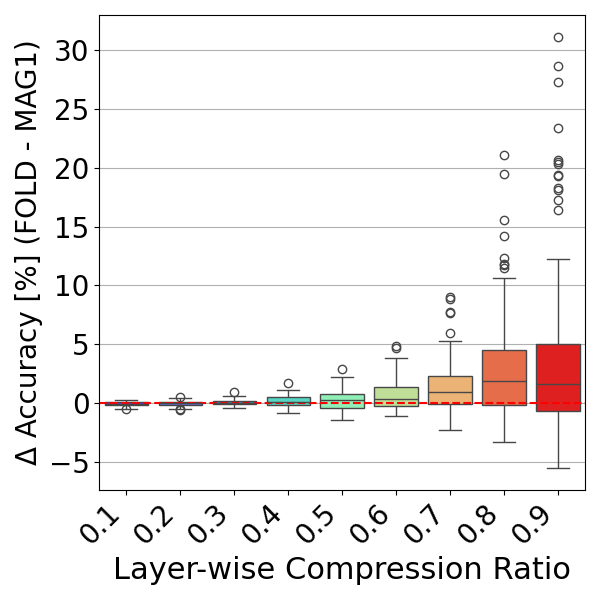}
     }
     \subfloat[][CLIP ViT-B/32, \magone vs \fold]{
        \includegraphics[height=0.24\textwidth]{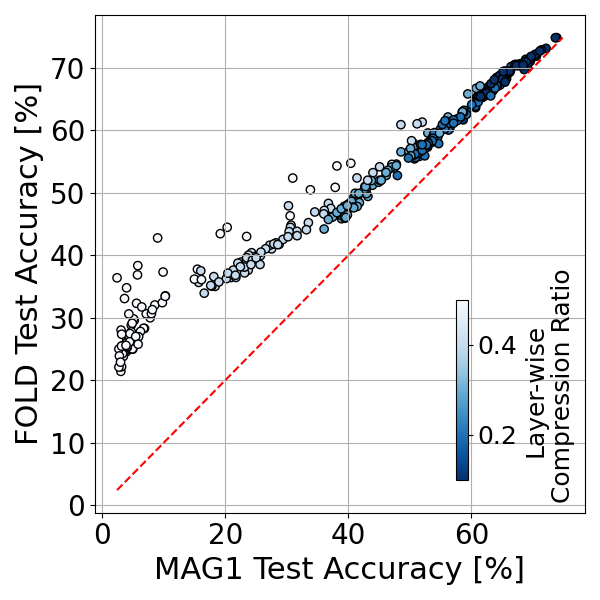}
        \includegraphics[height=0.24\textwidth]{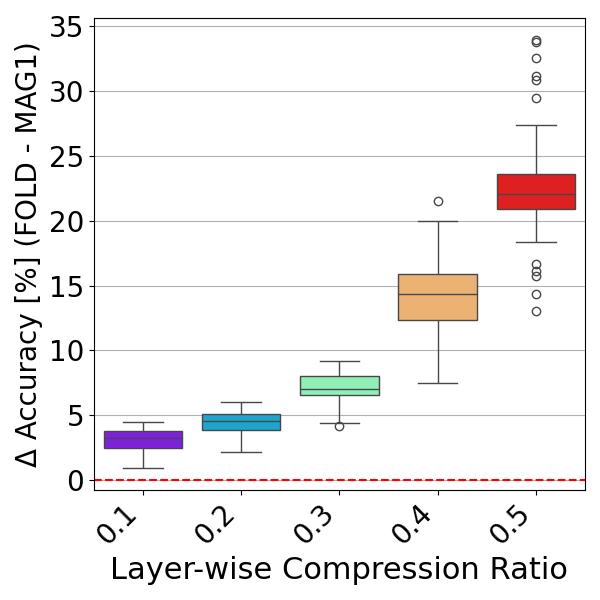}
     }
     \caption{\textbf{\magone versus \fold on ViTs after LayerNorm-only fine-tuning}
     for ViT-B/32 on CIFAR-10 and CLIP ViT-B/32 on ImageNet-1K. In the scatter plots, points are checkpoints, color encodes layer-wise compression. Bar plots depict the accuracy gain $\Delta=\mathrm{Acc}{(\text{\fold})}-\mathrm{Acc}{(\text{\magone})}$, which remains positive and typically grows with compression, indicating that even under lightweight LayerNorm adaptation \fold retains a consistent advantage over pruning.
     }
     \label{fig:ftLN_ViTs_FOLD_vs_MAG_L1}
\end{figure}

\begin{figure}[t]
     \centering
     \vskip -.4cm
     \subfloat[][\magone vs \fold, fine-tuning for\\1 (\textbf{left}) and 5 (\textbf{right}) epochs.]{\includegraphics[height=0.2\textwidth]{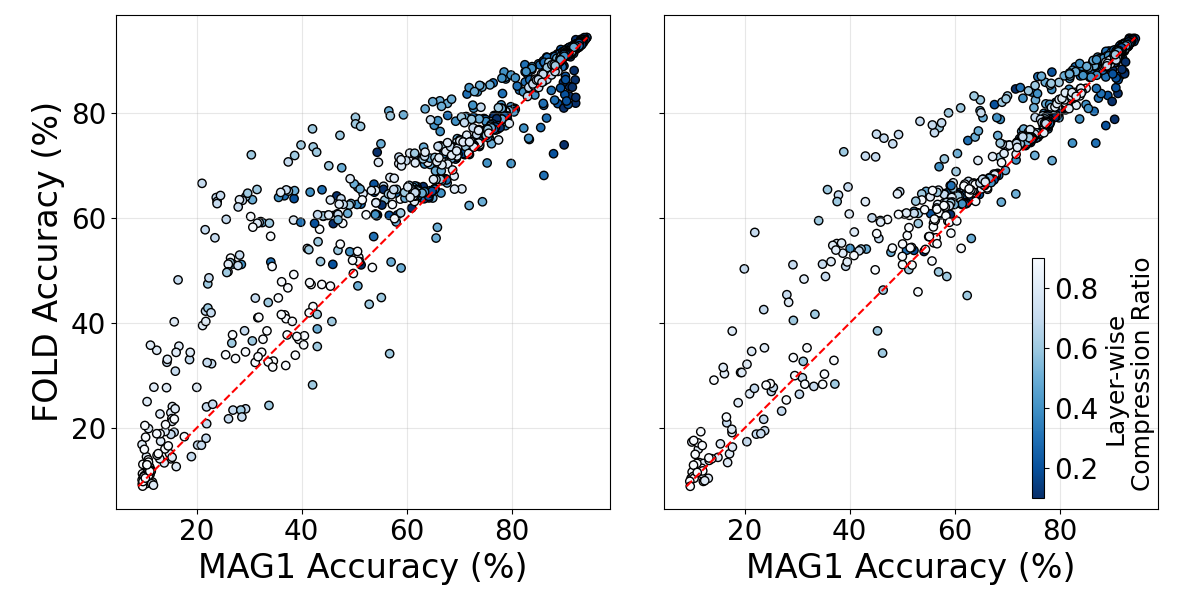}}
     \subfloat[][\magone vs \fold\\accuracy gap after ft.]{\includegraphics[height=0.2\textwidth]{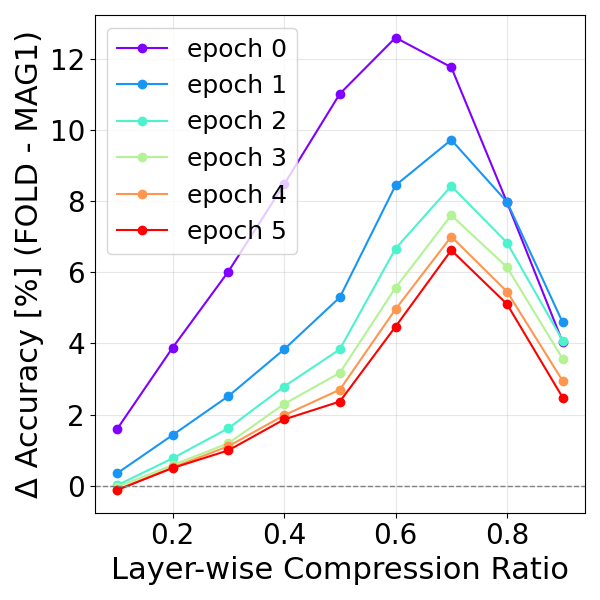}}
     \subfloat[][Before and after fine-tuning for 5 epochs: \fold (\textbf{left}) and \magone (\textbf{right}).]{\includegraphics[height=0.2\textwidth]{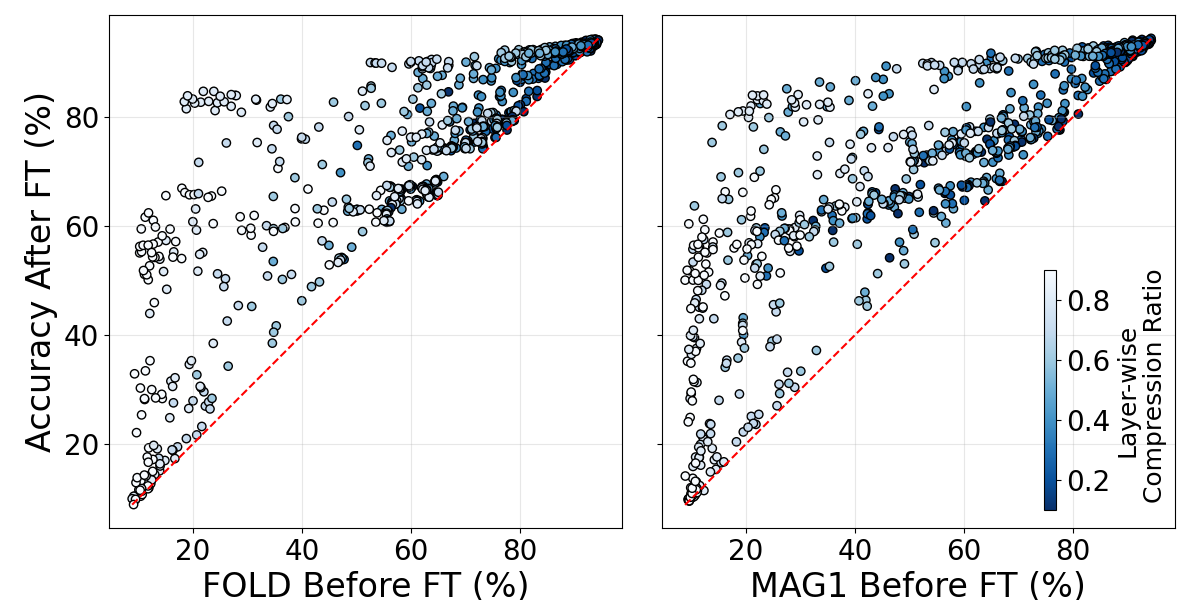}}
     
     \subfloat[][\magone vs \fold, fine-tuning for\\1 (\textbf{left}) and 5 (\textbf{right}) epochs.]
     {\includegraphics[height=0.2\textwidth]{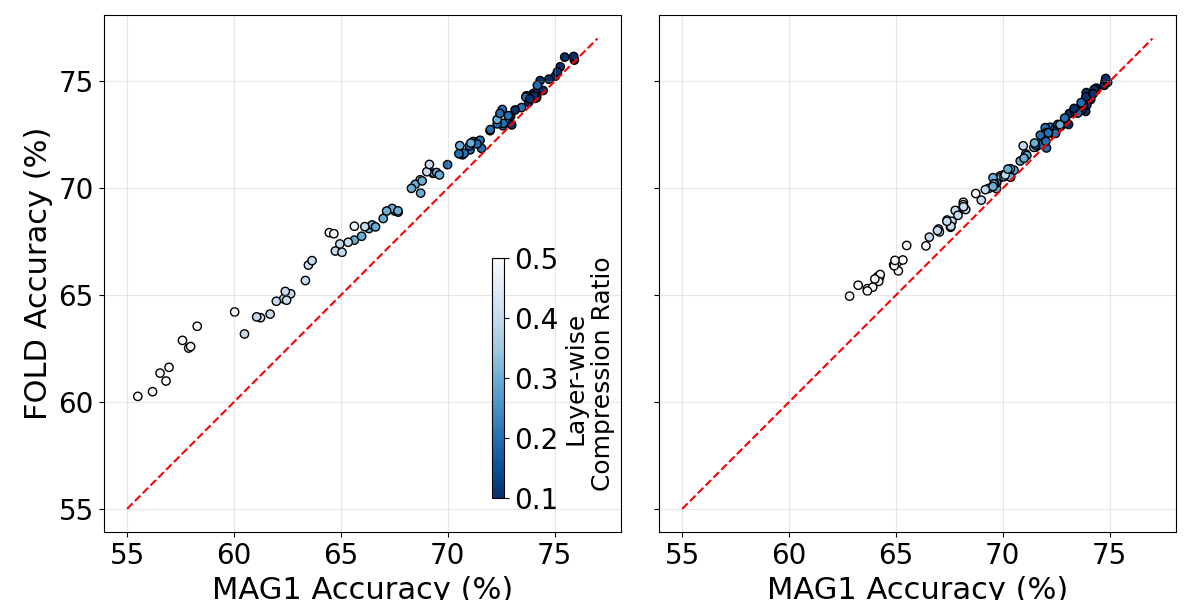}}
     \subfloat[][\magone vs \fold\\accuracy gap after ft.]{\includegraphics[height=0.2\textwidth]{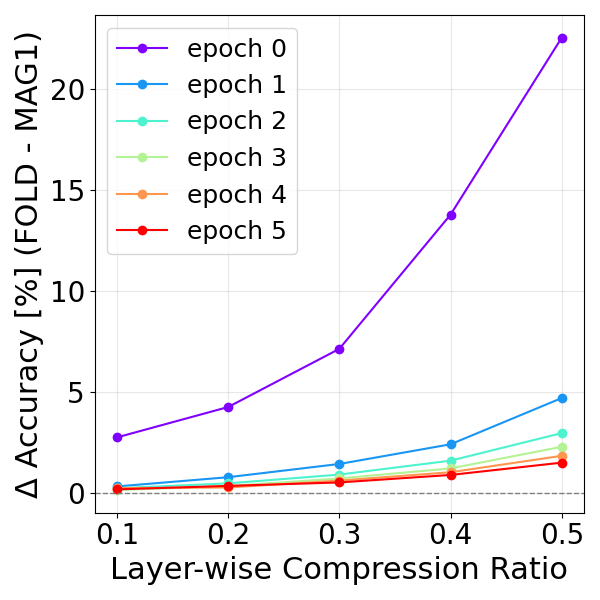}}
     \subfloat[][Before and after fine-tuning for 5 epochs: \fold (\textbf{left}) and \magone (\textbf{right}).]
     {\includegraphics[height=0.2\textwidth]{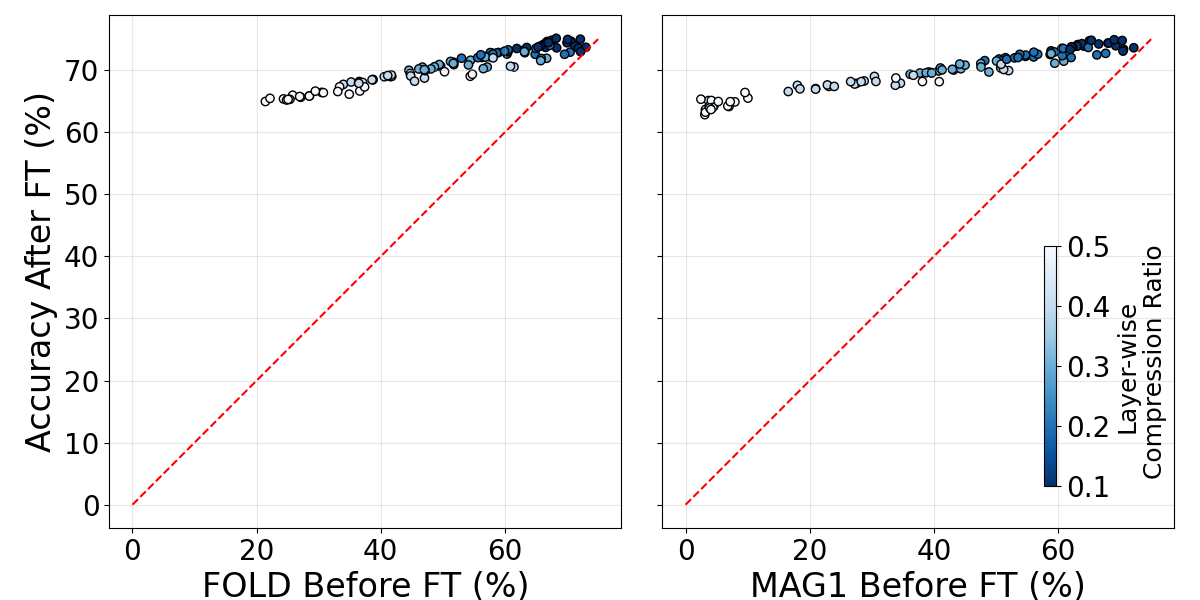}}
     
     \caption{\textbf{Folded models retain their accuracy advantage after fine-tuning.} Results for ResNet18 trained by Adam on CIFAR-10 (\textbf{top row}) and CLIP-ViT-B/32 trained on ImageNet-1K (\textbf{bottom row}): \textbf{(a,d)} compares post-compression accuracy of magnitude pruning (\magone) versus folding (\fold) after 1 and 5 epochs of fine-tuning. \textbf{(b,e)} show the accuracy gap between folding and pruning as a function of fine-tuning epochs, demonstrating that folding maintains a consistent lead, \ie the \fold accuracy delta is positive. \textbf{(c,f)} illustrate accuracy trajectories before and after 5 epochs of fine-tuning for both methods, highlighting that folded models recover accuracy faster. Further results in Appendix~\ref{appx:further_results}.}
     \label{fig:ft_CNN_and_ViTs_1_FOLD_vs_MAG_L1}
\end{figure}

We empirically compare model folding and structured pruning across CNNs, ViTs and LLaMA-60M models under matched training setups. Unless stated otherwise, we do not apply gradient-based fine-tuning: for CNNs we only re-estimate batch-normalization statistics via a single forward pass using \textsc{REPAIR}~\citep{jordan2023repairrenormalizingpermutedactivations} to isolate structural effects, and ViTs / LLaMA-60M models are left uncalibrated. Note that \textsc{REPAIR} was recently shown to substantially improve post-compression performance for pruned models~\citep{saikumar2025signalcollapseoneshotpruning}, and has also been applied on top of folding~\citep{wang2025forget}. We report results (i) immediately after compression (CNNs after \textsc{REPAIR}, ViTs with no further step), (ii) for ViTs additionally after a LayerNorm reset, and (iii) for both families after 1--5 epochs of full fine-tuning.

\textbf{Folding vs.\ Structured Pruning on CNNs and ViTs.}
We compare model folding (\fold) with structured magnitude pruning (\magprune) under L1 and L2 criteria (\magone, \magtwo) across representative CNN and ViT architectures. \Figref{fig:accuracy_FOLD_vs_MAG_L1} summarizes results: scatter plots show accuracy of \magone vs.\ \fold for each trained model, with compression ratio indicated by color. Results for \magtwo are in Appendix~\ref{appx:further_results}. Box plots depict the distribution of accuracy differences between \fold and \magone. Positive differences indicate folding outperforms pruning, with the gap widening at higher sparsity. This trend holds across ResNet18, PreActResNet18, ViT-B/32, and CLIP ViT-B/32 on both CIFAR-10 and ImageNet-1K, demonstrating robustness to architecture and dataset scale. These results support our theoretical claim (\Secref{sec:theory}): folding projects onto cluster-structured subspaces, preserving parameter alignment and reducing functional distortion, yielding consistent accuracy gains over magnitude pruning.

\begin{table*}[b!]
\centering
\small
\setlength{\tabcolsep}{4pt}
\renewcommand{\arraystretch}{1.2}
\resizebox{1.00\textwidth}{!}{%
\begin{tabular}{cccc|cc|cc}
\hline
\rowcolor{lightblue}
\textbf{weight\_decay} & \textbf{warmup\_steps} & \textbf{max\_lr} & 
\textbf{PPL$\downarrow$ 0\% sparsity} & \textbf{PPL$\downarrow$ \magtwo (20\%)} & \textbf{PPL$\downarrow$ \fold (20\%)} &
\textbf{PPL$\downarrow$ \magtwo (50\%)} & \textbf{PPL$\downarrow$ \fold (50\%)} \\
\hline
0.01 & 880  & 0.001 & 32.11 & 54.51 & \textbf{47.17} & 398.62 & \textbf{221.32} \\
0.01 & 1100 & 0.001 & 32.14 & 50.11 & \textbf{46.75} & 220.54 & \textbf{172.57} \\
0.01 & 2200 & 0.001 & 32.20 & \textbf{46.57} & 47.54 & \textbf{174.58} & 216.36 \\
0    & 880  & 0.001 & 32.17 & 51.14 & \textbf{48.23} & \textbf{220.33} & 223.86 \\
0    & 1100 & 0.001 & 32.21 & 50.03 & \textbf{47.47} & 231.41 & \textbf{204.47} \\
0    & 2200 & 0.001 & 32.40 & \textbf{46.38} & 46.92 & \textbf{177.48} & 185.27 \\
0.01 & 880  & 0.005 & 30.12 & 68.70 & \textbf{55.32} & 641.69 & \textbf{302.43} \\
0.01 & 1100 & 0.005 & 29.77 & 68.29 & \textbf{49.81} & 564.96 & \textbf{234.56} \\
0.01 & 2200 & 0.005 & 29.60 & 54.50 & \textbf{47.04} & 360.52 & \textbf{208.02} \\
0    & 880  & 0.005 & 30.47 & 78.73 & \textbf{62.35} & 762.05 & \textbf{395.04} \\
0    & 1100 & 0.005 & 30.17 & 59.20 & \textbf{49.58} & 544.87 & \textbf{184.74} \\
0    & 2200 & 0.005 & 29.75 & 56.18 & \textbf{46.55} & 353.35 & \underline{\textbf{165.21}} \\
0.01 & 880  & 0.01  & 31.82 & 66.98 & \textbf{51.80} & 910.48 & \textbf{406.75} \\
0.01 & 1100 & 0.01  & 29.85 & 102.41 & \textbf{67.69} & 977.92 & \textbf{367.94} \\
0.01 & 2200 & 0.01  & \underline{29.25} & 51.46 & \underline{\textbf{44.28}} & 323.68 & \textbf{288.83} \\
0    & 880  & 0.01  & 108.56 & 129.77 & \textbf{123.85} & 279.17 & \textbf{198.72} \\
0    & 1100 & 0.01  & 30.31 & 97.97 & \textbf{61.19} & 860.14 & \textbf{533.62} \\
0    & 2200 & 0.01  & 29.57 & 54.43 & \textbf{47.77} & 351.11 & \textbf{209.06} \\
\hline
\end{tabular}
}
\caption{\textbf{Evaluation of \fold and \magtwo on LLaMA-60M.} We train and evaluate 18 LLaMA-family models with 60M parameters on C4 by varying max\_lr, warmup steps and weight decay. Columns 4--8 show perplexity of the trained model before compression and after pruning / folding using layer-wise pruning ratio of $20\%$ and $50\%$. We prune only FFN blocks. Except for low learning rates with long warmup schedules, \fold outperforms \magtwo (highlighted in bold).}
\label{tab:llama60m}
\end{table*}

\textbf{Performance Comparison after Lightweight and Full Fine-Tuning.}
The previous results isolate structural effects by evaluating models without further optimization. We now test whether folding’s advantage persists after fine-tuning. 
\Figref{fig:ftLN_ViTs_FOLD_vs_MAG_L1} compares \magone and \fold on ViTs under lightweight LayerNorm-only adaptation: across ViT-B/32 (CIFAR-10) and CLIP ViT-B/32 (ImageNet-1K), folding consistently reaches higher post-compression accuracy, with the gap $\Delta=\mathrm{Acc}(\fold)-\mathrm{Acc}(\magone)$ remaining positive and typically growing with compression.

Next, we allow brief fine-tuning (1–5 epochs). 
\Figref{fig:ft_CNN_and_ViTs_1_FOLD_vs_MAG_L1} shows that folded models (a,d) start from higher accuracy and retain their lead, (b,e) maintain a positive relative gap, and (c,f) recover faster with fewer plateaus. 
Thus, folding provides a better initialization and requires fewer updates to regain performance, making it advantageous in settings with limited fine-tuning.

\textbf{Performance Comparison on LLaMA-60M.}
\Tabref{tab:llama60m} evaluates \fold and \magtwo on LLaMA-60M trained on C4 under 18 hyperparameter settings (varying learning rate, warmup, and weight decay). We prune or fold only the FFN blocks and report perplexity at baseline and at 20\% and 50\% layer-wise sparsity. Except for models trained with very low learning rates and long warmup, \fold consistently outperforms \magtwo. Similar finding have been obtained for LLaMA-130M models in \Tabref{tab:llama130m}.

\section{Model Compression Ablation Studies}
\label{sec:ablations}
The previous sections demonstrated that folding often outperforms structured pruning across architectures and compression ratios. On ResNets and ViTs, we probe which training factors impact this advantage. Specifically, we analyze sensitivity to learning rate, the use of sharpness-aware training (SAM)~\citep{foret2021sharpnessawareminimizationefficientlyimproving}, regularization and data augmentation~\citep{prabhu2019understandingadversarialrobustnessloss}---the hyperparameters known to influence loss landscape geometry and generalization~\citep{fort2019largescalestructureneural,li2018visualizinglosslandscapes,neyshabur2017exploringgeneralizationdeeplearning,chen2022visiontransformersoutperformresnets} in non-trivial ways~\citep{andriushchenko2023modernlookrelationshipsharpness}. To validate these curvature-related hypotheses, Appendix~\ref{appx:rebuttal} includes a sharpness analysis. Our measurements quantify how hyperparameters shift the local geometry of the loss landscape and help explain when \fold’s advantage widens or narrows.

\begin{wrapfigure}{r}{0.5\textwidth}
     \centering
     \includegraphics[height=0.24\textwidth]{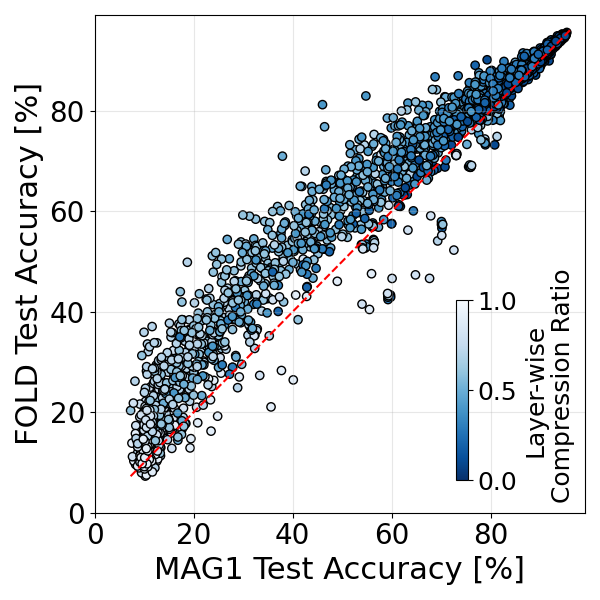}
     \includegraphics[height=0.24\textwidth]{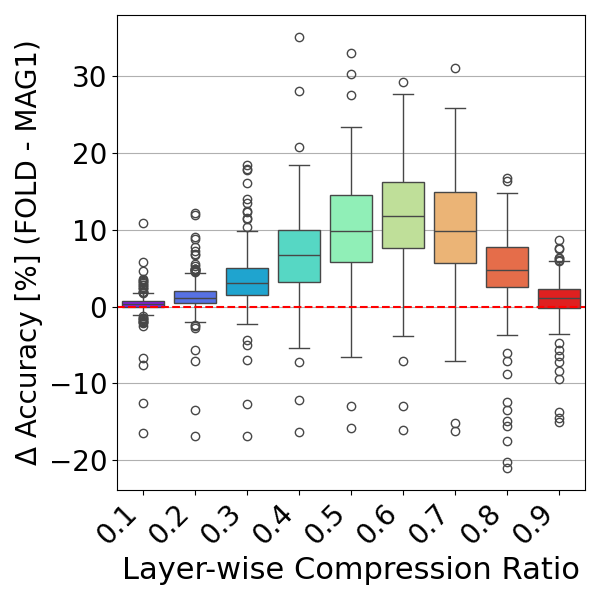}
     \caption{\textbf{Optimizer effect} evaluated on ResNet18 checkpoints trained on CIFAR-10 with SGD (no L1 regularization). The figure complements  \Figref{fig:accuracy_FOLD_vs_MAG_L1}(a).}
     \label{fig:optimizer}
\end{wrapfigure}

\textbf{Role of Optimizer.}
We repeat the ResNet18 analysis under Adam and SGD to gauge optimizer sensitivity. Compared to the Adam-trained sweep in \Figref{fig:accuracy_FOLD_vs_MAG_L1}(a), the complementary SGD sweep in \Figref{fig:optimizer} shows the same qualitative ordering—\fold exceeds \magone across compression levels—but with different baselines and dispersion: SGD checkpoints form a tighter cloud and exhibit a smaller median gap, whereas Adam yields larger variance and at times a more pronounced \fold advantage, especially at higher compression. 
The \fold--\magone difference remains positive under both optimizers in most cases, but its magnitude is optimizer-dependent. 

\textbf{Effect of Learning Rate.}
\Figref{fig:lr-effects} reports post-compression accuracy for \fold versus \magone across learning rates on ResNet18 (Adam, SGD), PreActResNet18, and ViT-B/32. With Adam, \fold’s edge is largest at moderate–low rates, narrows and can reverse at very high rates, and vanishes again at extremely small rates (both methods degrade). For SGD, the dependence is weaker and can be inverted (\eg ViT-B/32). 
The effect of learning rate is expressed through sharpness (see Appendix~\ref{appx:rebuttal}): when training places the model in regions where folding produces a smaller sharpness increase than pruning, folding wins. When folding produces a larger sharpness increase (most visible at high learning rates under Adam), pruning can outperform.
Adaptive methods like Adam are associated with sharper minima and distinct generalization behavior compared to SGD, amplifying this sensitivity~\citep{wilson2018marginalvalueadaptivegradient,jastrzębski2018factorsinfluencingminimasgd,zhou2021theoreticallyunderstandingsgdgeneralizes}.

\begin{figure}[h]
     \centering
     \vskip -.4cm
     \subfloat[][ResNet18, Adam, \textbf{no} L1 reg., \textbf{no} weight decay]{
        \includegraphics[height=0.24\textwidth]{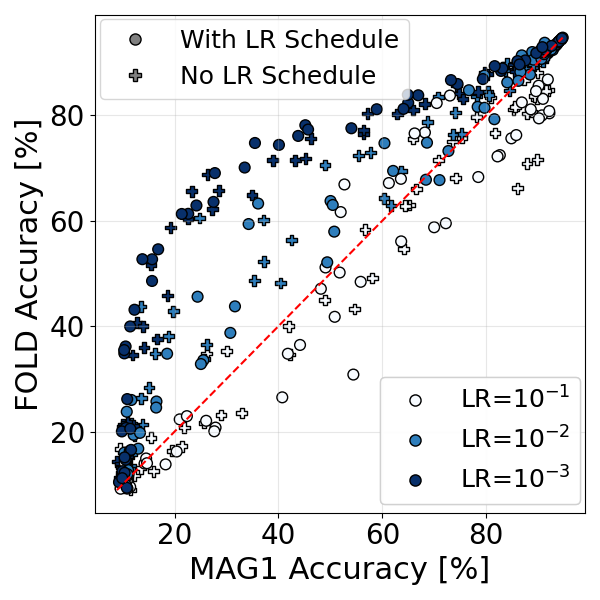}
        \includegraphics[height=0.24\textwidth]{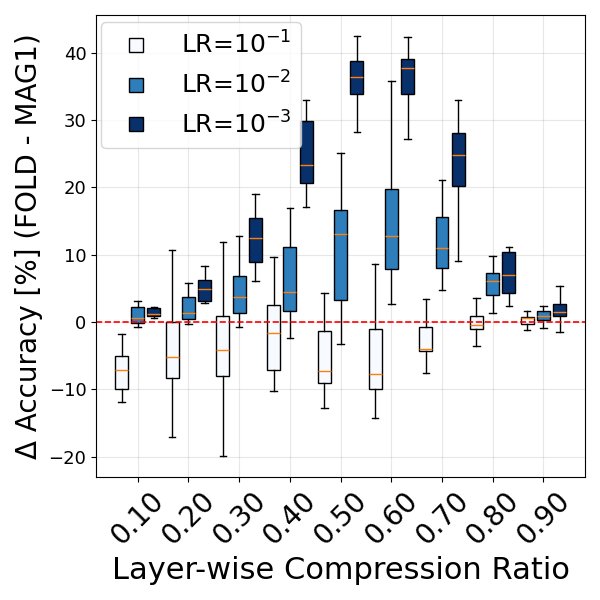}
    }
    \subfloat[][ResNet18, SGD, \textbf{no} L1 reg., \textbf{no} weight decay]{
        \includegraphics[height=0.24\textwidth]{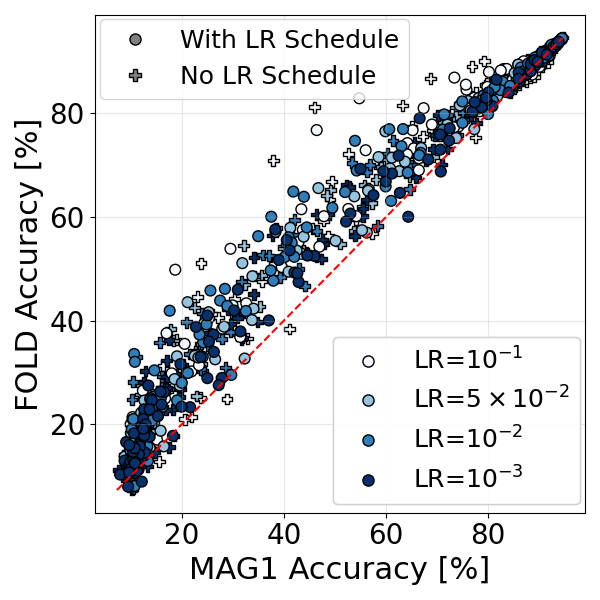}
        \includegraphics[height=0.24\textwidth]{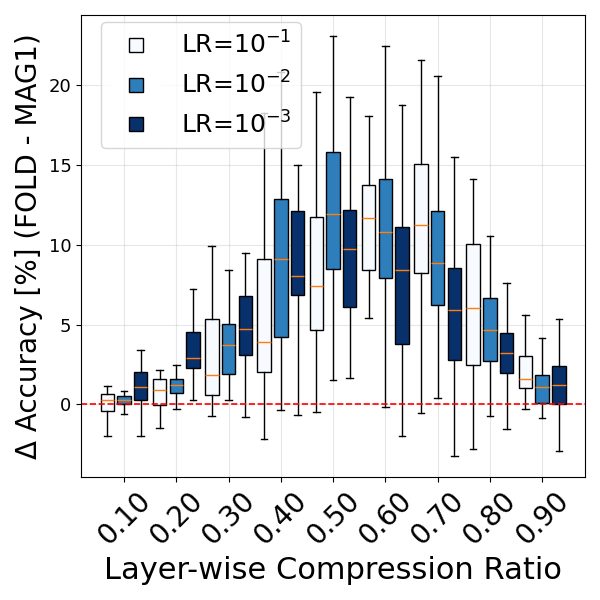}
    }

    \subfloat[][PreActResNet18, SGD]{
        \includegraphics[height=0.24\textwidth]{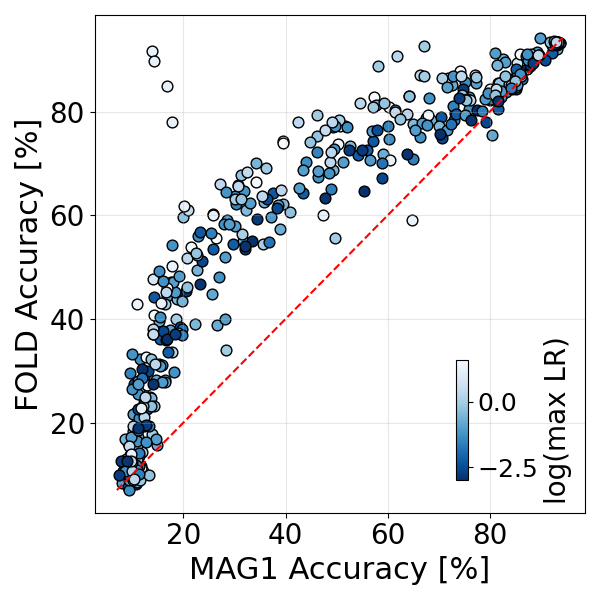}
        \includegraphics[height=0.24\textwidth]{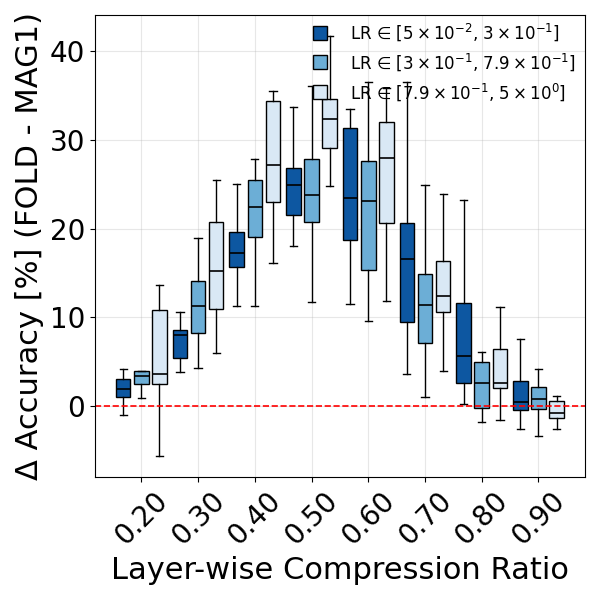}
    }
    \subfloat[][ViT-B/32, SGD]{
        \includegraphics[height=0.24\textwidth]{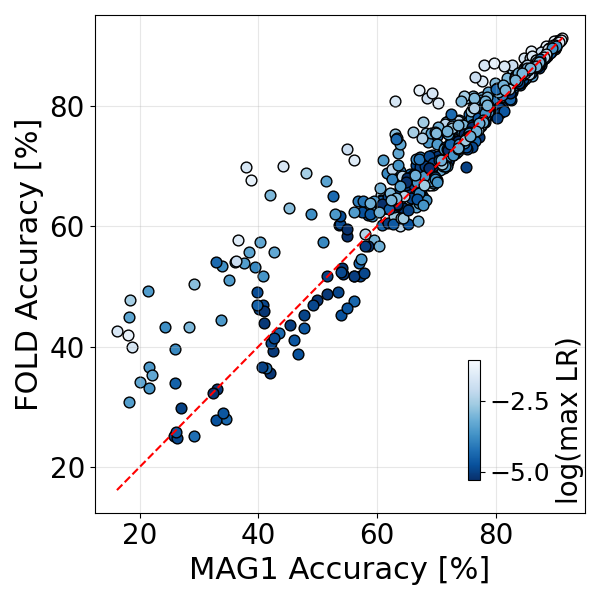}
        \includegraphics[height=0.24\textwidth]{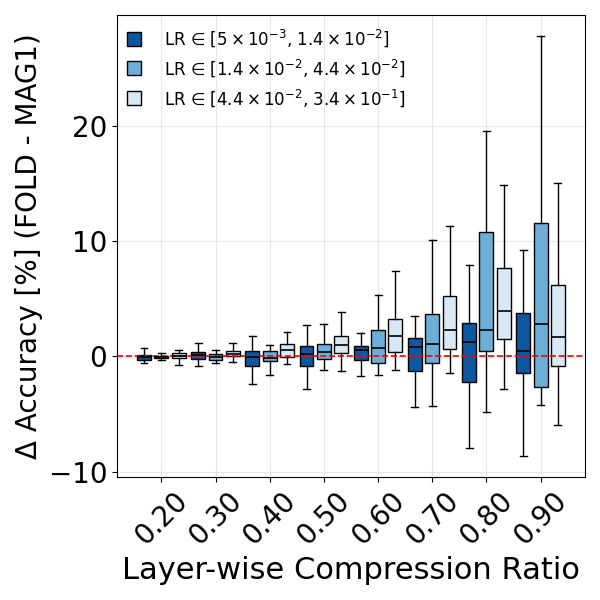}
    }
     
     \caption{\textbf{Learning rate modulates folding’s edge.} Post-compression accuracy of \fold and \magone across learning rates: ResNet18 with Adam \textbf{(a)} and SGD \textbf{(b)}, PreActResNet18 \textbf{(c)}, and ViT-B/32 \textbf{(d)}. \fold leads at moderate–low rates. With Adam, the gap shrinks or reverses at very high rates, and closes again at extremely small rates. SGD shows weaker or opposite dependence.}
    \label{fig:lr-effects}
\end{figure}

\textbf{Effect of SAM.}
\Figref{fig:sam-effects} evaluates training with and without SAM and measures post-compression accuracy. Across models, SAM lifts both methods. 
With light L1 regularization ($10^{-5}$) during training shown in (b), pruning narrows the gap at \emph{low} compression (where induced sparsity aligns with L1), yet \fold regains and extends its lead as compression increases. These trends are consistent with the view that SAM steers training to flatter solutions, reducing curvature sensitivity. 
Within this flatter neighborhood both pruning and folding projections operate inside the same robustness ball, so their geometric differences matter less and the gap narrows—an effect stronger for ViT-B/32, where high $\rho$ homogenizes head\,/\,channel saliencies and reduces the relative advantage of clustering.

\begin{figure}[h]
     \centering
     \vskip -.4cm
     \subfloat[][ResNet18, Adam, \magone vs \fold, \textbf{without} L1 reg.]{
        \includegraphics[height=0.24\textwidth]{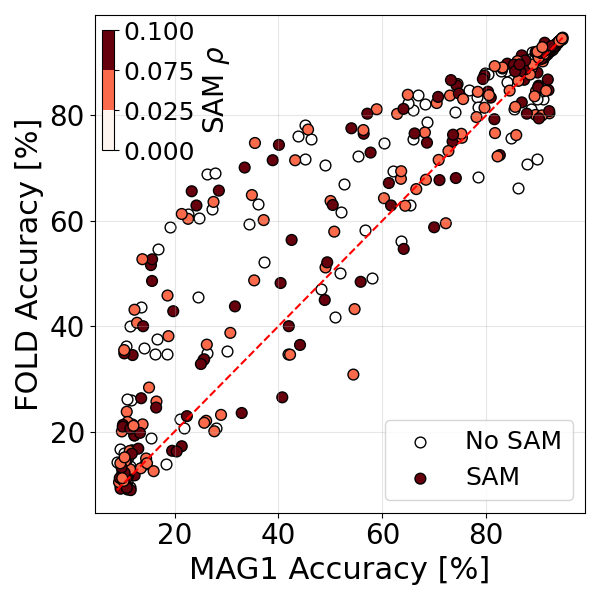}
        \includegraphics[height=0.24\textwidth]{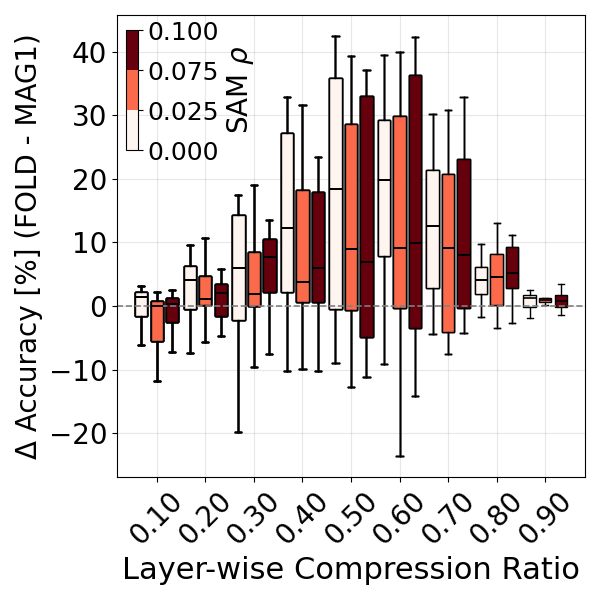}
     }
     \subfloat[][ResNet18, Adam, \magone vs \fold, \textbf{with} L1 reg.]{
        \includegraphics[height=0.24\textwidth]{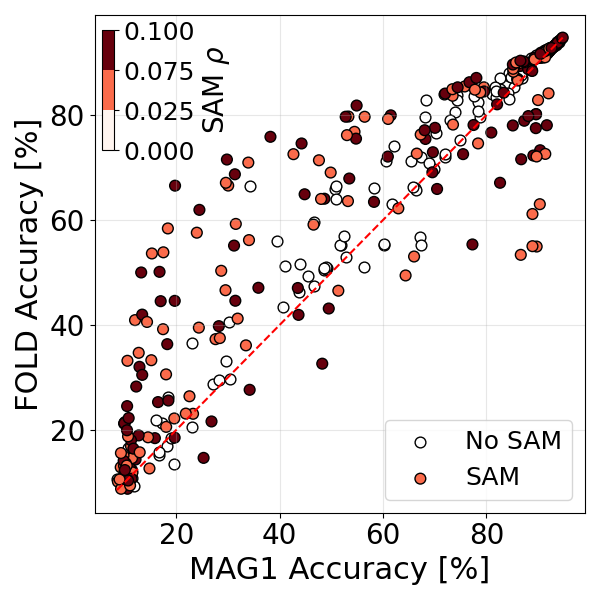}
        \includegraphics[height=0.24\textwidth]{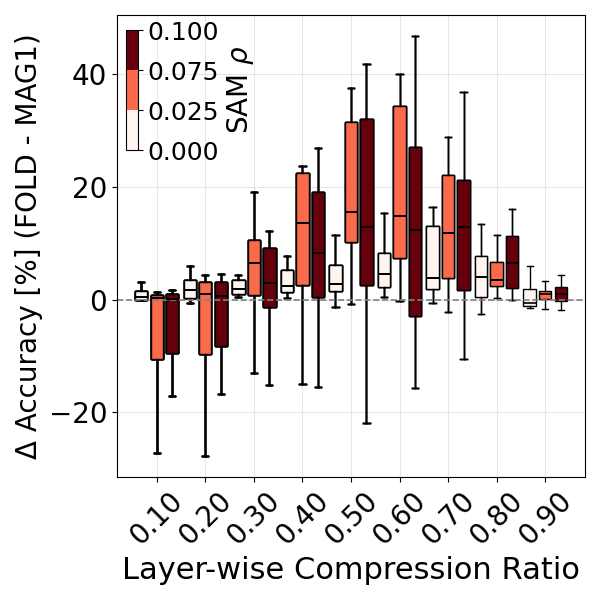}
     }

     \subfloat[][PreActResNet, \magone vs \fold, \textbf{no} L1 regularization]{
        \includegraphics[height=0.24\textwidth]{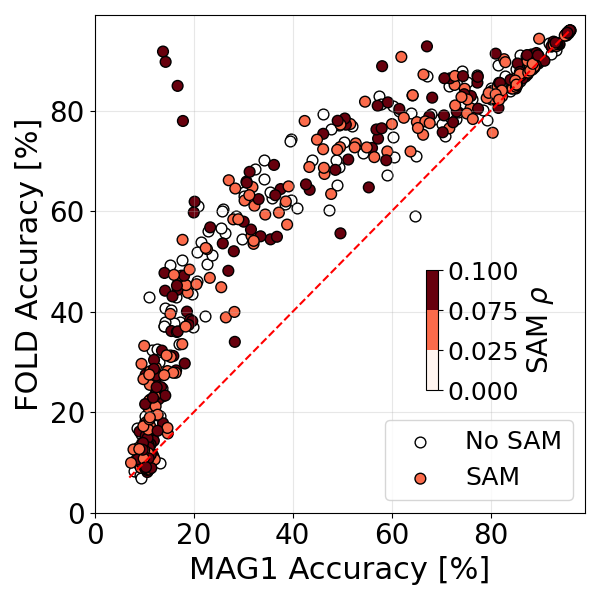}
        \includegraphics[height=0.24\textwidth]{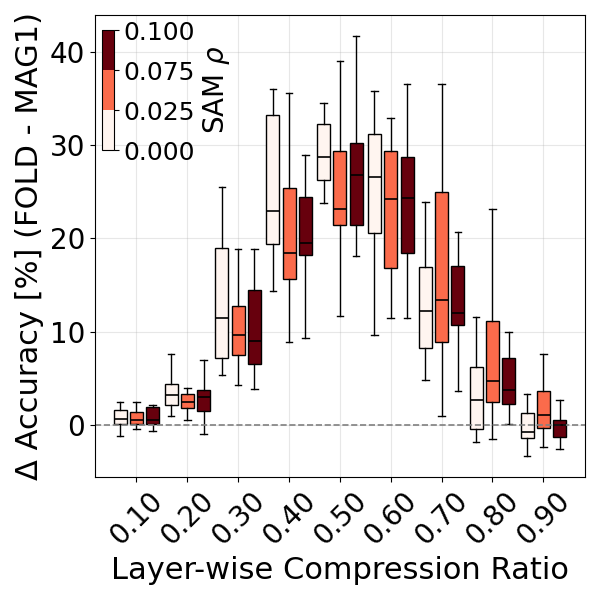}
     }
     \subfloat[][ViT-B/32, \magone vs \fold, \textbf{no} L1 regularization]{
        \includegraphics[height=0.24\textwidth]{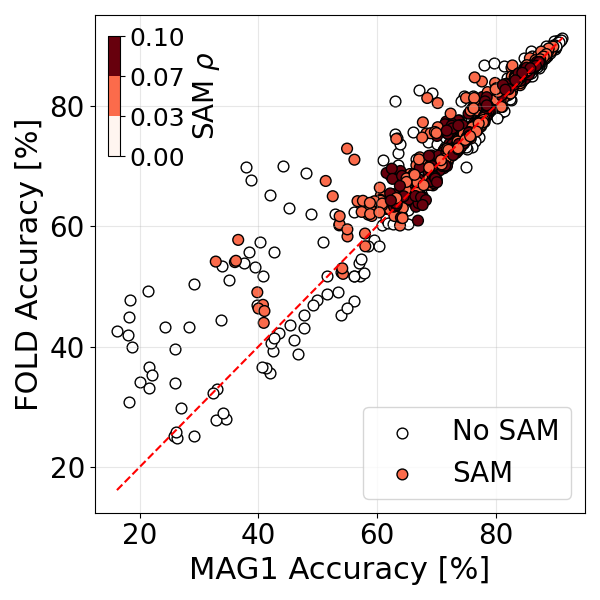}
        \includegraphics[height=0.24\textwidth]{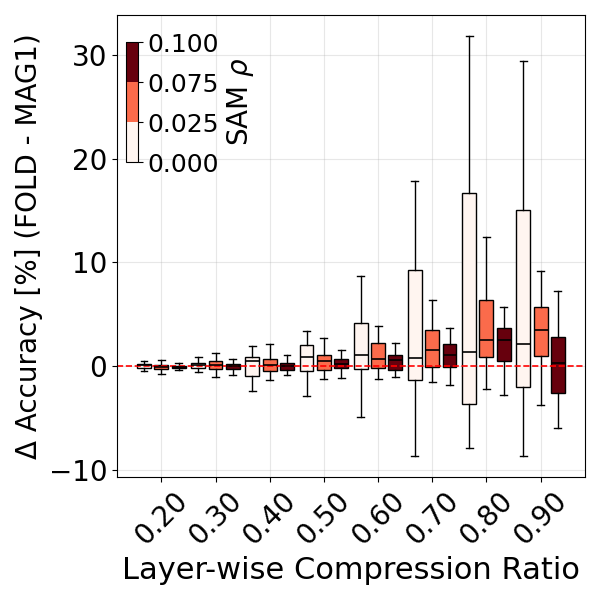}
      }

     \caption{\textbf{SAM~\citep{foret2021sharpnessawareminimizationefficientlyimproving} can boost model compression.} Post-compression accuracy under training with\,/\,without SAM. \textbf{(a)} ResNet18 (Adam), no L1. \textbf{(b)} ResNet18 (Adam), L1$=10^{-5}$. \textbf{(c)} PreActResNet18 (SGD), no L1. \textbf{(d)} ViT\mbox{-}B/32, no L1. SAM improves both \fold and \magone, but the uplift is consistently larger for \fold, especially with Adam. Light L1 regularization helps \magone at low compression, yet \fold retains a clear advantage at moderate–high compression.}
    \label{fig:sam-effects}
\end{figure}

\textbf{Effect of Data Augmentation.}
\Figref{fig:raug-effects} plots the distribution of \(\Delta\)Accuracy \((\fold-\magone)\) across checkpoints versus the layer-wise compression ratio, contrasting runs without (gray) and with RandAugment (green). For ResNet18 (Adam and SGD) and PreActResNet18, RAUG reduces or shifts \fold’s relative benefit. In contrast, for ViT-B/32 RAUG increases \fold’s advantage: the median \(\Delta\) rises with compression, suggesting that augmented ViT representations are especially amenable to projection-based removal. A plausible mechanism is that augmentation biases training toward flatter, more invariant solutions. This is supported by our sharpness analysis in Appendix~\ref{appx:rebuttal} and consistent with recent theory linking augmentation-induced input perturbations to equivalent parameter-space perturbations and showing that augmentations bias training toward flatter minima~\citep{yoo2025flatminimaperspectiveunderstanding}. 
Standard augmentation (\texttt{augm=True}) shows a similar trend and is omitted for brevity.

\begin{figure}[h]
     \centering
     \vskip -.4cm
     \subfloat[][ResNet18, Adam]{
        \includegraphics[height=0.24\textwidth]{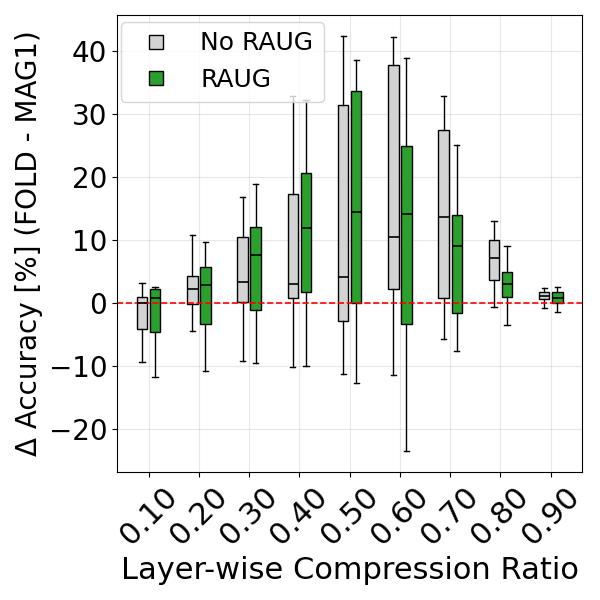}}
     \subfloat[][ResNet18, SGD]{
        \includegraphics[height=0.24\textwidth]{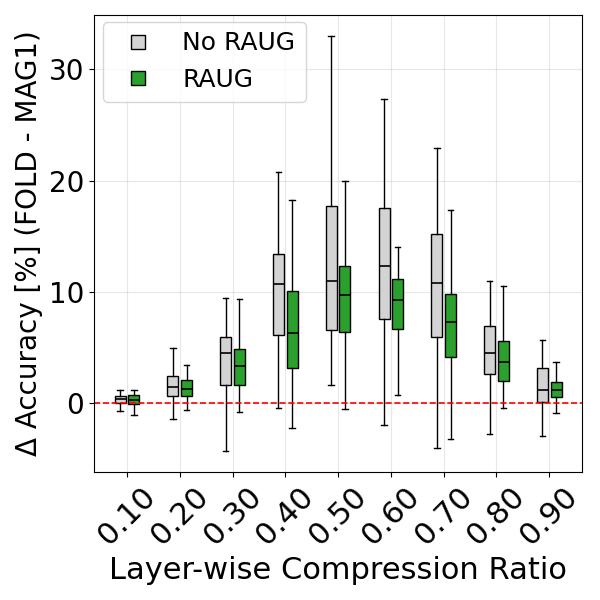}}
     \subfloat[][PreActResNet]{
        \includegraphics[height=0.24\textwidth]{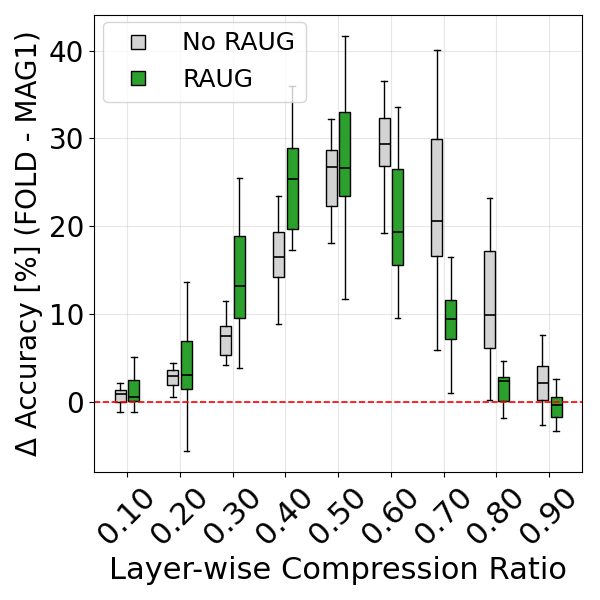}}
    \subfloat[][ViT-B/32]{
        \includegraphics[height=0.24\textwidth]{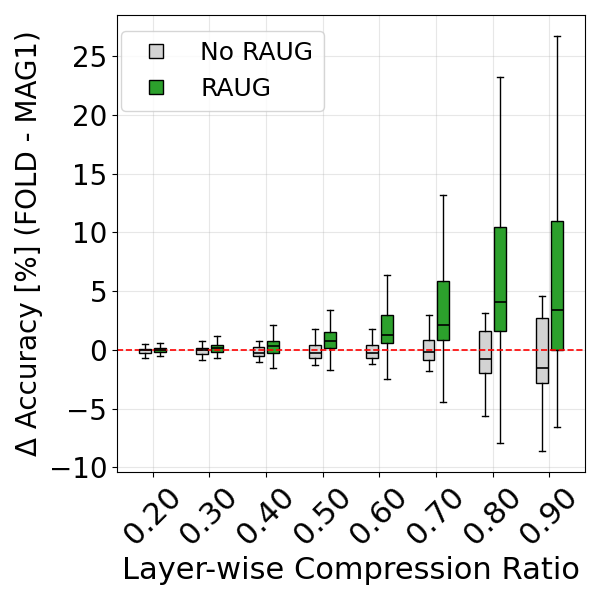}
     }
     \caption{\textbf{Augmentations have a generally positive effect on the post-compression accuracy.} 
    Post-compression accuracy without\,/\,with random augmentations for \textbf{(a)} ResNet18 (Adam), \textbf{(b)} ResNet18 (SGD), \textbf{(c)} PreActResNet18, and \textbf{(d)} ViT-B/32. 
    Augmentations boost both \fold and \magone. On ResNet18 they also narrow \fold’s advantage—most noticeably at moderate compression—consistent with added invariances making axis-aligned removals less damaging. 
    In ViT-B/32, augmentations are essential for folding\protect\footnotemark.}
    \label{fig:raug-effects}
\end{figure}

These ablations reveal a consistent pattern: conditions that encourage flatter and structured solutions—moderately low learning rates and SAM with a small–moderate radius—magnify \fold’s advantage, whereas extremes reduce it: very high or very low learning rates, stronger augmentations, or large SAM radii narrow the gap; SGD generally dampens all effects relative to Adam. This aligns with our projection view (\Secref{sec:theory}): when weights are well aligned, clustering reduces projection error more than coordinate removal and thus perturbs the function less, while weaker alignment or broad robustness neighborhoods make the two projections behave more similarly.

\section{Conclusion, Limitations, and Outlook}
We framed structured pruning and model folding as projection-based compression and showed that folding achieves smaller parameter deviation with a one-rank slack, implying tighter functional preservation under mild smoothness. A calibration-free evaluation over $>$1'000 checkpoints (ResNet18, PreActResNet18, ViT-B/32, CLIP ViT-B/32; CIFAR-10, ImageNet-1K; and LLaMA-60M and LLaMA-130M on C4) found that \fold typically surpasses \magone in post-compression accuracy, with the clearest gains at moderate–high compression and under training conditions that induce flatter, more structured solutions (\eg moderate learning rates, SAM). The gap narrows at very low compression and can shrink under strong data augmentation or large SAM radii, but the overall trend is robust across optimizers and hyperparameters.

\textbf{Limitations.}
Our theoretical guarantee allows a one-component increase in compressed rank but does not establish universal dominance at exactly matched sizes. Empirically, we focus on standard CNN and ViT families on CIFAR-10 and ImageNet-1K, as well as small LLaMA models on C4. For ViTs and LLaMA, pruning and folding are applied only to the FFN blocks. Extensions to attention layers is left for future work. We evaluate in strictly calibration-free settings, with optional BatchNorm/LayerNorm resets and short fine-tuning budgets, and compare primarily against magnitude-based structured pruning. Interactions with quantization, distillation, and unstructured sparsity are not considered. Larger LLMs are beyond the scope of this study due to the computational cost of training across diverse hyperparameter settings. We note that most SoTA pruning methods for LLMs rely on calibration data (\eg activation-aware/second-order) and are exclusively pruning-based.

\textbf{Outlook.}
We plan to extend folding to pruning\,/\,folding attention blocks, calibration-based settings and evaluate on larger LLMs\,/\,VLMs. We also plan to study interactions with quantization and adaptation methods. More broadly, our projection-based view positions folding as a geometry-aware primitive for compression: a foundation on which novel calibration-based compression methods, hybrid pipelines with quantization and distillation can be built, and a step toward principled model compression framework. In this sense, folding is not only a practical tool but a building block for the next generation of compression methods tailored to foundation models and deployment at scale.

\footnotetext{Note that the base accuracy of ViT-B/32 checkpoints trained without RAUG is lower than with RAUG.}

\paragraph{Reproducibility Statement.}
Our compression operators and evaluation protocol are described in \Secref{sec:theory}-\Secref{sec:exp}, with ablation studies in \Secref{sec:ablations}. A repository with configs and scripts to regenerate all figures/tables is linked in Appendix~\ref{appx:resources}. Complete proofs are in Appendix~\ref{appx:theory}. Related literature is summarized in Appendix~\ref{appx:related}. Training setups, datasets, links to the used checkpoints, and hyperparameter grids in Appendix~\ref{appx:hyperparameters}, and extended results in Appendix~\ref{appx:further_results} and Appendix~\ref{appx:rebuttal}.  Our limited LLM usage statement is in Appendix~\ref{appx:use_llms}. Together, these materials enable re-running the full pipeline and regenerating the results.

\paragraph{Acknowledgments.}
This work was in part supported by the FFG COMET K1 Centre “Pro$^{2}$Future II” (Cognitive and Sustainable Products and Production Systems of the Future; Contract No. 911655). The results presented in this paper were obtained using the computational resources of Pro2Future GmbH, and the Austrian Scientific Computing (ASC) infrastructure.

\bibliography{arxiv.bib}

@inproceedings{wang2025forget,
    title={Forget the Data and Fine-Tuning! Just Fold the Network to Compress},
    author={Dong Wang and Haris {\v{S}}iki{\'c} and Lothar Thiele and Olga Saukh},
    booktitle={The Thirteenth International Conference on Learning Representations},
    year={2025},
    url={https://openreview.net/forum?id=W2Wkp9MQsF}
}

@misc{jordan2023repairrenormalizingpermutedactivations,
      title={REPAIR: REnormalizing Permuted Activations for Interpolation Repair}, 
      author={Keller Jordan and Hanie Sedghi and Olga Saukh and Rahim Entezari and Behnam Neyshabur},
      year={2023},
      eprint={2211.08403},
      archivePrefix={arXiv},
      primaryClass={cs.LG},
      url={https://arxiv.org/abs/2211.08403}, 
}

@misc{andriushchenko2023modernlookrelationshipsharpness,
      title={A Modern Look at the Relationship between Sharpness and Generalization}, 
      author={Maksym Andriushchenko and Francesco Croce and Maximilian Müller and Matthias Hein and Nicolas Flammarion},
      year={2023},
      eprint={2302.07011},
      archivePrefix={arXiv},
      primaryClass={cs.LG},
      url={https://arxiv.org/abs/2302.07011}, 
}

@InProceedings{pmlr-v139-radford21a,
  title = 	 {Learning Transferable Visual Models From Natural Language Supervision},
  author =       {Radford, Alec and Kim, Jong Wook and Hallacy, Chris and Ramesh, Aditya and Goh, Gabriel and Agarwal, Sandhini and Sastry, Girish and Askell, Amanda and Mishkin, Pamela and Clark, Jack and Krueger, Gretchen and Sutskever, Ilya},
  booktitle = 	 {Proceedings of the 38th International Conference on Machine Learning},
  pages = 	 {8748--8763},
  year = 	 {2021},
  editor = 	 {Meila, Marina and Zhang, Tong},
  volume = 	 {139},
  series = 	 {Proceedings of Machine Learning Research},
  month = 	 {18--24 Jul},
  publisher =    {PMLR},
  url = 	 {https://proceedings.mlr.press/v139/radford21a.html},
}

@misc{ramesh2022hierarchicaltextconditionalimagegeneration,
      title={Hierarchical Text-Conditional Image Generation with CLIP Latents}, 
      author={Aditya Ramesh and Prafulla Dhariwal and Alex Nichol and Casey Chu and Mark Chen},
      year={2022},
      eprint={2204.06125},
      archivePrefix={arXiv},
      primaryClass={cs.CV},
      url={https://arxiv.org/abs/2204.06125}, 
}

@misc{wortsman2022modelsoupsaveragingweights,
      title={Model soups: averaging weights of multiple fine-tuned models improves accuracy without increasing inference time}, 
      author={Mitchell Wortsman and Gabriel Ilharco and Samir Yitzhak Gadre and Rebecca Roelofs and Raphael Gontijo-Lopes and Ari S. Morcos and Hongseok Namkoong and Ali Farhadi and Yair Carmon and Simon Kornblith and Ludwig Schmidt},
      year={2022},
      eprint={2203.05482},
      archivePrefix={arXiv},
      primaryClass={cs.LG},
      url={https://arxiv.org/abs/2203.05482}, 
}

@misc{han2015learningweightsconnectionsefficient,
      title={Learning both Weights and Connections for Efficient Neural Networks}, 
      author={Song Han and Jeff Pool and John Tran and William J. Dally},
      year={2015},
      eprint={1506.02626},
      archivePrefix={arXiv},
      primaryClass={cs.NE},
      url={https://arxiv.org/abs/1506.02626}, 
}

@misc{mishra2021acceleratingsparsedeepneural,
      title={Accelerating Sparse Deep Neural Networks}, 
      author={Asit Mishra and Jorge Albericio Latorre and Jeff Pool and Darko Stosic and Dusan Stosic and Ganesh Venkatesh and Chong Yu and Paulius Micikevicius},
      year={2021},
      eprint={2104.08378},
      archivePrefix={arXiv},
      primaryClass={cs.LG},
      url={https://arxiv.org/abs/2104.08378}, 
}

@misc{lu2023steplearningnmstructured,
      title={STEP: Learning N:M Structured Sparsity Masks from Scratch with Precondition}, 
      author={Yucheng Lu and Shivani Agrawal and Suvinay Subramanian and Oleg Rybakov and Christopher De Sa and Amir Yazdanbakhsh},
      year={2023},
      eprint={2302.01172},
      archivePrefix={arXiv},
      primaryClass={cs.LG},
      url={https://arxiv.org/abs/2302.01172}, 
}

@misc{ding2024usmlitequantizationsparsityaware,
      title={USM-Lite: Quantization and Sparsity Aware Fine-tuning for Speech Recognition with Universal Speech Models}, 
      author={Shaojin Ding and David Qiu and David Rim and Yanzhang He and Oleg Rybakov and Bo Li and Rohit Prabhavalkar and Weiran Wang and Tara N. Sainath and Zhonglin Han and Jian Li and Amir Yazdanbakhsh and Shivani Agrawal},
      year={2024},
      eprint={2312.08553},
      archivePrefix={arXiv},
      primaryClass={eess.AS},
      url={https://arxiv.org/abs/2312.08553}, 
}

@misc{bambhaniya2024progressivegradientflowrobust,
      title={Progressive Gradient Flow for Robust N:M Sparsity Training in Transformers}, 
      author={Abhimanyu Rajeshkumar Bambhaniya and Amir Yazdanbakhsh and Suvinay Subramanian and Sheng-Chun Kao and Shivani Agrawal and Utku Evci and Tushar Krishna},
      year={2024},
      eprint={2402.04744},
      archivePrefix={arXiv},
      primaryClass={cs.LG},
      url={https://arxiv.org/abs/2402.04744}, 
}

@misc{kurtic2022optimalbertsurgeonscalable,
      title={The Optimal BERT Surgeon: Scalable and Accurate Second-Order Pruning for Large Language Models}, 
      author={Eldar Kurtic and Daniel Campos and Tuan Nguyen and Elias Frantar and Mark Kurtz and Benjamin Fineran and Michael Goin and Dan Alistarh},
      year={2022},
      eprint={2203.07259},
      archivePrefix={arXiv},
      primaryClass={cs.CL},
      url={https://arxiv.org/abs/2203.07259}, 
}

@misc{sanh2020movementpruningadaptivesparsity,
      title={Movement Pruning: Adaptive Sparsity by Fine-Tuning}, 
      author={Victor Sanh and Thomas Wolf and Alexander M. Rush},
      year={2020},
      eprint={2005.07683},
      archivePrefix={arXiv},
      primaryClass={cs.CL},
      url={https://arxiv.org/abs/2005.07683}, 
}

@inproceedings{rouhani_microscaling_2023,
 author = {Darvish Rouhani, Bita and Lo, Daniel and Zhao, Ritchie and Liu, Ming and Fowers, Jeremy and Ovtcharov , Kalin and Vinogradsky, Anna and Massengill , Sarah and Yang, Lita and Bittner, Ray and Forin, Alessandro and Zhu, Haishan and Na, Taesik and Patel, Prerak and Che, Shuai and Chand Koppaka , Lok and SONG, XIA and Som, Subhojit  and Das, Kaustav  and T, Saurabh and Reinhardt , Steve and Lanka, Sitaram and Chung, Eric and Burger, Doug},
 booktitle = {Advances in Neural Information Processing Systems},
 editor = {H. Larochelle and M. Ranzato and R. Hadsell and M.F. Balcan and H. Lin},
 pages = {10271--10281},
 publisher = {Curran Associates, Inc.},
 title = {Pushing the Limits of Narrow Precision Inferencing at Cloud Scale with Microsoft Floating Point},
 url = {https://proceedings.neurips.cc/paper_files/paper/2020/file/747e32ab0fea7fbd2ad9ec03daa3f840-Paper.pdf},
 volume = {33},
 year = {2020}
}

@INPROCEEDINGS{zhang_fast_2022,
  author={Qian Zhang, Sai and McDanel, Bradley and Kung, H. T.},
  booktitle={2022 IEEE International Symposium on High-Performance Computer Architecture (HPCA)}, 
  title={FAST: DNN Training Under Variable Precision Block Floating Point with Stochastic Rounding}, 
  year={2022},
  volume={},
  number={},
  pages={846-860},
  doi={10.1109/HPCA53966.2022.00067}
}

@article{Yao_2019,
   title={Balanced Sparsity for Efficient DNN Inference on GPU},
   volume={33},
   ISSN={2159-5399},
   url={http://dx.doi.org/10.1609/aaai.v33i01.33015676},
   DOI={10.1609/aaai.v33i01.33015676},
   number={01},
   journal={Proceedings of the AAAI Conference on Artificial Intelligence},
   publisher={Association for the Advancement of Artificial Intelligence (AAAI)},
   author={Yao, Zhuliang and Cao, Shijie and Xiao, Wencong and Zhang, Chen and Nie, Lanshun},
   year={2019},
   month=jul, pages={5676–5683} 
}

@article{Kang_2020,
   title={Accelerator-Aware Pruning for Convolutional Neural Networks},
   ISSN={1558-2205},
   url={http://dx.doi.org/10.1109/TCSVT.2019.2911674},
   DOI={10.1109/tcsvt.2019.2911674},
   journal={IEEE Transactions on Circuits and Systems for Video Technology},
   publisher={Institute of Electrical and Electronics Engineers (IEEE)},
   author={Kang, Hyeong-Ju},
   year={2020},
   pages={1–1} 
}

@misc{frantar2023sparsegptmassivelanguagemodels,
      title={SparseGPT: Massive Language Models Can Be Accurately Pruned in One-Shot}, 
      author={Elias Frantar and Dan Alistarh},
      year={2023},
      eprint={2301.00774},
      archivePrefix={arXiv},
      primaryClass={cs.LG},
      url={https://arxiv.org/abs/2301.00774}, 
}

@misc{sun2024simpleeffectivepruningapproach,
      title={A Simple and Effective Pruning Approach for Large Language Models}, 
      author={Mingjie Sun and Zhuang Liu and Anna Bair and J. Zico Kolter},
      year={2024},
      eprint={2306.11695},
      archivePrefix={arXiv},
      primaryClass={cs.CL},
      url={https://arxiv.org/abs/2306.11695}, 
}

@misc{bauckhage2015,
      title={{K-means clustering is matrix factorization}}, 
      author={Christian Bauckhage},
      year={2015},
      howpublished={arXiv:1512.07548},
      url={https://arxiv.org/abs/1512.07548}
}

@misc{saikumar2025signalcollapseoneshotpruning,
      title={Signal Collapse in One-Shot Pruning: When Sparse Models Fail to Distinguish Neural Representations}, 
      author={Dhananjay Saikumar and Blesson Varghese},
      year={2025},
      eprint={2502.15790},
      archivePrefix={arXiv},
      primaryClass={cs.LG},
      url={https://arxiv.org/abs/2502.15790}, 
}

@misc{foret2021sharpnessawareminimizationefficientlyimproving,
      title={Sharpness-Aware Minimization for Efficiently Improving Generalization}, 
      author={Pierre Foret and Ariel Kleiner and Hossein Mobahi and Behnam Neyshabur},
      year={2021},
      eprint={2010.01412},
      archivePrefix={arXiv},
      primaryClass={cs.LG},
      url={https://arxiv.org/abs/2010.01412}, 
}

@misc{prabhu2019understandingadversarialrobustnessloss,
      title={Understanding Adversarial Robustness Through Loss Landscape Geometries}, 
      author={Vinay Uday Prabhu and Dian Ang Yap and Joyce Xu and John Whaley},
      year={2019},
      eprint={1907.09061},
      archivePrefix={arXiv},
      primaryClass={cs.LG},
      url={https://arxiv.org/abs/1907.09061}, 
}

@misc{fort2019largescalestructureneural,
      title={Large Scale Structure of Neural Network Loss Landscapes}, 
      author={Stanislav Fort and Stanislaw Jastrzebski},
      year={2019},
      eprint={1906.04724},
      archivePrefix={arXiv},
      primaryClass={cs.LG},
      url={https://arxiv.org/abs/1906.04724}, 
}

@inproceedings{li2018visualizinglosslandscapes,
    author = {Li, Hao and Xu, Zheng and Taylor, Gavin and Studer, Christoph and Goldstein, Tom},
    title = {Visualizing the loss landscape of neural nets},
    year = {2018},
    publisher = {Curran Associates Inc.},
    address = {Red Hook, NY, USA},
    booktitle = {Proceedings of the 32nd International Conference on Neural Information Processing Systems},
    pages = {6391–6401},
    numpages = {11},
    location = {Montr\'{e}al, Canada},
    series = {NIPS'18}
}

@misc{neyshabur2017exploringgeneralizationdeeplearning,
      title={Exploring Generalization in Deep Learning}, 
      author={Behnam Neyshabur and Srinadh Bhojanapalli and David McAllester and Nathan Srebro},
      year={2017},
      eprint={1706.08947},
      archivePrefix={arXiv},
      primaryClass={cs.LG},
      url={https://arxiv.org/abs/1706.08947}, 
}

@misc{chen2022visiontransformersoutperformresnets,
      title={When Vision Transformers Outperform ResNets without Pre-training or Strong Data Augmentations}, 
      author={Xiangning Chen and Cho-Jui Hsieh and Boqing Gong},
      year={2022},
      eprint={2106.01548},
      archivePrefix={arXiv},
      primaryClass={cs.CV},
      url={https://arxiv.org/abs/2106.01548}, 
}

@misc{jastrzębski2018factorsinfluencingminimasgd,
      title={Three Factors Influencing Minima in SGD}, 
      author={Stanisław Jastrzębski and Zachary Kenton and Devansh Arpit and Nicolas Ballas and Asja Fischer and Yoshua Bengio and Amos Storkey},
      year={2018},
      eprint={1711.04623},
      archivePrefix={arXiv},
      primaryClass={cs.LG},
      url={https://arxiv.org/abs/1711.04623}, 
}

@misc{zhou2021theoreticallyunderstandingsgdgeneralizes,
      title={Towards Theoretically Understanding Why SGD Generalizes Better Than ADAM in Deep Learning}, 
      author={Pan Zhou and Jiashi Feng and Chao Ma and Caiming Xiong and Steven Hoi and Weinan E},
      year={2021},
      eprint={2010.05627},
      archivePrefix={arXiv},
      primaryClass={cs.LG},
      url={https://arxiv.org/abs/2010.05627}, 
}

@misc{wilson2018marginalvalueadaptivegradient,
      title={The Marginal Value of Adaptive Gradient Methods in Machine Learning}, 
      author={Ashia C. Wilson and Rebecca Roelofs and Mitchell Stern and Nathan Srebro and Benjamin Recht},
      year={2018},
      eprint={1705.08292},
      archivePrefix={arXiv},
      primaryClass={stat.ML},
      url={https://arxiv.org/abs/1705.08292}, 
}

@misc{yoo2025flatminimaperspectiveunderstanding,
      title={A Flat Minima Perspective on Understanding Augmentations and Model Robustness}, 
      author={Weebum Yoo and Sung Whan Yoon},
      year={2025},
      eprint={2505.24592},
      archivePrefix={arXiv},
      primaryClass={cs.LG},
      url={https://arxiv.org/abs/2505.24592}, 
}

@article{Hartigan1979,
  author = {Hartigan, J. A. and Wong, M. A.},
  journal = {JSTOR: Applied Statistics},
  number = 1,
  pages = {100--108},
  title = {A k-means clustering algorithm},
  volume = 28,
  year = 1979
}

@article{raffel2020exploring,
  title={Exploring the limits of transfer learning with a unified text-to-text transformer},
  author={Raffel, Colin and Shazeer, Noam and Roberts, Adam and Lee, Katherine and Narang, Sharan and Matena, Michael and Zhou, Yanqi and Li, Wei and Liu, Peter J},
  journal={Journal of machine learning research},
  volume={21},
  number={140},
  pages={1--67},
  year={2020}
}

@misc{raffel2023t5,
      title={Exploring the Limits of Transfer Learning with a Unified Text-to-Text Transformer}, 
      author={Colin Raffel and Noam Shazeer and Adam Roberts and Katherine Lee and Sharan Narang and Michael Matena and Yanqi Zhou and Wei Li and Peter J. Liu},
      year={2023},
      eprint={1910.10683},
      archivePrefix={arXiv},
      primaryClass={cs.LG},
      url={https://arxiv.org/abs/1910.10683}, 
}

@misc{sltrain,
      title={SLTrain: a sparse plus low-rank approach for parameter and memory efficient pretraining}, 
      author={Andi Han and Jiaxiang Li and Wei Huang and Mingyi Hong and Akiko Takeda and Pratik Jawanpuria and Bamdev Mishra},
      year={2024},
      eprint={2406.02214},
      archivePrefix={arXiv},
      primaryClass={cs.LG},
      url={https://arxiv.org/abs/2406.02214}, 
}

@misc{glentis2025scalableparametermemoryefficient,
      title={Scalable Parameter and Memory Efficient Pretraining for LLM: Recent Algorithmic Advances and Benchmarking}, 
      author={Athanasios Glentis and Jiaxiang Li and Qiulin Shang and Andi Han and Ioannis Tsaknakis and Quan Wei and Mingyi Hong},
      year={2025},
      eprint={2505.22922},
      archivePrefix={arXiv},
      primaryClass={cs.LG},
      url={https://arxiv.org/abs/2505.22922}, 
}

@misc{touvron2023llama2openfoundation,
      title={Llama 2: Open Foundation and Fine-Tuned Chat Models}, 
      author={Hugo Touvron and Louis Martin and Kevin Stone and Peter Albert and Amjad Almahairi and Yasmine Babaei and Nikolay Bashlykov and Soumya Batra and Prajjwal Bhargava and Shruti Bhosale and Dan Bikel and Lukas Blecher and Cristian Canton Ferrer and Moya Chen and Guillem Cucurull and David Esiobu and Jude Fernandes and Jeremy Fu and Wenyin Fu and Brian Fuller and Cynthia Gao and Vedanuj Goswami and Naman Goyal and Anthony Hartshorn and Saghar Hosseini and Rui Hou and Hakan Inan and Marcin Kardas and Viktor Kerkez and Madian Khabsa and Isabel Kloumann and Artem Korenev and Punit Singh Koura and Marie-Anne Lachaux and Thibaut Lavril and Jenya Lee and Diana Liskovich and Yinghai Lu and Yuning Mao and Xavier Martinet and Todor Mihaylov and Pushkar Mishra and Igor Molybog and Yixin Nie and Andrew Poulton and Jeremy Reizenstein and Rashi Rungta and Kalyan Saladi and Alan Schelten and Ruan Silva and Eric Michael Smith and Ranjan Subramanian and Xiaoqing Ellen Tan and Binh Tang and Ross Taylor and Adina Williams and Jian Xiang Kuan and Puxin Xu and Zheng Yan and Iliyan Zarov and Yuchen Zhang and Angela Fan and Melanie Kambadur and Sharan Narang and Aurelien Rodriguez and Robert Stojnic and Sergey Edunov and Thomas Scialom},
      year={2023},
      eprint={2307.09288},
      archivePrefix={arXiv},
      primaryClass={cs.CL},
      url={https://arxiv.org/abs/2307.09288}, 
}

@misc{touvron2023llamaopenefficientfoundation,
      title={LLaMA: Open and Efficient Foundation Language Models}, 
      author={Hugo Touvron and Thibaut Lavril and Gautier Izacard and Xavier Martinet and Marie-Anne Lachaux and Timothée Lacroix and Baptiste Rozière and Naman Goyal and Eric Hambro and Faisal Azhar and Aurelien Rodriguez and Armand Joulin and Edouard Grave and Guillaume Lample},
      year={2023},
      eprint={2302.13971},
      archivePrefix={arXiv},
      primaryClass={cs.CL},
      url={https://arxiv.org/abs/2302.13971}, 
}

@misc{bair2024adaptivesharpnessawarepruningrobust,
      title={Adaptive Sharpness-Aware Pruning for Robust Sparse Networks}, 
      author={Anna Bair and Hongxu Yin and Maying Shen and Pavlo Molchanov and Jose Alvarez},
      year={2024},
      eprint={2306.14306},
      archivePrefix={arXiv},
      primaryClass={cs.LG},
      url={https://arxiv.org/abs/2306.14306}, 
}

@misc{zhang2025sparseprunedeepnetwork,
      title={How Sparse Can We Prune A Deep Network: A Fundamental Limit Perspective}, 
      author={Qiaozhe Zhang and Ruijie Zhang and Jun Sun and Yingzhuang Liu},
      year={2025},
      eprint={2306.05857},
      archivePrefix={arXiv},
      primaryClass={stat.ML},
      url={https://arxiv.org/abs/2306.05857}, 
}

@misc{hinton2015distillingknowledgeneuralnetwork,
      title={Distilling the Knowledge in a Neural Network}, 
      author={Geoffrey Hinton and Oriol Vinyals and Jeff Dean},
      year={2015},
      eprint={1503.02531},
      archivePrefix={arXiv},
      primaryClass={stat.ML},
      url={https://arxiv.org/abs/1503.02531}, 
}

@misc{micaelli2019zeroshotknowledgetransferadversarial,
      title={Zero-shot Knowledge Transfer via Adversarial Belief Matching}, 
      author={Paul Micaelli and Amos Storkey},
      year={2019},
      eprint={1905.09768},
      archivePrefix={arXiv},
      primaryClass={cs.LG},
      url={https://arxiv.org/abs/1905.09768}, 
}

@misc{chen2019datafreelearningstudentnetworks,
      title={Data-Free Learning of Student Networks}, 
      author={Hanting Chen and Yunhe Wang and Chang Xu and Zhaohui Yang and Chuanjian Liu and Boxin Shi and Chunjing Xu and Chao Xu and Qi Tian},
      year={2019},
      eprint={1904.01186},
      archivePrefix={arXiv},
      primaryClass={cs.LG},
      url={https://arxiv.org/abs/1904.01186}, 
}

@misc{fang2020datafreeadversarialdistillation,
      title={Data-Free Adversarial Distillation}, 
      author={Gongfan Fang and Jie Song and Chengchao Shen and Xinchao Wang and Da Chen and Mingli Song},
      year={2020},
      eprint={1912.11006},
      archivePrefix={arXiv},
      primaryClass={cs.LG},
      url={https://arxiv.org/abs/1912.11006}, 
}

@inproceedings{yu2023data,
  title={Data-free knowledge distillation via feature exchange and activation region constraint},
  author={Yu, Shikang and Chen, Jiachen and Han, Hu and Jiang, Shuqiang},
  booktitle={Proceedings of the IEEE/CVF Conference on Computer Vision and Pattern Recognition},
  pages={24266--24275},
  year={2023}
}

@misc{haroush2020knowledgewithinmethodsdatafree,
      title={The Knowledge Within: Methods for Data-Free Model Compression}, 
      author={Matan Haroush and Itay Hubara and Elad Hoffer and Daniel Soudry},
      year={2020},
      eprint={1912.01274},
      archivePrefix={arXiv},
      primaryClass={cs.LG},
      url={https://arxiv.org/abs/1912.01274}, 
}

@misc{ren2023lowrankpruneandfactorizelanguagemodel,
      title={Low-Rank Prune-And-Factorize for Language Model Compression}, 
      author={Siyu Ren and Kenny Q. Zhu},
      year={2023},
      eprint={2306.14152},
      archivePrefix={arXiv},
      primaryClass={cs.CL},
      url={https://arxiv.org/abs/2306.14152}, 
}

@misc{horvath2024maestrouncoveringlowrankstructures,
      title={Maestro: Uncovering Low-Rank Structures via Trainable Decomposition}, 
      author={Samuel Horvath and Stefanos Laskaridis and Shashank Rajput and Hongyi Wang},
      year={2024},
      eprint={2308.14929},
      archivePrefix={arXiv},
      primaryClass={cs.LG},
      url={https://arxiv.org/abs/2308.14929}, 
}

@misc{lebedev2015speedingupconvolutionalneuralnetworks,
      title={Speeding-up Convolutional Neural Networks Using Fine-tuned CP-Decomposition}, 
      author={Vadim Lebedev and Yaroslav Ganin and Maksim Rakhuba and Ivan Oseledets and Victor Lempitsky},
      year={2015},
      eprint={1412.6553},
      archivePrefix={arXiv},
      primaryClass={cs.CV},
      url={https://arxiv.org/abs/1412.6553}, 
}

@misc{kim2016compressiondeepconvolutionalneural,
      title={Compression of Deep Convolutional Neural Networks for Fast and Low Power Mobile Applications}, 
      author={Yong-Deok Kim and Eunhyeok Park and Sungjoo Yoo and Taelim Choi and Lu Yang and Dongjun Shin},
      year={2016},
      eprint={1511.06530},
      archivePrefix={arXiv},
      primaryClass={cs.CV},
      url={https://arxiv.org/abs/1511.06530}, 
}

@article{li2016pruning,
      title={Pruning filters for efficient convnets},
      author={Li, Hao and Kadav, Asim and Durdanovic, Igor and Samet, Hanan and Graf, Hans Peter},
      journal={arXiv preprint arXiv:1608.08710},
      year={2016}
}

@InProceedings{luo2017iccv,
    author = {Luo, Jian-Hao and Wu, Jianxin and Lin, Weiyao},
    title = {ThiNet: A Filter Level Pruning Method for Deep Neural Network Compression},
    booktitle = {Proceedings of the IEEE International Conference on Computer Vision (ICCV)},
    month = {Oct},
    year = {2017}
}

@misc{hu2016networktrimming,
      title={Network Trimming: A Data-Driven Neuron Pruning Approach towards Efficient Deep Architectures}, 
      author={Hengyuan Hu and Rui Peng and Yu-Wing Tai and Chi-Keung Tang},
      year={2016},
      eprint={1607.03250},
      archivePrefix={arXiv},
      primaryClass={cs.NE},
      url={https://arxiv.org/abs/1607.03250}, 
}

@article{wen2016learning,
  title={Learning structured sparsity in deep neural networks},
  author={Wen, Wei and Wu, Chunpeng and Wang, Yandan and Chen, Yiran and Li, Hai},
  journal={Advances in neural information processing systems},
  volume={29},
  year={2016}
}

@misc{entezari2020classdependentcompressiondeepneural,
      title={Class-dependent Compression of Deep Neural Networks}, 
      author={Rahim Entezari and Olga Saukh},
      year={2020},
      eprint={1909.10364},
      archivePrefix={arXiv},
      primaryClass={cs.LG},
      url={https://arxiv.org/abs/1909.10364}, 
}

@article{theus2024metapruning,
  title={Towards Meta-Pruning via Optimal Transport},
  author={Theus, Alexander and Geimer, Olin and Wicke, Friedrich and Hofmann, Thomas and Anagnostidis, Sotiris and Singh, Sidak Pal},
  journal={arXiv preprint arXiv:2402.07839},
  year={2024}
}

@misc{wortsman2022model,
      title={Model soups: averaging weights of multiple fine-tuned models improves accuracy without increasing inference time}, 
      author={Mitchell Wortsman and Gabriel Ilharco and Samir Yitzhak Gadre and Rebecca Roelofs and Raphael Gontijo-Lopes and Ari S. Morcos and Hongseok Namkoong and Ali Farhadi and Yair Carmon and Simon Kornblith and Ludwig Schmidt},
      year={2022},
      eprint={2203.05482},
      archivePrefix={arXiv},
      primaryClass={cs.LG}
}

@misc{entezari2022role,
      title={The Role of Permutation Invariance in Linear Mode Connectivity of Neural Networks}, 
      author={Rahim Entezari and Hanie Sedghi and Olga Saukh and Behnam Neyshabur},
      year={2022},
      eprint={2110.06296},
      archivePrefix={arXiv},
      primaryClass={cs.LG},
      url={https://arxiv.org/pdf/2110.06296}
}

@misc{ainsworth2023git,
      title={Git {R}e-{B}asin: Merging Models modulo Permutation Symmetries}, 
      author={Samuel K. Ainsworth and Jonathan Hayase and Siddhartha Srinivasa},
      year={2023},
      eprint={2209.04836},
      archivePrefix={arXiv},
      primaryClass={cs.LG},
      url={https://arxiv.org/abs/2209.04836}
}

@article{chen2023going, 
    title={Going Beyond Neural Network Feature Similarity: The Network Feature Complexity and Its Interpretation Using Category Theory}, 
    author={Chen, Yiting and Zhou, Zhanpeng and Yan, Junchi}, journal={arXiv preprint arXiv:2310.06756}, year={2023} 
}

@misc{stoica2024zipitmergingmodelsdifferent,
      title={ZipIt! Merging Models from Different Tasks without Training}, 
      author={George Stoica and Daniel Bolya and Jakob Bjorner and Pratik Ramesh and Taylor Hearn and Judy Hoffman},
      year={2024},
      eprint={2305.03053},
      archivePrefix={arXiv},
      primaryClass={cs.CV},
      url={https://arxiv.org/abs/2305.03053}, 
}

\clearpage
\appendix
\section*{Appendix}
The following sections provide supplementary information and complement the main paper:
\begin{itemize}
    \item Appendix~\ref{appx:resources}: Code, Data, and Resources.
    \item Appendix~\ref{appx:theory}: Proofs of Theoretical Claims.
    \item Appendix~\ref{appx:related}: Related Work.
    \item Appendix~\ref{appx:hyperparameters}: Training Details.
    \item Appendix~\ref{appx:further_results}: Extended Empirical Comparison of Folding and Pruning.
    \item Appendix~\ref{appx:rebuttal}: Additional Analyses: Sharpness, Runtime, and LLMs.
    \item Appendix~\ref{appx:use_llms}: Use of Large Language Models.
\end{itemize}

\section{Code, Data, and Resources}
\label{appx:resources}

\textbf{Code and logs.}
A repository with all source code, experiment configs, and figure-generation scripts (including the exact logs used to render every plot/table) are released at \url{https://github.com/osaukh/folding_as_projection}. The repo contains: implementations of folding and pruning operators, training/evaluation pipelines, scripts to plot ablations, and notebooks to reproduce figures directly from logs. We log all training metrics and hyperparameters with Weights \& Biases\footnote{Weights \& Biases: \url{https://wandb.ai}} and export logs alongside the code for reproduction. Additionally, we provide another repository for reproducing results of compressing LLaMA-60M and LLaMA-130M with folding and magnitude structured pruning at \url{https://github.com/nanguoyu/simple_model_folding_public}. Our folding implementation is based on the code by~\citet{wang2025forget}\footnote{Model folding universal: \url{https://github.com/nanguoyu/model-folding-universal} and model folding for CNNs: \url{https://github.com/marza96/ModelFolding/}}.

\textbf{Datasets.}
We use CIFAR-10\footnote{CIFAR-10: \url{https://www.cs.toronto.edu/\~kriz/cifar.html}} and ImageNet-1K\footnote{ImageNet-1K: \url{https://image-net.org/}}. CIFAR-10 is downloaded automatically via \texttt{torchvision}. ImageNet-1K requires the official credentials and follows its license. Pretrained/fine-tuned checkpoints referenced in the paper are either trained by us (configs in the repo) or obtained from the cited works~\citep{andriushchenko2023modernlookrelationshipsharpness,wortsman2022modelsoupsaveragingweights}. The download links are also provided in Appendix~\ref{appx:hyperparameters}.

\textbf{Compute resources.}
Experiments were run on a cluster featuring 8$\times$ NVIDIA A100 (80\,GB RAM) GPUs. All random seeds are fixed in the configs and scripts.

\textbf{Computational complexity and memory cost.}
At inference and matched retained sizes, folding and structured pruning yield the same compute and memory. The difference lies in the compression step: magnitude pruning is a one-pass scoring and selection procedure (\(O(pm)\) to score \(p\) filters of dimension \(m\), plus \(O(p\log p)\) selection), whereas folding runs \(k\)-means on layer weights with \(T\) sweeps. Using Hartigan’s algorithm~\citep{Hartigan1979}, one sweep costs \(O(pkm)\), with max \(T=10\) sweeps the total is \(O(pkmT)\) (effectively linear in \(pm\) when \(k\) is small). This cost is paid once per layer and is small compared to training.

\textbf{Runtime overview.}
The most expensive step in our study is fine-tuning of CLIP VIT-B/32 on ImageNet-1K (1--5 epochs), which dominates wall-clock time (order of hours per run). In contrast, compression is lightweight. We detail measured runtime overhead of compression in Appendix~\ref{appx:rebuttal}.

\section{Proofs of Theoretical Claims}
\label{appx:theory}

\renewcommand{\thetheorem}{2.\arabic{theorem}}
\setcounter{theorem}{0} 

Below we prove that for any choice of pruning, there exists a folding that yields a more accurate approximation of the parameter matrix \( \mathbf{W} \).

\begin{theorem}
Given any pruning with basis \( \mathbf{U}_p \) of rank \( 0 \leq k_p \leq m-1 \) (\ie at least one parameter vector is pruned), there exists a folding with basis \( \mathbf{U}'_f \) and rank \( k_f = k_p + 1 \) such that
\[
\lVert \mathbf{W} - \mathbf{W}_p \rVert_F^2 \geq \lVert \mathbf{W} - \mathbf{W}'_f \rVert_F^2,
\]
where \( \mathbf{W}_p = \mathbf{C}_p \mathbf{W} \) and \( \mathbf{W}'_f = \mathbf{C}'_f \mathbf{W} \), with \( \mathbf{C}_p \) and \( \mathbf{C}'_f \) denoting the orthogonal projections defined in~\Eqref{eq:proj}.
\end{theorem}

\begin{proof}
The rows of \( \mathbf{W} \) can be ordered such that the pruned parameter vectors are first: \( w(1), ..., w(m-k_p) \). Then we find that 
\[
\mathbf{W} - \mathbf{W}_p = \begin{pmatrix} w(1) \\ \cdots \\ w(m-k_p) \\ 0 \\ \cdots \\ 0
\end{pmatrix}
\]
using \Eqref{eq:pruning}. For the existence proof, we choose a folding that clusters all parameter vectors \( w(1), ..., w(m-k_p) \) into a single cluster, all other parameter vectors have individual clusters, \ie 
\[
\mathbf{U}'_f = \begin{pmatrix} 
1 & 0 \\
\cdots & 0 \\
1 & 0 \\
0 & \mathbf{I}
\end{pmatrix} \quad ; \quad
\mathbf{W} - \mathbf{W}'_f = \begin{pmatrix} w(1) - \mu \\ \cdots \\ w(m-k_p) - \mu \\ 0 \\ \cdots \\ 0
\end{pmatrix} \quad ; \quad \mu = \frac{1}{m-k_p} \sum_{i=1}^{m-k_p} w(i)
\]
using \Eqref{eq:folding}.

We have \( \lVert \mathbf{W} - \mathbf{W}_p \rVert_F^2 = \sum_{i=1}^{m-k_p} w(i)^T w(i) \) and 
\begin{align*}
 \lVert \mathbf{W} - \mathbf{W}'_f \rVert_F^2 &= \sum_{i=1}^{m-k_p} (w(i) - \mu)^T (w(i) -\mu) = \sum_{i=1}^{m-k_p} \bigl(w(i)^T w(i) - 2w(i)^T \mu + \mu^T \mu \bigr)  \\ 
 &= \sum_{i=1}^{m-k_p} w(i)^T w(i) - (m-k_p) \mu^T \mu \\
 &\leq \sum_{i=1}^{m-k_p} w(i)^T w(i) = \lVert \mathbf{W} - \mathbf{W}_p \rVert_F^2
\end{align*} 
The latter inequality directly establishes the theorem.
\end{proof}

The following theorem shows that folding using optimal \(k\)-means clustering never yields a less accurate approximation of the parameter matrix \( \mathbf{W} \) than pruning. 

\begin{theorem}
Let \( \mathbf{U}^\star_f \) be the basis obtained from an optimal \( k \)-means clustering with \( k_f \) clusters, \ie the folding clusters are determined by a \( k \)-means algorithm minimizing the accumulated within-cluster sum of squares. Then, for any pruning with basis \( \mathbf{U}_p \) of rank \( k_p = k_f - 1 \), we have
\[
\lVert \mathbf{W} - \mathbf{W}_p \rVert_F^2 \geq \lVert \mathbf{W} - \mathbf{W}^\star_f \rVert_F^2,
\]
where \( \mathbf{W}_p = \mathbf{C}_p \mathbf{W} \) and \( \mathbf{W}^\star_f = \mathbf{C}^\star_f \mathbf{W} \), with \( \mathbf{C}_p \) and \( \mathbf{C}^\star_f \) denoting the orthogonal projections defined in~\Eqref{eq:proj}.
\end{theorem}

\begin{proof}
According to \citet{bauckhage2015} and \cite{wang2025forget}, the problem of \( k \)-means clustering can be formulated as the following constrained matrix factorization problem:
\[
    \min_{\mathbf{U}} \; \left\lVert \mathbf{W} - \mathbf{U} (\mathbf{U}^\top \mathbf{U})^{-1} \mathbf{U}^\top  \mathbf{W} \right\rVert_F^2
    \quad \text{subject to} \quad 
    u(i,j) \in \{0,1\}, \; \sum_{j} u(i,j) = 1 \; \forall i.
\]

This formulation coincides with the orthogonal projection of model folding, see \Eqref{eq:proj} and \Eqref{eq:folding}. Theorem~\ref{thm:folding-existence} guarantees the existence of a folding basis \( \mathbf{U}'_f \) and the corresponding projection \( \mathbf{C}'_f \) for any pruning \( \mathbf{W}_p \) of \( \mathbf{W} \), such that
\[
    \lVert \mathbf{W} - \mathbf{W}_p \rVert_F^2 \;\; \geq \;\; \lVert \mathbf{W} - \mathbf{W}'_f \rVert_F^2.
\]
Since optimal \( k \)-means clustering achieves the minimal possible error \(\lVert \mathbf{W} - \mathbf{W}^\star_f \rVert_F^2 \;\; \leq \;\; \lVert \mathbf{W} - \mathbf{W}'_f \rVert_F^2 \), the theorem follows.
\end{proof}

\section{Related Work}
\label{appx:related}
Model compression encompasses a wide range of approaches designed to reduce inference cost while preserving model utility. We focus on \emph{post-training, calibration-free structured compression}, where the model architecture is modified without access to data or gradients. In this setting, the dominant baselines are structured pruning and, more recently, model folding. Below we discuss these families of methods and clarify how our projection-theoretic view relates to and extends prior work.

\textbf{Post-training compression.}
Quantization reduces arithmetic precision~\citep{rouhani_microscaling_2023,zhang_fast_2022}, but typically requires calibration to maintain activation ranges. Knowledge distillation~\citep{hinton2015distillingknowledgeneuralnetwork} produces reduced students trained to imitate teacher logits. Even data-free variants~\citep{micaelli2019zeroshotknowledgetransferadversarial,chen2019datafreelearningstudentnetworks,fang2020datafreeadversarialdistillation,yu2023data,haroush2020knowledgewithinmethodsdatafree} require full training dynamics and do not yield structural compression. Low-rank factorization via matrix or tensor decompositions~\citep{ren2023lowrankpruneandfactorizelanguagemodel,horvath2024maestrouncoveringlowrankstructures,lebedev2015speedingupconvolutionalneuralnetworks,kim2016compressiondeepconvolutionalneural} approximates pretrained weights by continuous subspaces but generally requires fine-tuning for restoration. These approaches differ fundamentally from our objective: they modify numerical precision or parameterization, not the discrete structure of the model.

\textbf{Structured pruning.}
Structured pruning removes neurons, channels, filters, or blocks~\citep{li2016pruning,luo2017iccv,hu2016networktrimming,wen2016learning}. Magnitude-based criteria~\citep{han2015learningweightsconnectionsefficient,lu2023steplearningnmstructured,ding2024usmlitequantizationsparsityaware,entezari2020classdependentcompressiondeepneural} dominate due to simplicity and hardware alignment. However, structured pruning typically requires fine-tuning or recalibration~\citep{kurtic2022optimalbertsurgeonscalable,sanh2020movementpruningadaptivesparsity} to mitigate accuracy degradation, and even calibration-based methods such as SparseGPT~\citep{frantar2023sparsegptmassivelanguagemodels} or Wanda~\citep{sun2024simpleeffectivepruningapproach} operate through axis-aligned removal of coordinates. One-shot improvements using N:M sparsity~\citep{Yao_2019,Kang_2020} or OT-based structural alignment~\citep{theus2024metapruning} still operate within the same paradigm: pruning corresponds to enforcing that the retained parameter vectors lie in a fixed coordinate-aligned subspace.

Our work shows that such axis-aligned projections are geometrically restrictive. We formalize pruning as an orthogonal projection onto a coordinate subspace and demonstrate that, at matched ranks up to one slack, pruning is provably dominated by projections onto cluster-structured subspaces.

\textbf{Weight clustering and model folding.}
Model folding, recently introduced by~\citet{wang2025forget}, ties groups of similar channels by replacing them with their mean, yielding dense low-rank layers that preserve structural couplings. Folding implicitly performs a \emph{cluster-structured projection} determined by discrete assignments, and practical implementations rely on k-means clustering. This operator class is strictly richer than axis-aligned pruning: folding enables coordinated merging rather than coordinate removal, while remaining compatible with dense inference. IFM~\citep{chen2023going} is related in that it also merges channels via grouping, but its variance-collapse correction is ineffective~\citep{wang2025forget}, leading to substantially weaker performance.

Our work strengthens this line along two axes. First, we provide a unified \emph{projection-geometric framework} showing that both pruning and folding are orthogonal projections, but onto fundamentally different subspaces: coordinate-aligned versus cluster-structured. Second, we prove that for any pruned solution of rank $k$, there exists a folded solution of rank $k\!+\!1$ with strictly smaller parameter reconstruction error, and that optimal k-means folding minimizes this projection error among all cluster-structured projections. This establishes a strict theoretical separation between pruning and folding and explains the empirical superiority of folding in calibration-free settings.

\textbf{Model merging and alignment.}
Model merging combines independently trained models via parameter averaging or permutation alignment. Model soups~\citep{wortsman2022model} exploit shared initialization. Permutation matching~\citep{entezari2022role,ainsworth2023git} constructs neuron correspondences. REPAIR~\citep{jordan2023repairrenormalizingpermutedactivations} stabilizes fused models by re-normalizing preactivations. Intra-model merging approaches such as ZipIt!~\citep{stoica2024zipitmergingmodelsdifferent} combine computational units but do not target compression under fixed architectural constraints.

These works differ from ours in both objective and mechanism. Merging seeks functional fusion across networks, whereas folding compresses a \emph{single} network by exploiting intra-layer redundancy. Our projection-theoretic formulation shows that folding operates as a structured projection with explicit geometric optimality guarantees—properties not shared by merging methods.

\textbf{Positioning of this work.}
Across pruning, folding, and merging, prior efforts lack a unifying mathematical framework that characterizes the geometry of post-training structural compression. Our contribution is to introduce such a framework: we cast pruning and folding as orthogonal projections and show that cluster-structured projections admit strictly smaller distortion than coordinate projections under practically negligible rank slack. This perspective yields nontrivial theoretical guarantees and aligns closely with the empirical phenomena observed across CNNs, ViTs, and LLaMA models.

\section{Training Details}
\label{appx:hyperparameters}

The following subsections detail the hyperparameters used to train our checkpoints. For checkpoints taken from the literature, we summarize the available training details.

\subsection{ResNet18 on CIFAR-10 Training Setup with Adam and SGD}
We trained a total of 792 ResNet18 models on CIFAR-10 by varying hyperparameter configurations. We used two optimizers: Adam and SGD. \Tabref{tab:training-config} summarizes the parameter combinations explored for each optimizer. For Adam, we used 3 learning rates and 1 momentum value. For SGD, we used 3 learning rates and 2 momentum values. The remaining parameters were shared across both optimizers: weight decay (3 values), L1 regularization (2 values), RandAugment (2 values), Sharpness-Aware Minimization (3 values), and learning rate scheduling (2 values). This resulted in 216 models trained with Adam and 576 models trained with SGD. In the ablation studies, we filter checkpoints (as specified in the figure captions) to highlight the observed effects.

\begin{table}[h]
\centering
\begin{tabular}{ll}
\toprule
\rowcolor{lightblue}
\textbf{Parameter} & \textbf{Values} \\ \midrule
Optimizer & \texttt{adam}, \texttt{sgd} \\
Learning Rate & \texttt{adam}: 0.1, 0.01, 0.001 \\
                & \texttt{sgd}: 0.1, 0.05, 0.01, 0.001 \\
Momentum & \texttt{adam}: 0.0 \\
                & \texttt{sgd}: 0.9, 0.99 \\
Weight Decay & 0.0, 0.0005, 0.001 \\
L1 Regularization & 0.0, $1 \times 10^{-5}$ \\
RandAugment & True, False \\
SAM (Sharpness-Aware Minimization) & None, 0.05, 0.1 \\
Learning Rate Schedule & True, False \\ \bottomrule
\end{tabular}
\caption{Hyperparameter combinations used for ResNet18 training on CIFAR-10.}
\label{tab:training-config}
\end{table}

\begin{figure}[t]
     \centering
     \vskip -0.4cm
     \subfloat[][ResNet18, Adam, \fold vs \magtwo, \textbf{no} L1 regularization]{
        \includegraphics[height=0.24\textwidth]{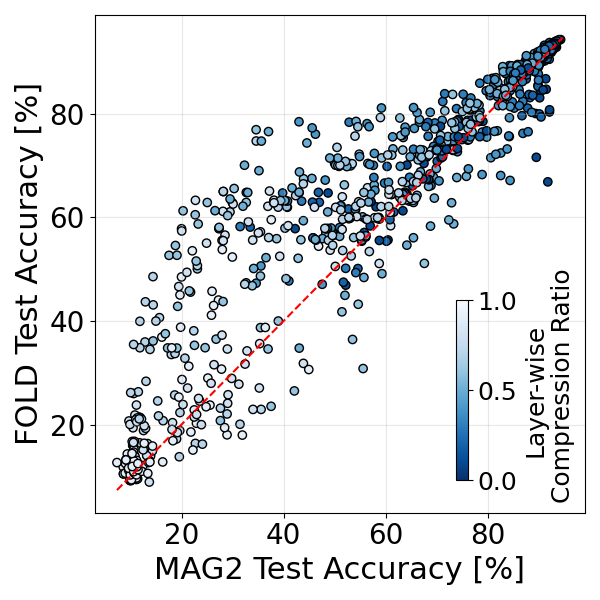}
        \includegraphics[height=0.24\textwidth]{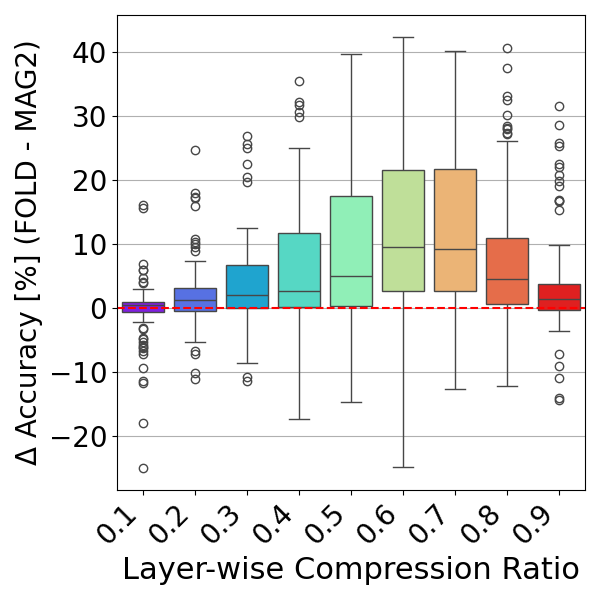}
    }
    \subfloat[][PreActResNet18, \fold vs \magtwo]{
        \includegraphics[height=0.24\textwidth]{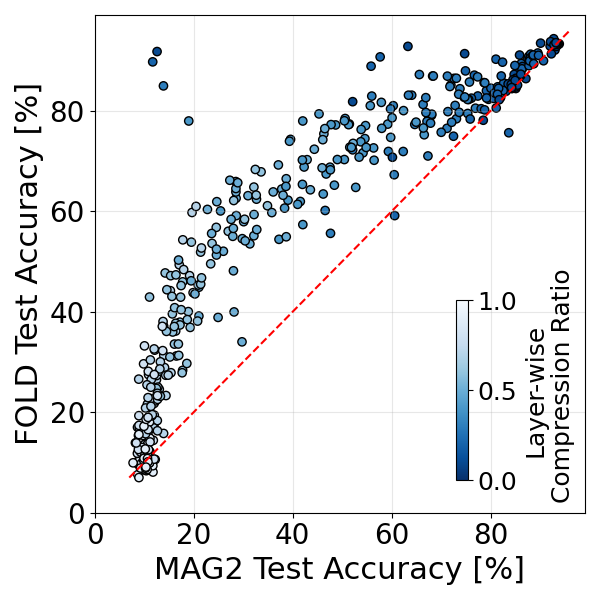}
        \includegraphics[height=0.24\textwidth]{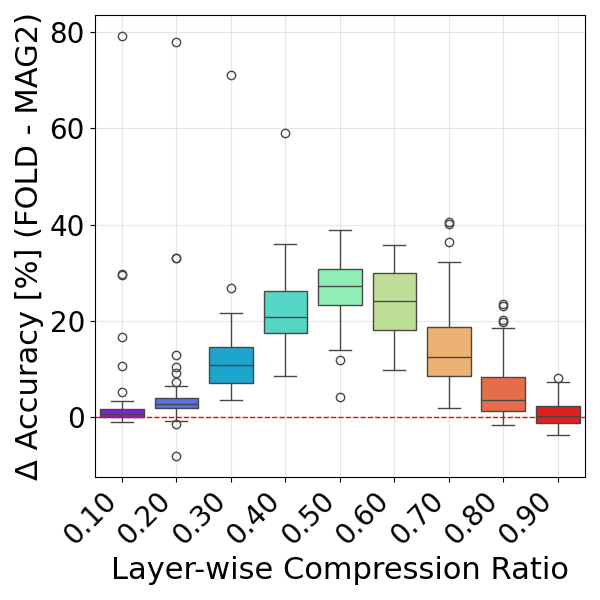}
    }

    \subfloat[][ViT-B/32, \fold vs \magtwo]{
        \includegraphics[height=0.24\textwidth]{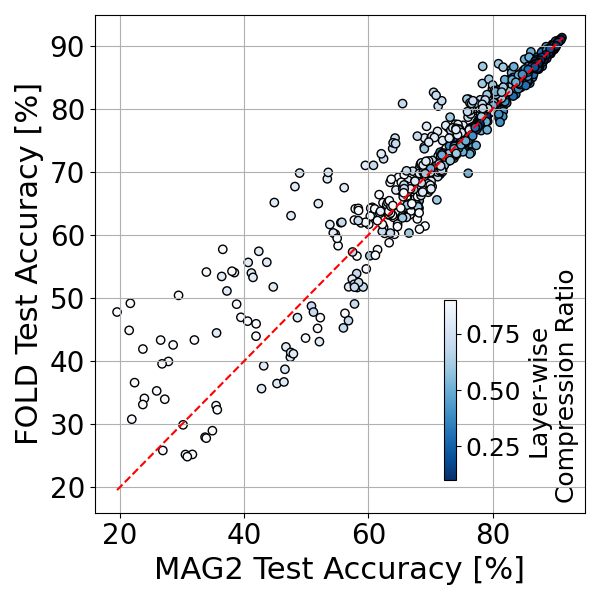}
        \includegraphics[height=0.24\textwidth]{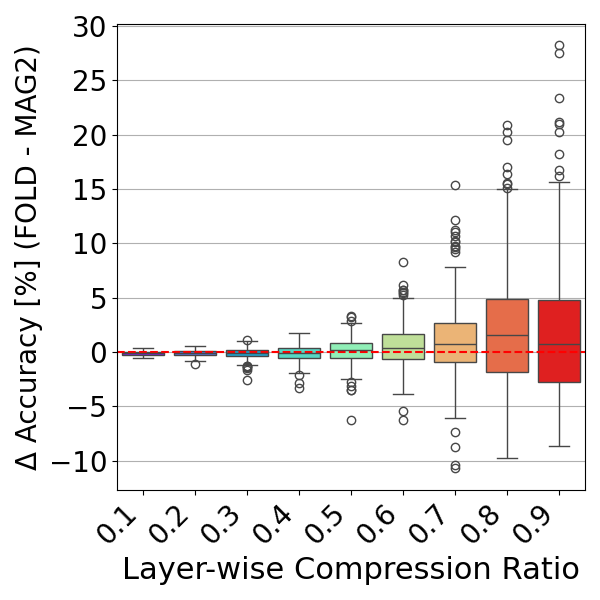}
    }
    \subfloat[][CLIP ViT-B/32, \fold vs \magtwo]{
        \includegraphics[height=0.24\textwidth]{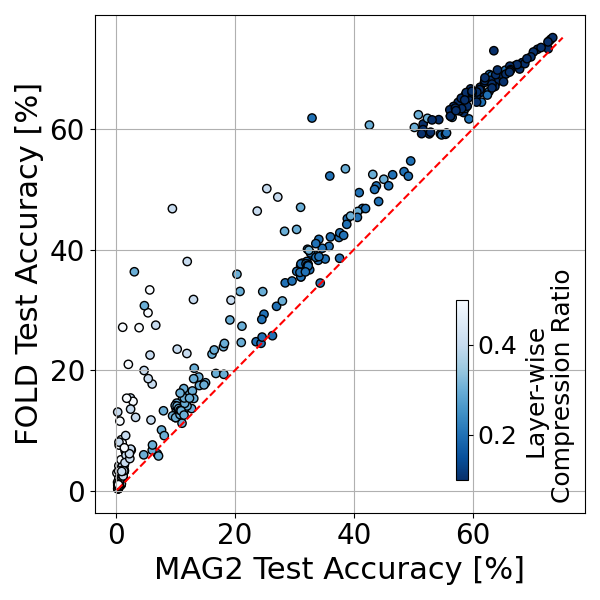}
        \includegraphics[height=0.24\textwidth]{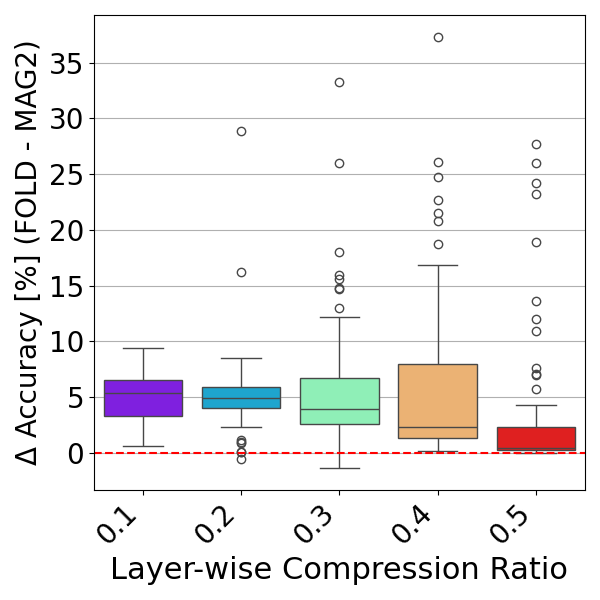}
    }
    \caption{\textbf{Folding outperforms magnitude pruning across diverse training regimes.} The same setup as in \Figref{fig:accuracy_FOLD_vs_MAG_L1}, but compared to the L2 magnitude pruning criterion.
    \textbf{Top row:} ResNet18 and PreActResNet18 on CIFAR-10. ResNet18 checkpoints were trained from scratch with Adam using different hyperparameter configurations.
    \textbf{Bottom row:} ViT-B/32 on CIFAR-10 and CLIP ViT-B/32 on ImageNet-1K. Scatter plots show post-compression accuracy for folding versus magnitude pruning (L2 criterion) at uniform per-layer compression ratios. Bar plots depict the accuracy gain by folding, computed as $\Delta=\mathrm{Acc}{(\text{\fold})}-\mathrm{Acc}{(\text{\magtwo})}$, as a function of layer-wise compression ratio. Folding yields the largest improvements at moderate to high compression, confirming its robustness across architectures and datasets.}
     \label{fig:accuracy_FOLD_vs_MAG_L2}
\end{figure}

\subsection{PreActResNet18 on CIFAR-10}
We use 50 trained PreActResNet18 models on CIFAR-10 from \citep{andriushchenko2023modernlookrelationshipsharpness}\footnote{Download link: \url{https://drive.google.com/drive/folders/1LmthJCb3RXBFWjeTOC4UOOl7Ppgg2h7n}}. The models are trained using a fixed set of training parameters and a sweep over a few key hyperparameters. \Tabref{tab:preact-config} summarizes varied parameters used in this experiment. All checkpoints used the same training protocol: 200 epochs, batch size 128, and no label noise. The model width was fixed at 64 and the learning rate schedule followed a cyclic pattern. Only the maximum learning rate (\texttt{lr\_max}), SAM strength (\texttt{sam\_rho}), and augmentation settings were varied. For the learning rate ablation studies, we adopt the reported maximum learning rate.

\begin{figure}[t]
     \centering
     \vskip -0.4cm
     \subfloat[][ViT-B/32, \magtwo vs \fold, base accuracy $>$75\%]{
        \includegraphics[height=0.24\textwidth]{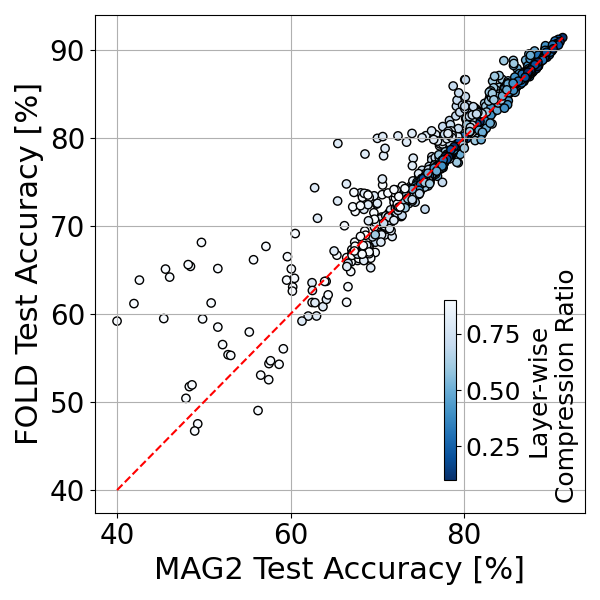}
        \includegraphics[height=0.24\textwidth]{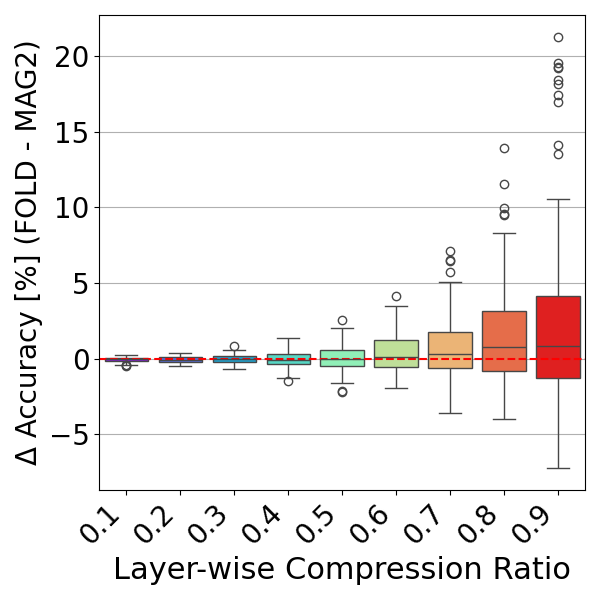}
     }
     \subfloat[][CLIP ViT-B/32, \magtwo vs \fold]{
        \includegraphics[height=0.24\textwidth]{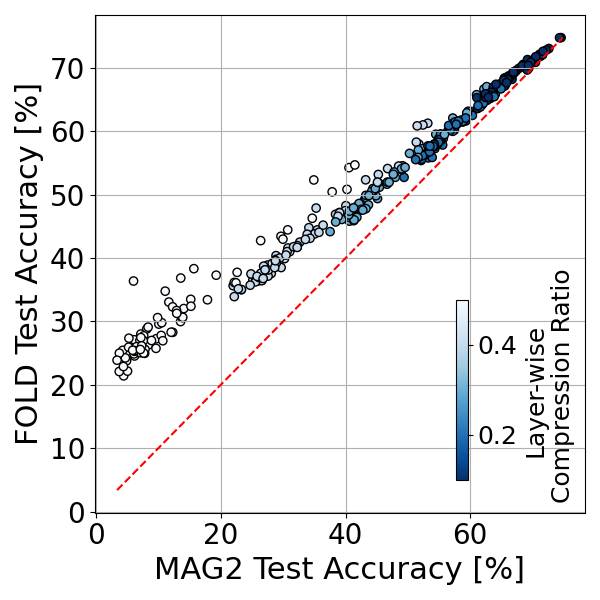}
        \includegraphics[height=0.24\textwidth]{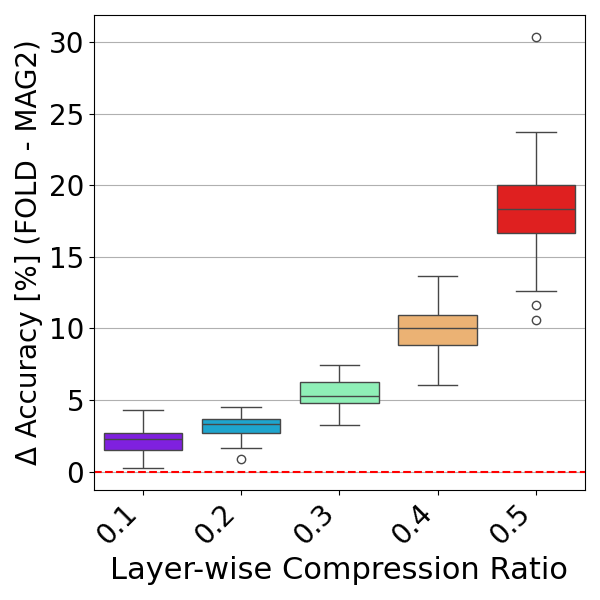}
     }
     \caption{\textbf{\fold versus \magtwo on ViTs after LayerNorm-only fine-tuning}
     for ViT-B/32 on CIFAR-10 and CLIP ViT-B/32 on ImageNet-1K. In the scatter plots, points are checkpoints, color encodes layer-wise compression. Bar plots depict the accuracy gain $\Delta=\mathrm{Acc}{(\text{\fold})}-\mathrm{Acc}{(\text{\magone})}$, which remains positive and typically grows with compression, indicating that even under lightweight LayerNorm adaptation \fold retains a consistent advantage over pruning. The figure follows the same setup as \Figref{fig:ftLN_ViTs_FOLD_vs_MAG_L1} in the main paper, but for \magtwo.
     }
     \label{fig:ftLN_ViTs_FOLD_vs_MAG_L2}
\end{figure}

\begin{table}[h]
\centering
\begin{tabular}{ll}
\rowcolor{lightblue}
\toprule
\textbf{Parameter} & \textbf{Values} \\ \midrule
Optimizer & \texttt{sgd} \\
Max / Base Learning Rate (\texttt{lr\_max}) & from 0.0504 to 4.9759 \\
SAM Strength (\texttt{sam\_rho}) & 0.0, 0.05, 0.1 \\
Standard Augmentation (\texttt{augm}) & True, False \\
RandAugment (\texttt{randaug}) & True, False \\
\bottomrule
\end{tabular}
\caption{Fixed and varying parameters for PreActResNet18 training on CIFAR-10.}
\label{tab:preact-config}
\end{table}

\begin{figure}[t]
     \centering
     \vskip -.4cm
     \subfloat[][\magone vs \fold, fine-tuning for\\1 (\textbf{left}) and 5 (\textbf{right}) epochs.]
     {\includegraphics[height=0.2\textwidth]{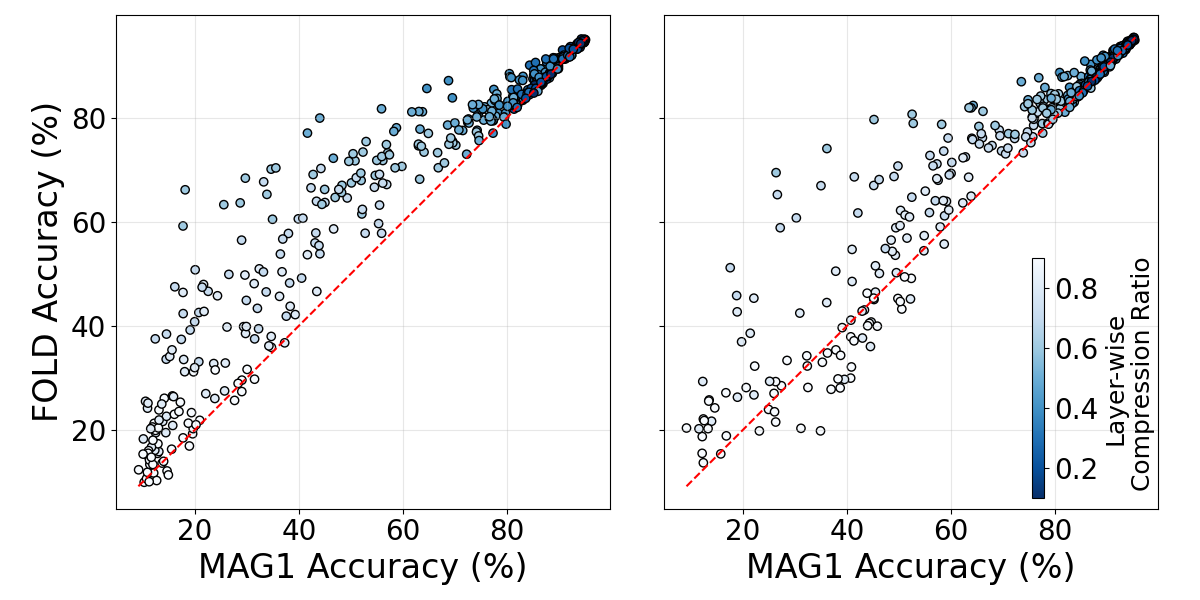}}
     \subfloat[][\magone vs \fold\\accuracy gap after ft.]{\includegraphics[height=0.2\textwidth]{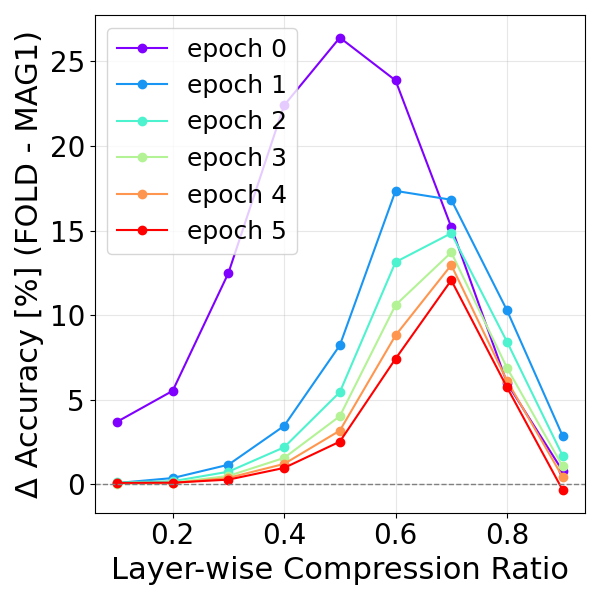}}
     \subfloat[][Before and after fine-tuning for 5 epochs: \fold (\textbf{left}) and \magone (\textbf{right}).]
     {\includegraphics[height=0.2\textwidth]{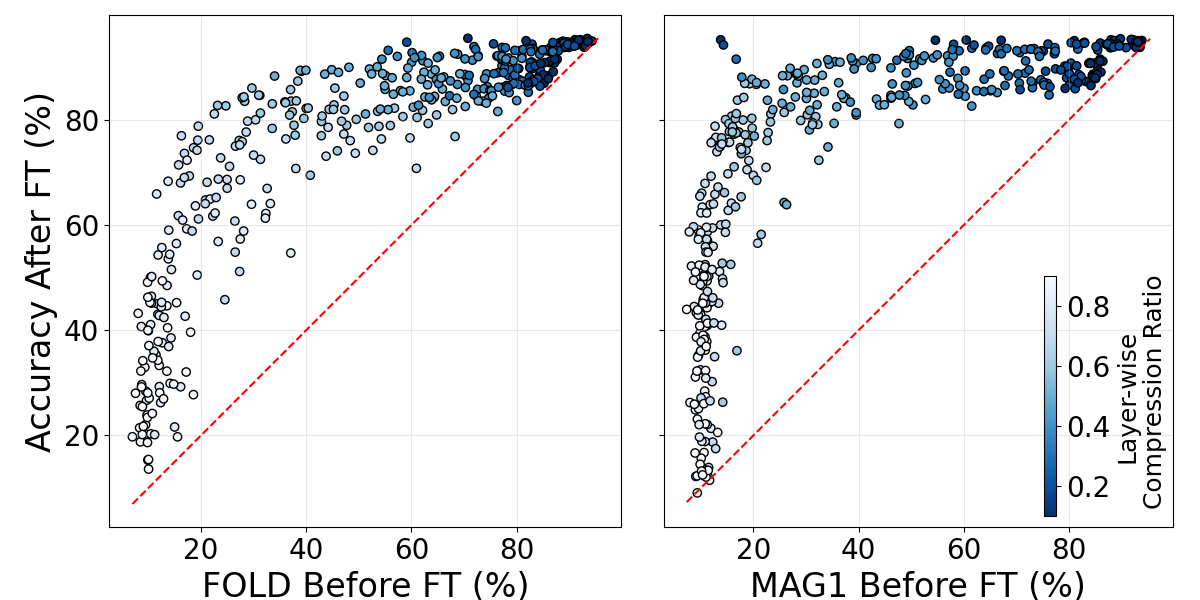}}

     \subfloat[][\magone vs \fold, fine-tuning for\\1 (\textbf{left}) and 5 (\textbf{right}) epochs.]{\includegraphics[height=0.2\textwidth]{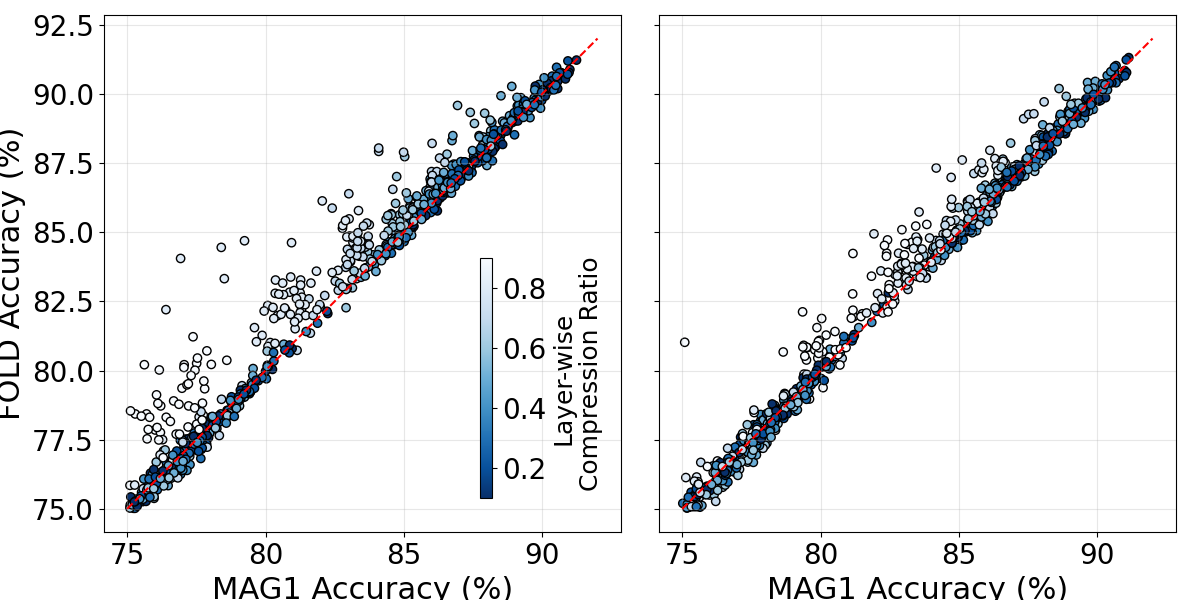}}
     \subfloat[][\magone vs \fold\\accuracy gap after ft.]{\includegraphics[height=0.2\textwidth]{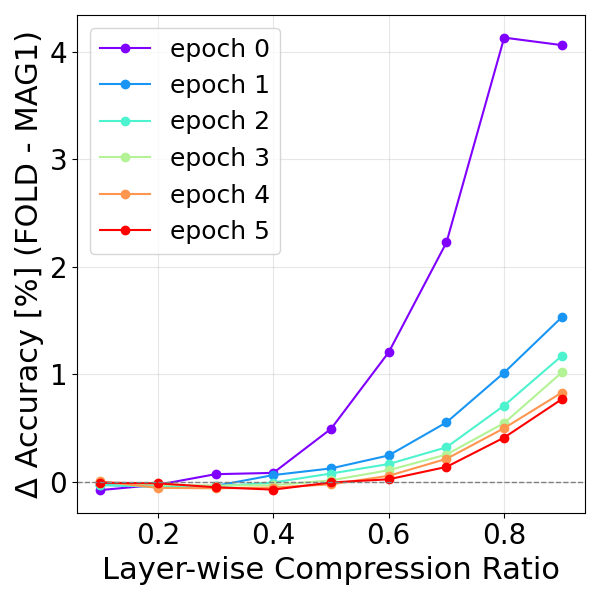}}
     \subfloat[][Before and after fine-tuning for 5 epochs: \fold (\textbf{left}) and \magone (\textbf{right}).]{\includegraphics[height=0.2\textwidth]{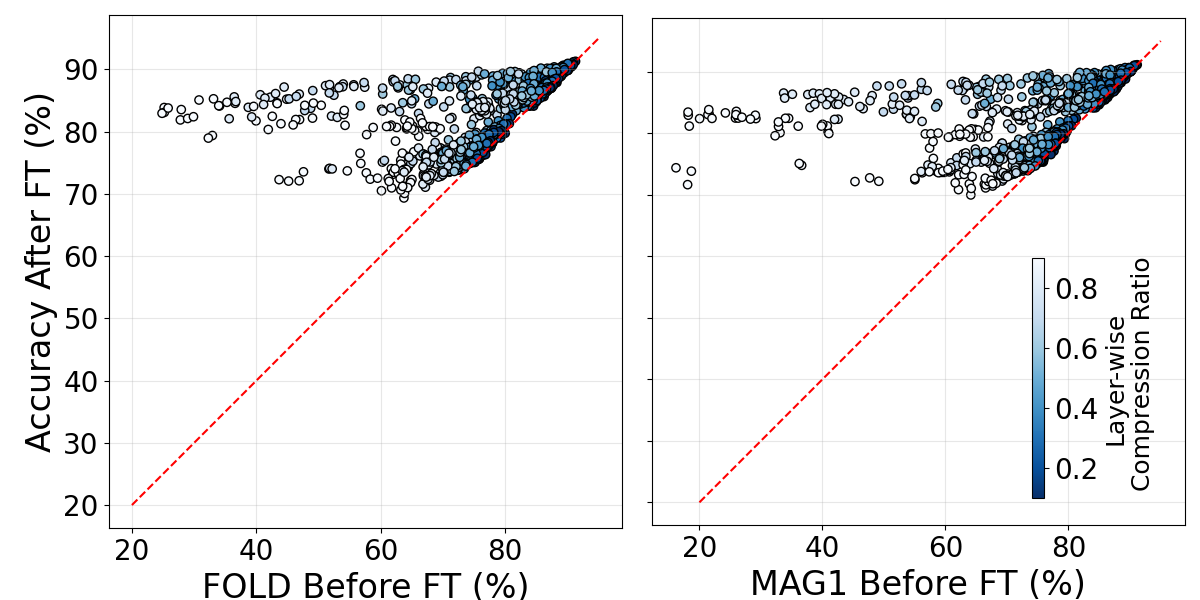}}
     
     \caption{\textbf{\fold outperforms \magone after full fine-tuning for 1--5 epochs on PreActResNet18 and ViT-B/32 on CIFAR-10.} Results for PreActResNet18 (\textbf{top}) and ViT-B/32 (\textbf{bottom}). \textbf{(a,d)} accuracy of \magone vs. \fold after 1 and 5 epochs of fine-tuning. \textbf{(b,e)} accuracy gap $\Delta$ over epochs, remaining positive. \textbf{(c,f)} accuracy trajectories from post-compression through 5 epochs, showing faster recovery and higher final accuracy for \fold. The figure extends \Figref{fig:ft_CNN_and_ViTs_1_FOLD_vs_MAG_L1} in the main paper to PreActResNet18 and ViT-B/32 architectures where \fold is benchmarked against \magone.}
     \label{fig:ft_CNN_and_ViTs_2_FOLD_vs_MAG_L1}
\end{figure}

\subsection{ViT-B/32 on CIFAR-10}
The 200 Vision Transformers (ViT) also from \citep{andriushchenko2023modernlookrelationshipsharpness}, width=256, were trained on CIFAR-10, batch size 128, for 200 epochs with a cosine learning rate schedule and linear warmup. The main hyperparameters are summarized in \Tabref{tab:vit-config}. 
We made use of the maximum learning rate, the use of data augmentation, and the use of Sharpness-Aware Minimization (SAM) in our evaluations. 
All other settings were fixed.

\begin{table}[h]
\centering
\begin{tabular}{ll}
\toprule
\rowcolor{lightblue}
\textbf{Parameter} & \textbf{Values} \\
\midrule
Optimizer & \texttt{sgd} \\
Max / Base Learning Rate (\texttt{lr\_max}) & from 0.005087 to 0.492936 \\
SAM Strength (\texttt{sam\_rho}) & 0.0, 0.05, 0.1 \\
Standard Augmentation (\texttt{augm}) & True, False \\
RandAugment (\texttt{randaug}) & True, False \\
\bottomrule
\end{tabular}
\caption{Fixed and varying parameters for ViT-B/32 Base training on CIFAR-10.}
\label{tab:vit-config}
\end{table}

\subsection{CLIP ViT-B/32 on ImageNet-1K}

CLIP~\citep{pmlr-v139-radford21a} models are known for the widespread use of CLIP features~\citep{ramesh2022hierarchicaltextconditionalimagegeneration}. We use the pool of models introduced by \citet{wortsman2022modelsoupsaveragingweights}, who fine-tuned the CLIP ViT-B/32 architecture on ImageNet-1K multiple times using different randomly sampled training hyperparameters\footnote{Download link: \url{https://github.com/mlfoundations/model-soups/releases/}}. These hyperparameters include learning rate, number of training epochs, weight decay, label smoothing, and augmentation strategies, as stated in \citep{wortsman2022modelsoupsaveragingweights}. The resulting collection of 72 fine-tuned models provides a strong basis for evaluating the performance of model folding compared to pruning on CLIP ViT architectures. All checkpoints were evaluated jointly in our study, without parameter-specific ablations.

\subsection{LLaMA-60M on Colossal Clean Crawled Corpus (C4)}
We train 36 LLaMA-family models~\citep{touvron2023llamaopenefficientfoundation,touvron2023llama2openfoundation} with 60M and 130M parameters on the Colossal Clean Crawled Corpus (C4)~\citep{raffel2020exploring} on a NVIDIA DGX Station A100 featuring eight NVIDIA A100 GPUs (each equipped with 80GB memory). The training time for a LLaMA-60M model is about 45 minutes. \Tabref{llama_setting} summarizes the fixed hyperparameters used to train LLaMA-60M and LLaMA-130M. The learning rate is linearly warmed up, followed by a cosine annealing schedule that decays to 10\% of the initial value. We use the T5-base tokenizer~\citep{raffel2023t5} and AdamW optimizer, consistent with prior work~\citep{glentis2025scalableparametermemoryefficient,sltrain}.
\begin{table}[ht]
\centering
\begin{tabular}{ccccccc}
\rowcolor{lightblue}
\toprule
Params & Hidden & Intermediate & Heads & Layers & Steps & Data (Tokens) \\ \midrule
60M    & 512    & 1376         & 8     & 8      & 11K   & 1.3B  \\
130M   & 768    & 2048         & 12    & 12     & 22K   & 2.6B  \\
\bottomrule
\end{tabular}
\caption{Training hyperparameters of LLaMA-60M architecture.}
\label{llama_setting}
\end{table}

Note that in our work, pruning and folding are applied exclusively to the feed-forward network (FFN) layers of the trained LLaMA-60M and LLaMA-130M models.

\section{Extended Empirical Comparison of Folding and Pruning}
\label{appx:further_results}

We provide additional experiments to complement the main results.  
\Figref{fig:accuracy_FOLD_vs_MAG_L2} mirrors the setup of \Figref{fig:accuracy_FOLD_vs_MAG_L1} in the main paper, but replaces the L1 criterion for magnitude pruning with L2 (\magtwo).  
Similarly, \Figref{fig:ftLN_ViTs_FOLD_vs_MAG_L2}, \Figref{fig:ft_CNN_and_ViTs_2_FOLD_vs_MAG_L1}, \Figref{fig:ft_FOLD_vs_MAG_L2}, and \Figref{fig:ft_ViTs_FOLD_vs_MAG_L2} extend the corresponding figures in the main paper to other network architectures and to the L2 case. Across all comparisons, the qualitative picture remains the same: \fold consistently matches or outperforms magnitude pruning, independent of the chosen norm.

We further include ablations to study the robustness of these findings with respect to training hyperparameters.  
\Figref{fig:lr-effects-other}, \Figref{fig:sam-effects-other}, and \Figref{fig:raug-effects-other} report the effect of varying learning rate, SAM strength, and RandAugment, respectively.  
Finally, \Figref{fig:wd-effects-other} shows the influence of weight decay.  
Taken together, these studies confirm that the relative advantage of \fold is stable across different regularization strategies and training configurations.

\begin{figure}[t]
     \centering
     \vskip -0.4cm
     \subfloat[][\magone vs \fold, fine-tuning for\\1 (\textbf{left}) and 5 (\textbf{right}) epochs.]{\includegraphics[height=0.2\textwidth]{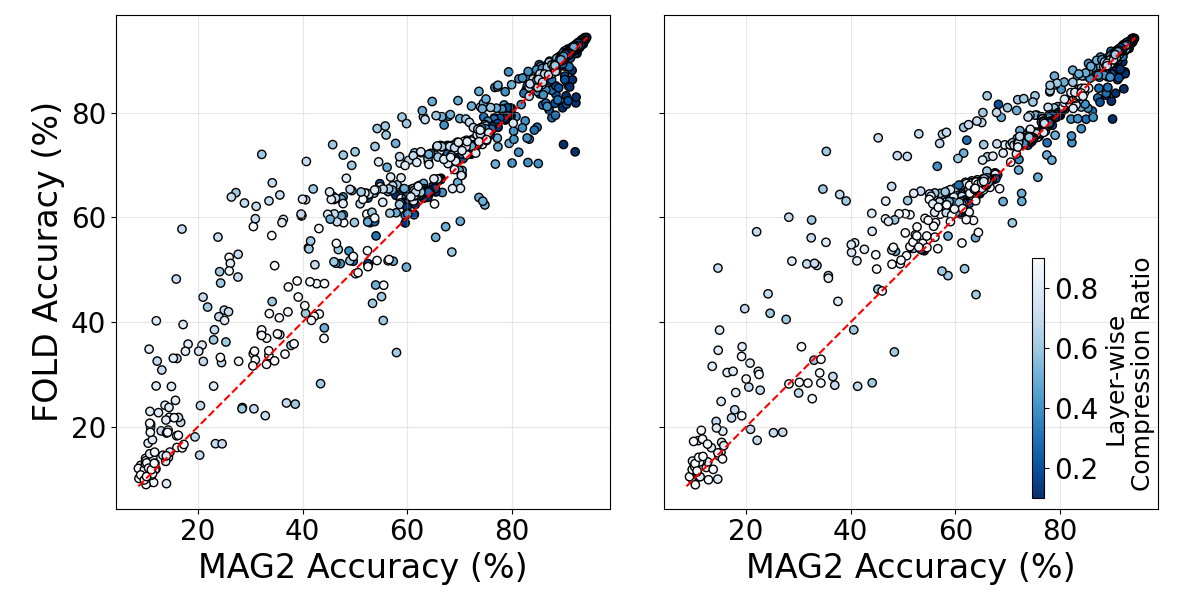}}
     \subfloat[][\fold vs \magone\\accuracy gap after ft.]{\includegraphics[height=0.2\textwidth]{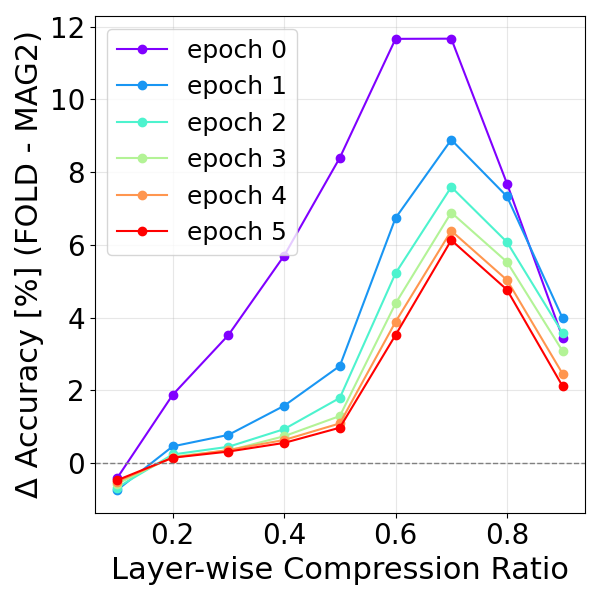}}
     \subfloat[][Before and after fine-tuning for 5 epochs: \fold (\textbf{left}) and \magone (\textbf{right}).]{\includegraphics[height=0.2\textwidth]{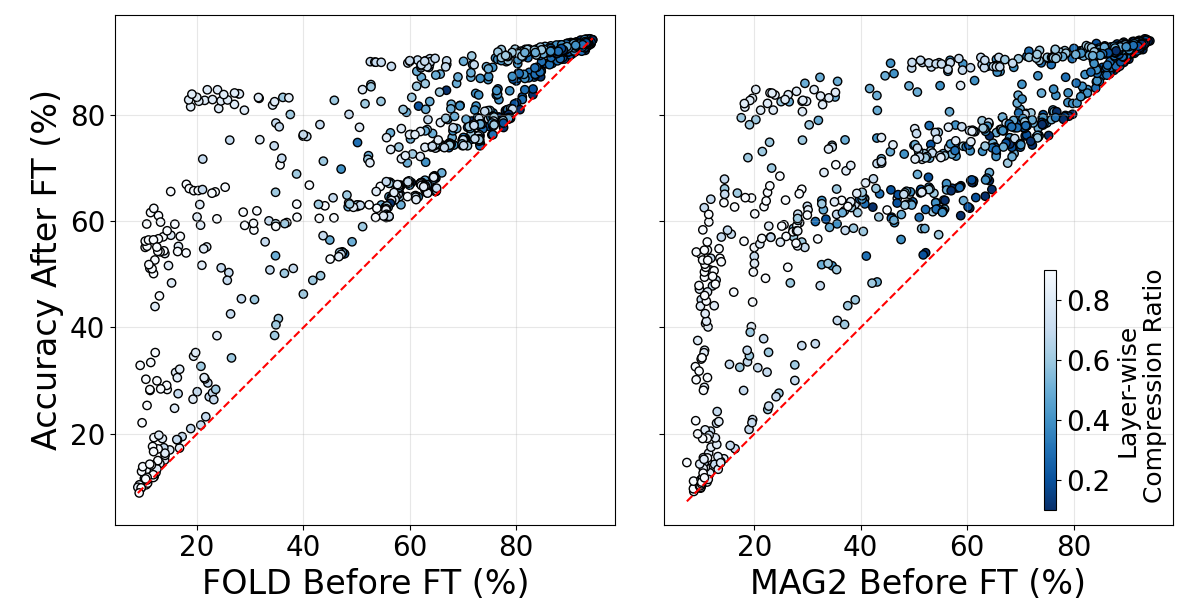}}
     
     \subfloat[][\magone vs \fold, fine-tuning for\\1 (\textbf{left}) and 5 (\textbf{right}) epochs.]
     {\includegraphics[height=0.2\textwidth]{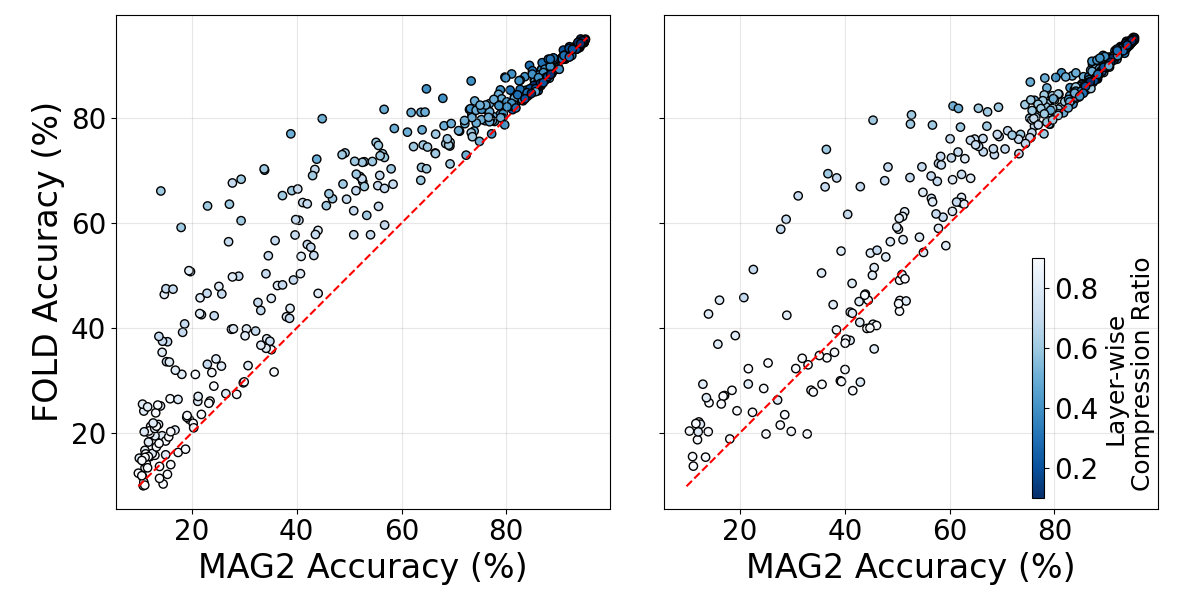}}
     \subfloat[][\fold vs \magone\\accuracy gap after ft.]{\includegraphics[height=0.2\textwidth]{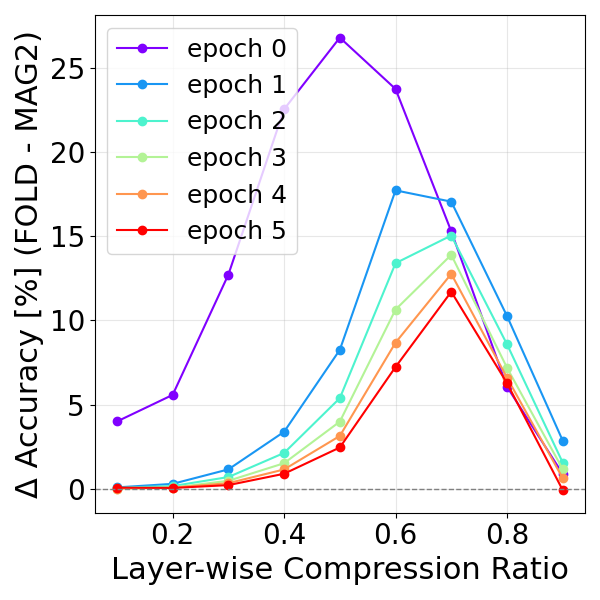}}
     \subfloat[][Before and after fine-tuning for 5 epochs: \fold (\textbf{left}) and \magone (\textbf{right}).]
     {\includegraphics[height=0.2\textwidth]{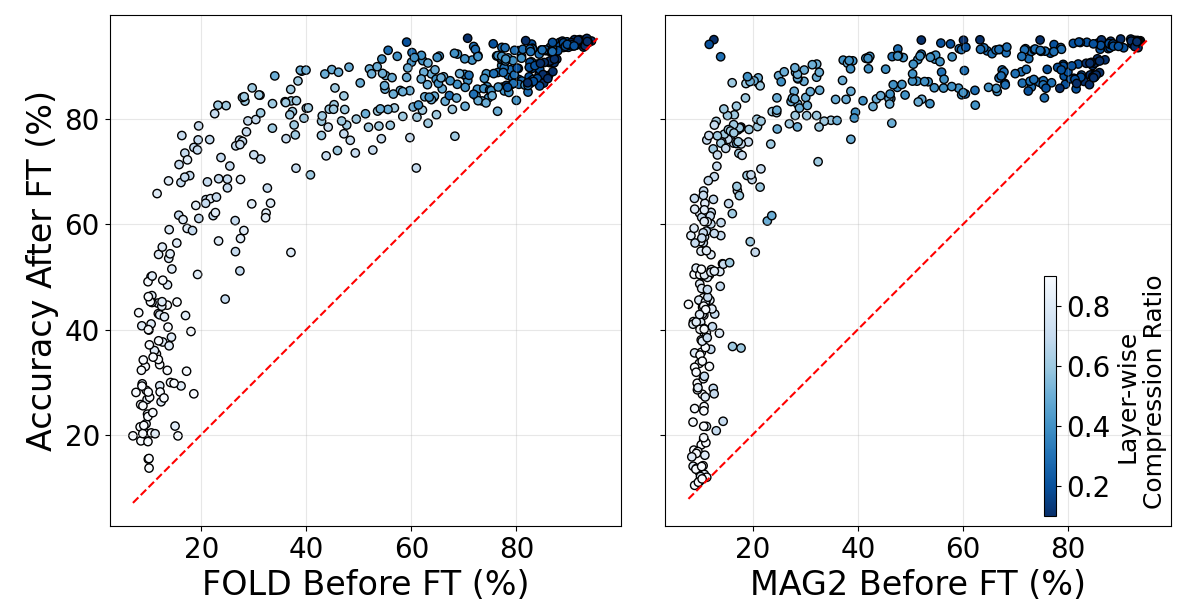}}
     
     \caption{\textbf{Folded models retain their accuracy advantage after fine-tuning.} Results for ResNet18 trained by Adam (\textbf{top row}) and PreActResNet18 trained by SGD on CIFAR-10 (\textbf{bottom row}): \textbf{(a,d)} compares post-compression accuracy of magnitude pruning with L2 criterion (\magtwo) versus folding (\fold) after 1 and 5 epochs of fine-tuning. \textbf{(b,e)} show the accuracy gap between folding and pruning as a function of fine-tuning epochs, demonstrating that folding maintains a consistent lead, \ie the \fold accuracy delta is positive. \textbf{(c,f)} illustrate accuracy trajectories before and after 5 epochs of fine-tuning for both methods, highlighting that folded models recover accuracy faster and reach higher final performance than pruned models. The figure extends \Figref{fig:ft_CNN_and_ViTs_1_FOLD_vs_MAG_L1} in the main paper and \Figref{fig:ft_CNN_and_ViTs_2_FOLD_vs_MAG_L1} in the appendix to \magtwo.
     }
     \label{fig:ft_FOLD_vs_MAG_L2}
\end{figure}

\begin{figure}[t]
     \centering
     \vskip -0.4cm
     \subfloat[][\magtwo vs \fold, fine-tuning for\\1 (\textbf{left}) and 5 (\textbf{right}) epochs.]{\includegraphics[height=0.2\textwidth]{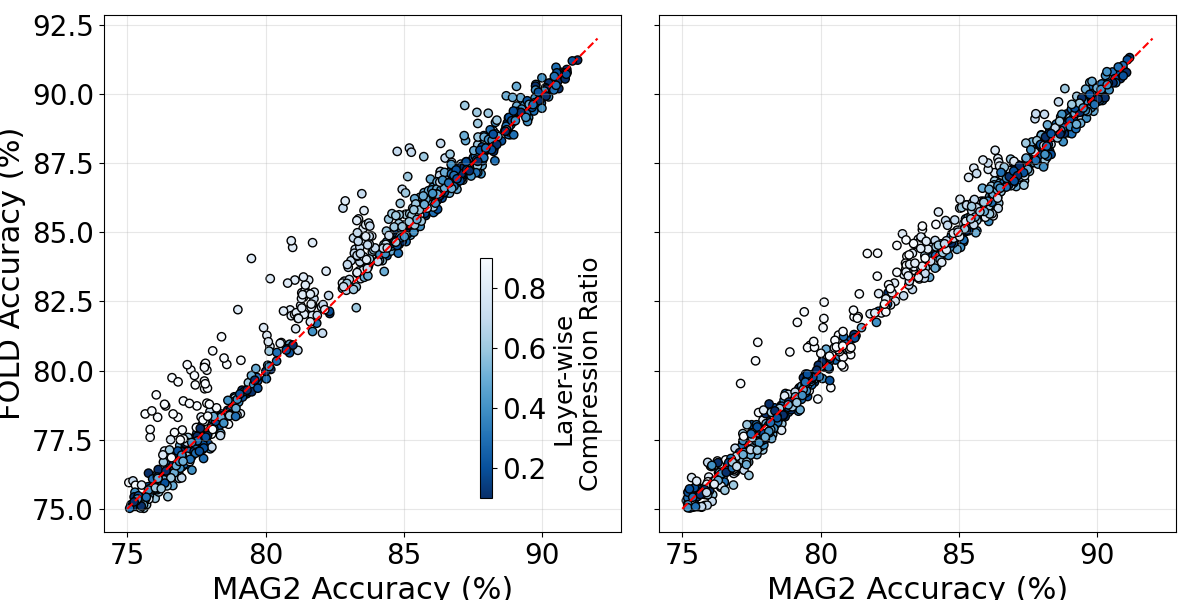}}
     \subfloat[][\magtwo vs \fold\\accuracy gap after ft.]{\includegraphics[height=0.2\textwidth]{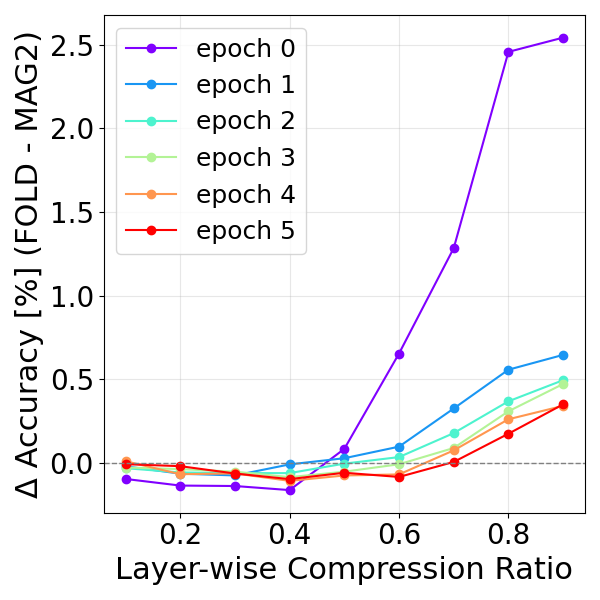}}
     \subfloat[][Before and after fine-tuning for 5 epochs: \fold (\textbf{left}) and \magtwo (\textbf{right}).]{\includegraphics[height=0.2\textwidth]{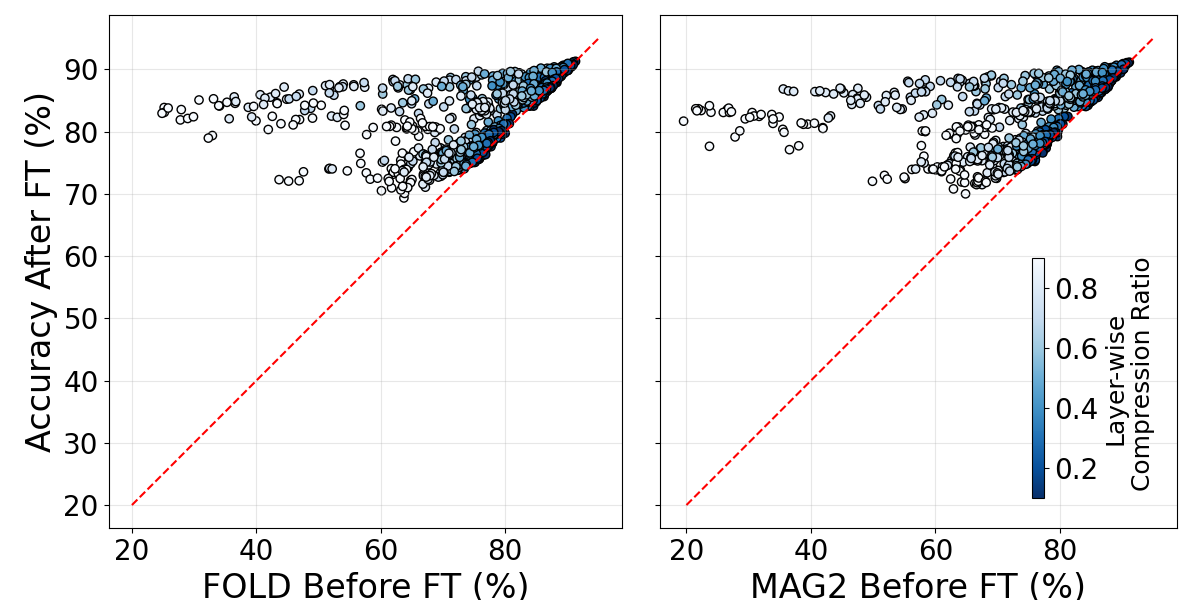}}
     
     \subfloat[][\magtwo vs \fold, fine-tuning for\\1 (\textbf{left}) and 5 (\textbf{right}) epochs.]
     {\includegraphics[height=0.2\textwidth]{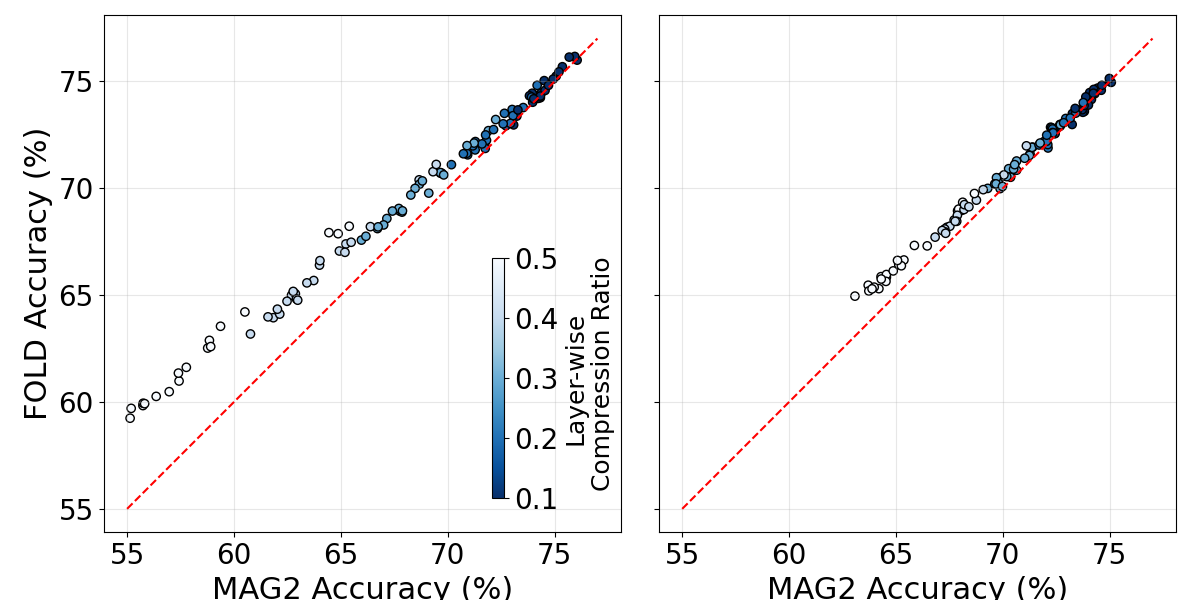}}
     \subfloat[][\magtwo vs \fold\\accuracy gap after ft.]{\includegraphics[height=0.2\textwidth]{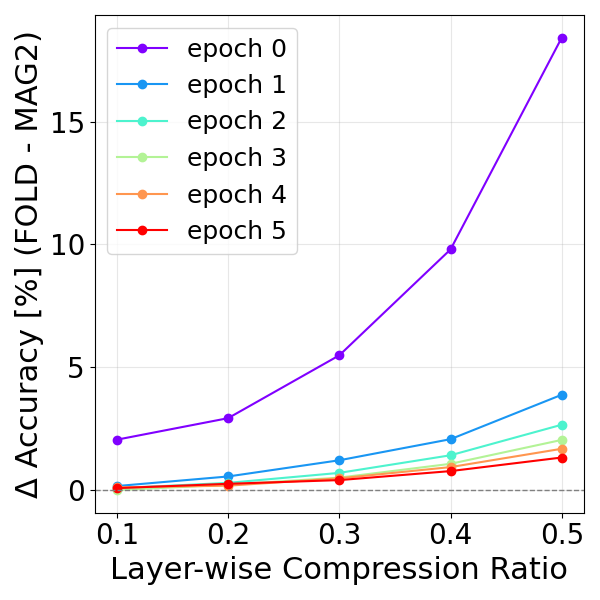}}
     \subfloat[][Before and after fine-tuning for 5 epochs: \fold (\textbf{left}) and \magtwo (\textbf{right}).]
     {\includegraphics[height=0.2\textwidth]{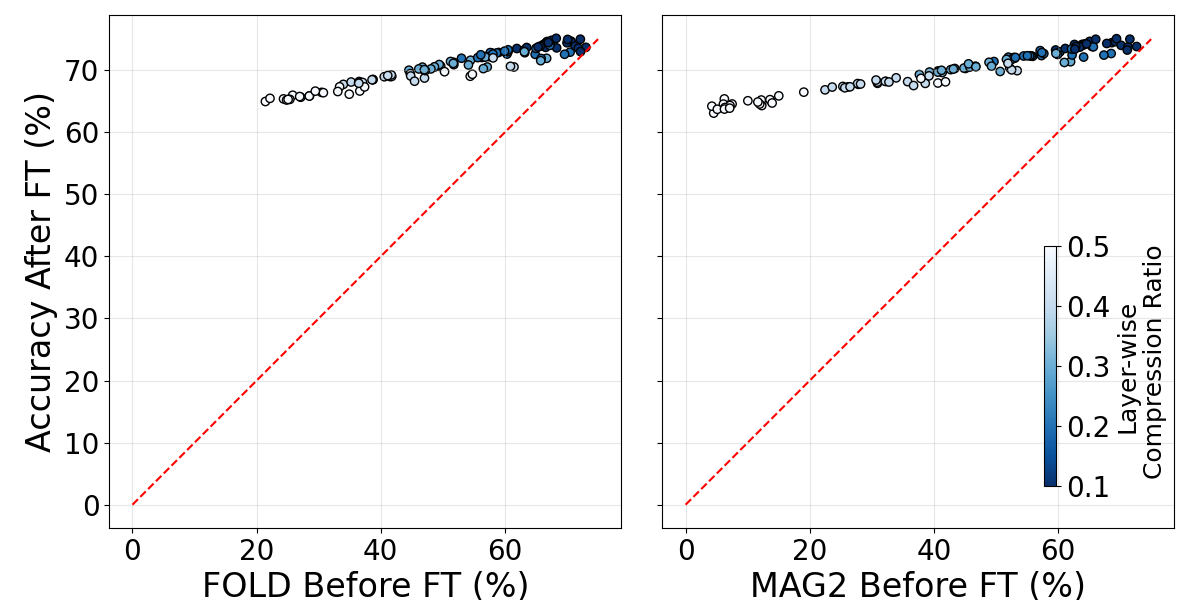}}
     
     \caption{\textbf{\fold outperforms \magtwo after full fine-tuning for 1--5 epochs on ViT-B/32 and CLIP ViT-B/32.} Results for ViT-B/32 on CIFAR-10 (\textbf{top}) and CLIP ViT-B/32 on ImageNet-1K (\textbf{bottom}). \textbf{(a,d)} accuracy of \magtwo vs. \fold after 1 and 5 epochs of fine-tuning. \textbf{(b,e)} accuracy gap $\Delta$ over epochs, remaining positive. \textbf{(c,f)} accuracy trajectories from post-compression through 5 epochs, showing faster recovery and higher final accuracy for \fold. The figure extends \Figref{fig:ft_CNN_and_ViTs_1_FOLD_vs_MAG_L1} in the main paper and \Figref{fig:ft_CNN_and_ViTs_2_FOLD_vs_MAG_L1} in the appendix to \magtwo.}
     \label{fig:ft_ViTs_FOLD_vs_MAG_L2}
\end{figure}

\begin{figure}[t]
     \centering
     \vskip -0.4cm
     \subfloat[][ResNet18, Adam, \textbf{no} L1 reg., \textbf{no} weight decay]{
        \includegraphics[height=0.24\textwidth]{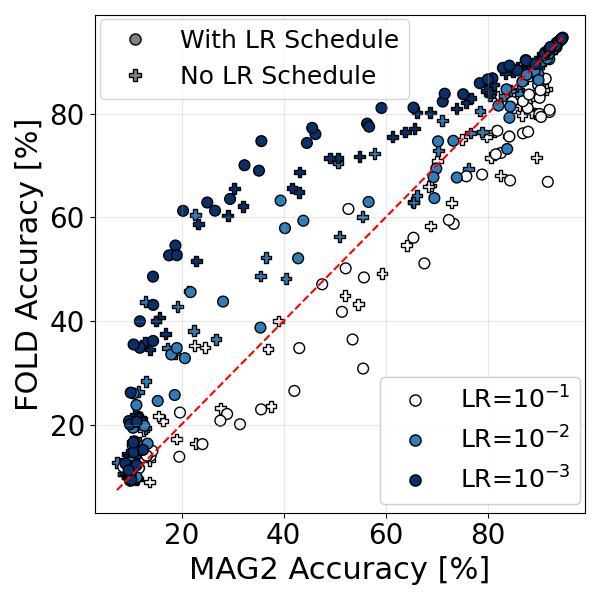}
        \includegraphics[height=0.24\textwidth]{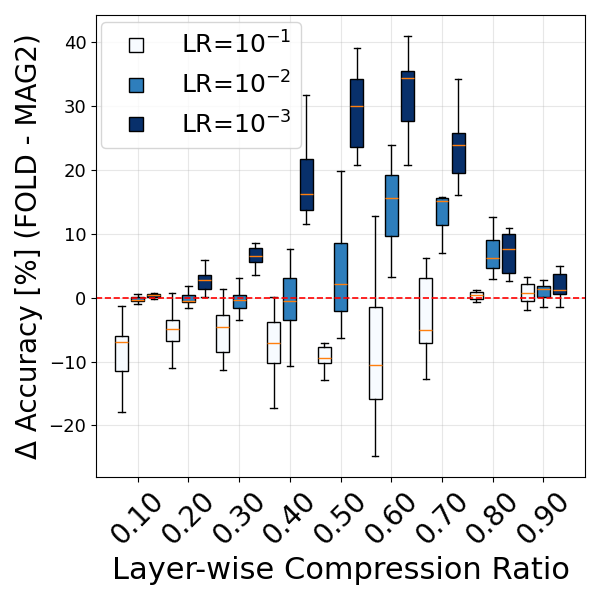}
    }
    \subfloat[][ResNet18, SGD, \textbf{no} L1 reg., \textbf{no} weight decay]{
        \includegraphics[height=0.24\textwidth]{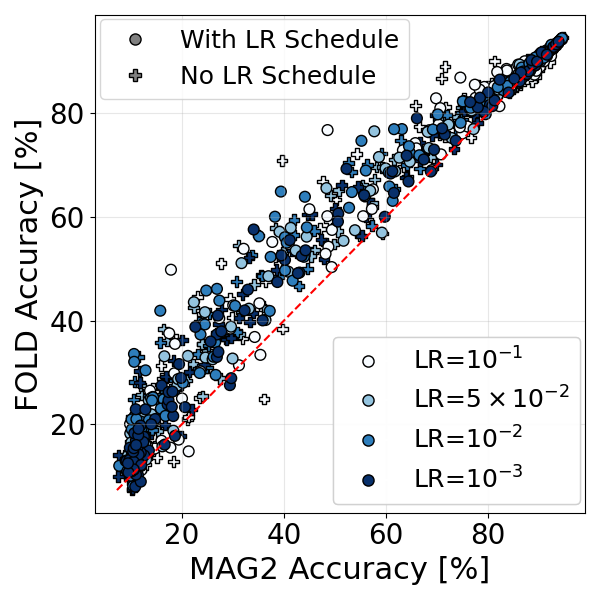}
        \includegraphics[height=0.24\textwidth]{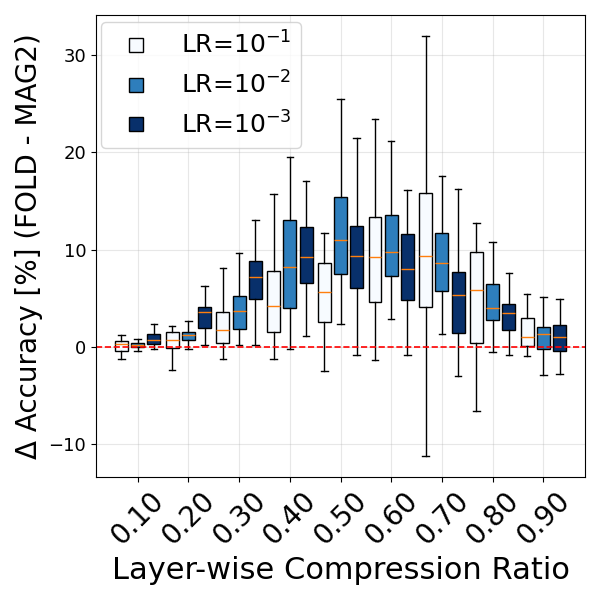}
    }

    \subfloat[][PreActResNet18]{
        \includegraphics[height=0.24\textwidth]{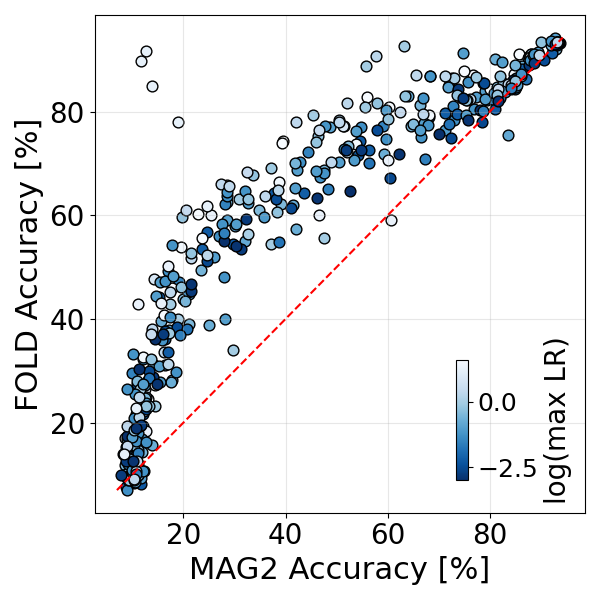}
        \includegraphics[height=0.24\textwidth]{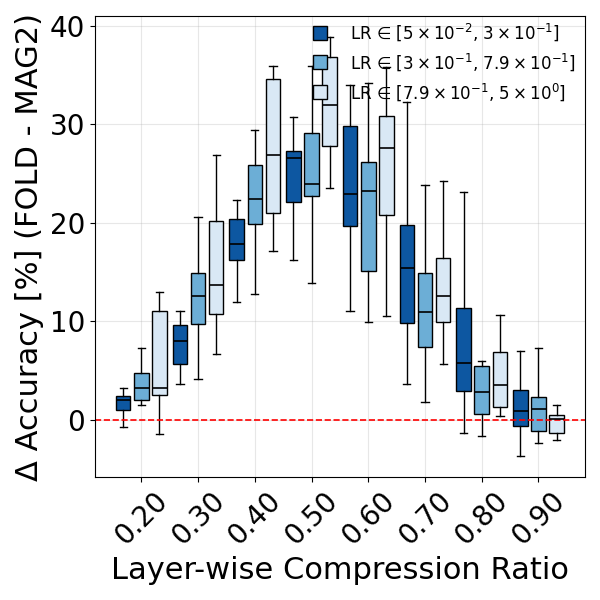}
    }
    \subfloat[][ViT-B/32]{
        \includegraphics[height=0.24\textwidth]{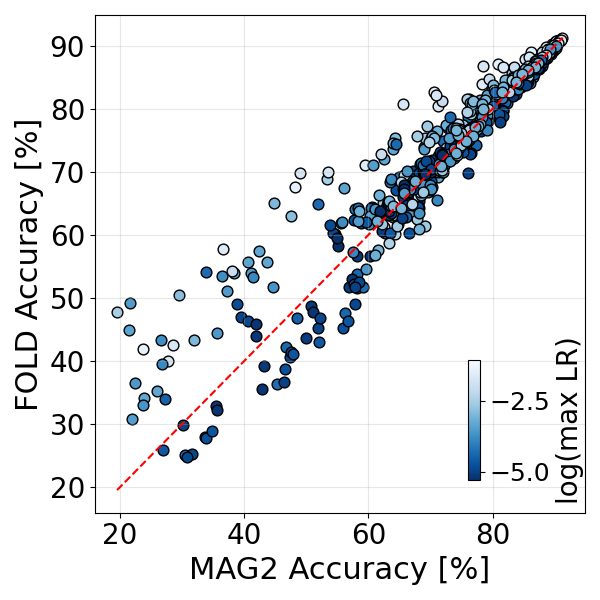}
        \includegraphics[height=0.24\textwidth]{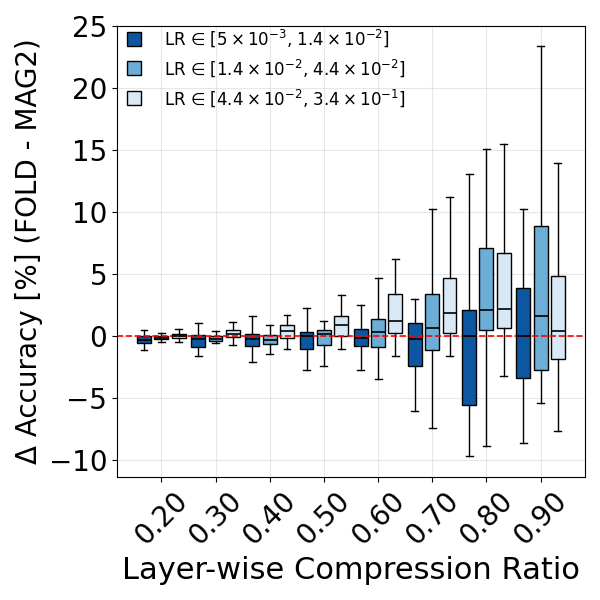}
    }
     
     \caption{\textbf{Learning rate modulates folding’s edge.} Post-compression accuracy of \magtwo and \fold across learning rates: ResNet18 with Adam \textbf{(a)} and SGD \textbf{(b)}, PreActResNet18 \textbf{(c)}, and ViT-B/32 \textbf{(d)}. \fold typically leads at moderate–low rates; the gap shrinks or reverses at very high rates, and closes again at extremely small rates. The same setup as in \Figref{fig:lr-effects} in the main paper, but for \magtwo.}
    \label{fig:lr-effects-other}
\end{figure}

\begin{figure}[t]
     \centering
     \vskip -0.4cm
     \subfloat[][ResNet18, SGD, \magtwo vs \fold, \textbf{without} L1 reg.]{
        \includegraphics[height=0.24\textwidth]{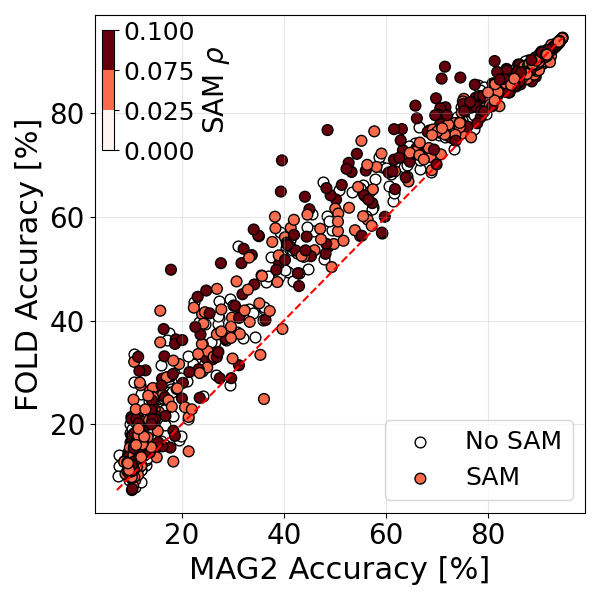}
        \includegraphics[height=0.24\textwidth]{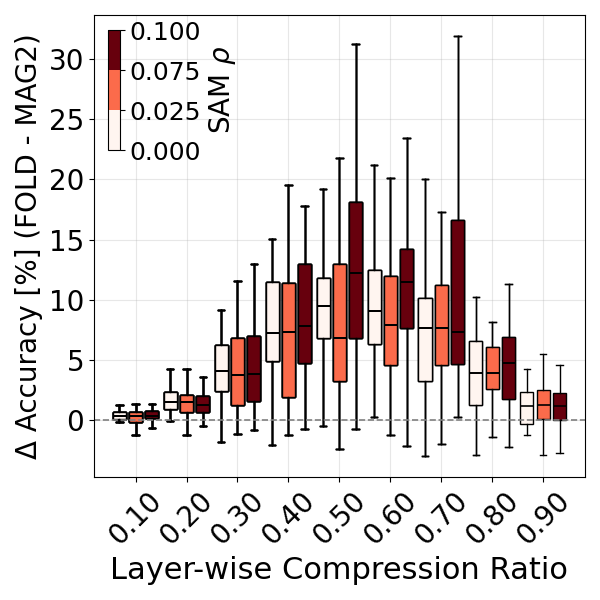}
     }
     \subfloat[][ResNet18, SGD, \magtwo vs \fold, \textbf{with} L1 reg.]{
        \includegraphics[height=0.24\textwidth]{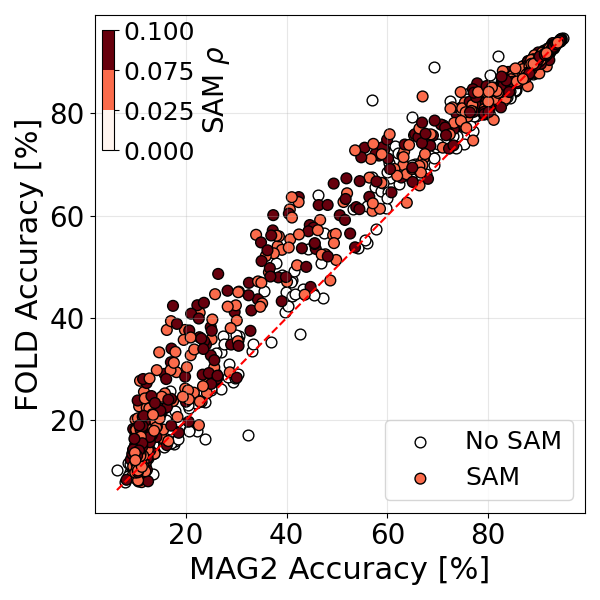}
        \includegraphics[height=0.24\textwidth]{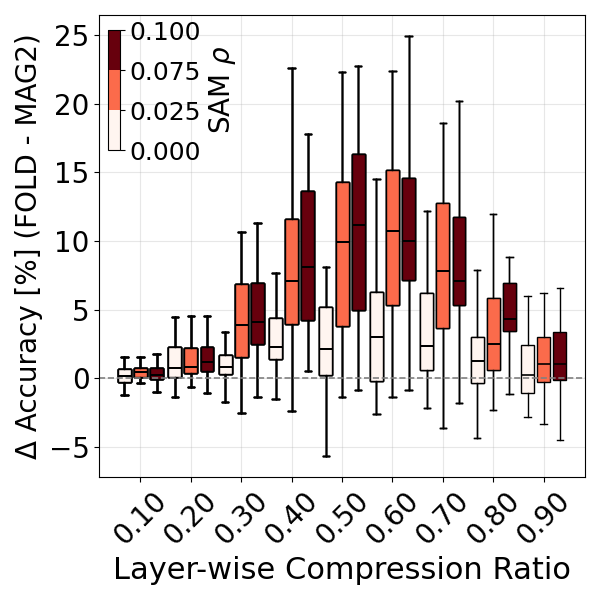}
     }

     \subfloat[][PreActResNet, \magtwo vs \fold, \textbf{no} L1 regularization]{
        \includegraphics[height=0.24\textwidth]{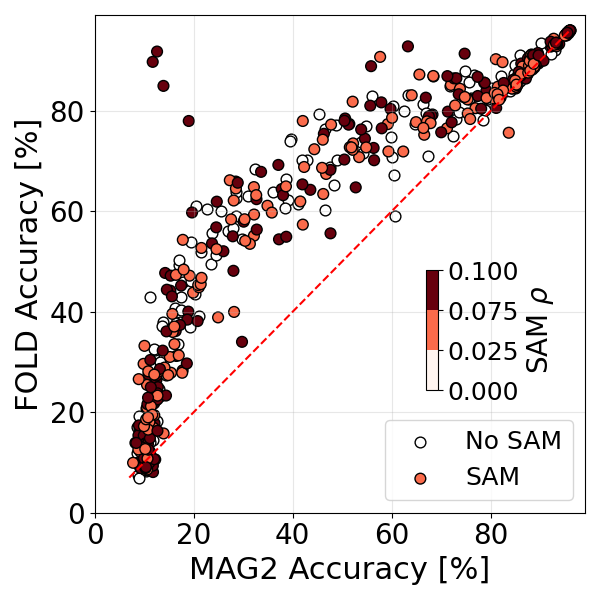}
        \includegraphics[height=0.24\textwidth]{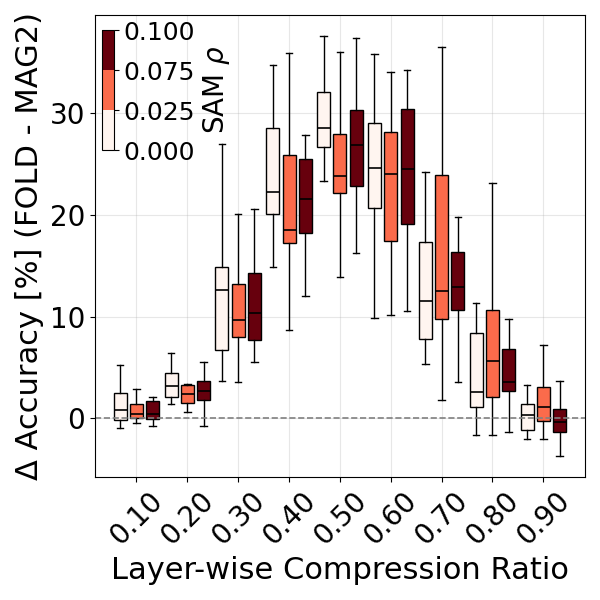}
     }
     \subfloat[][ViT-B/32, \magtwo vs \fold, \textbf{no} L1 regularization]{
        \includegraphics[height=0.24\textwidth]{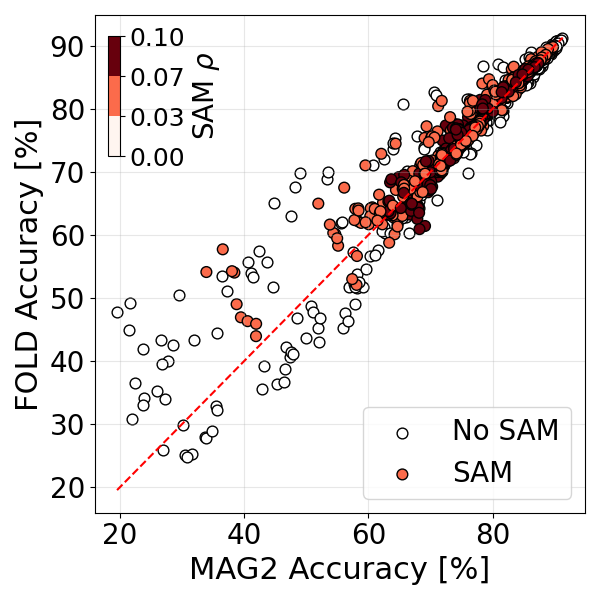}
        \includegraphics[height=0.24\textwidth]{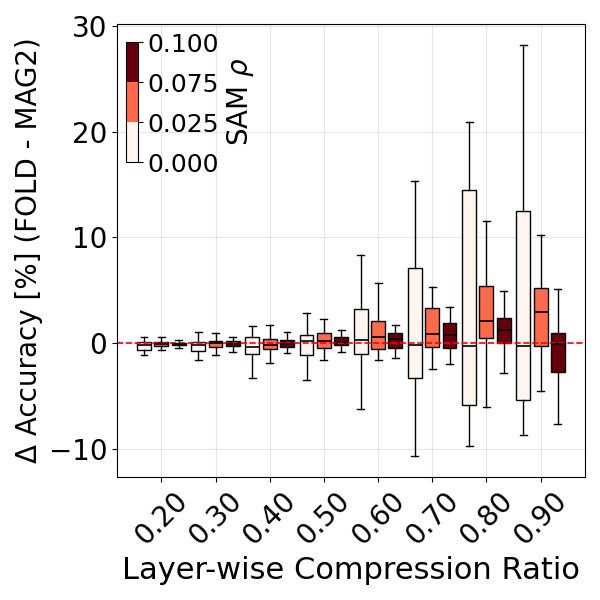}
      }

     \caption{\textbf{SAM can boost model compression.} Post-compression accuracy under training with\,/\,without SAM. \textbf{(a)} ResNet18 (Adam), no L1. \textbf{(b)} ResNet18 (Adam), L1$=10^{-5}$. \textbf{(c)} PreActResNet18 (SGD), no L1. \textbf{(d)} ViT\mbox{-}B/32, no L1. The figure extends the results in \Figref{fig:sam-effects} to \magtwo.}
    \label{fig:sam-effects-other}
\end{figure}

\begin{figure}[t]
    \centering
    \vskip -0.4cm
    \subfloat[][Adam, \magtwo vs \fold]{
        \includegraphics[height=0.24\textwidth]{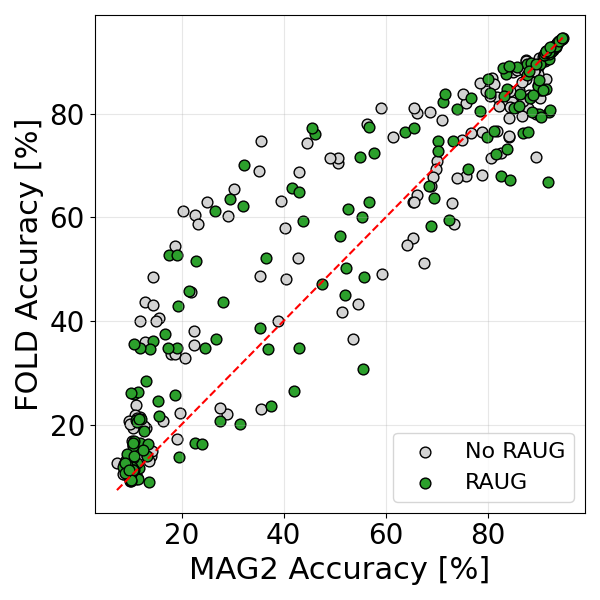}
        \includegraphics[height=0.24\textwidth]{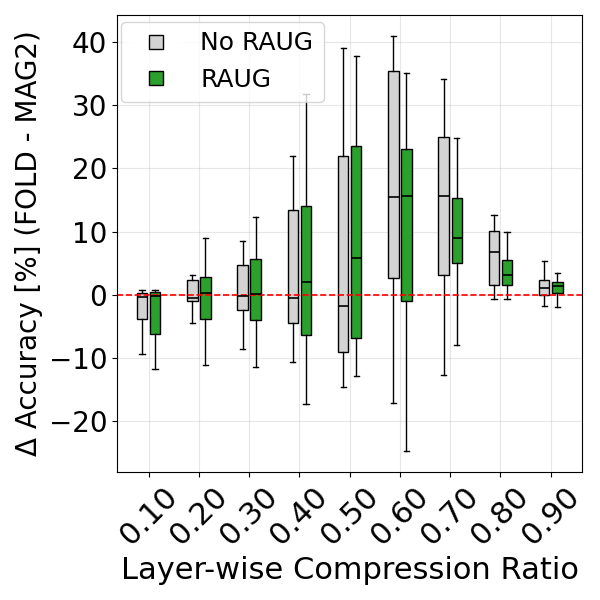}
    }
    \subfloat[][SGD, \magtwo vs \fold]{
        \includegraphics[height=0.24\textwidth]{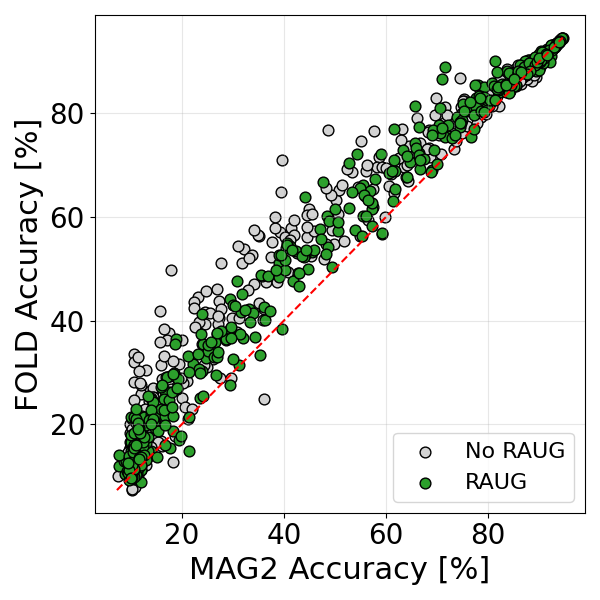}
        \includegraphics[height=0.24\textwidth]{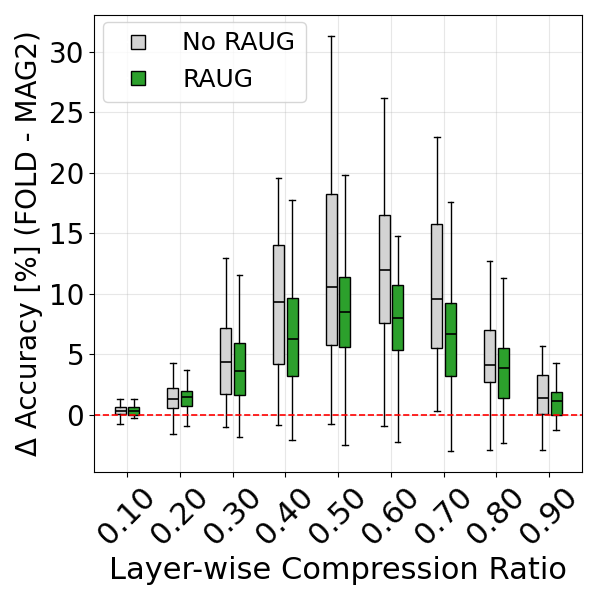}
    }

    \subfloat[][PreActResNet, \magtwo vs \fold]{
        \includegraphics[height=0.24\textwidth]{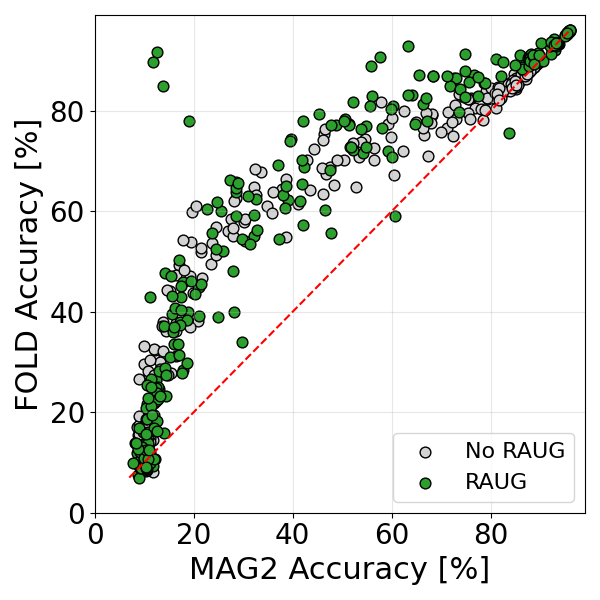}
        \includegraphics[height=0.24\textwidth]{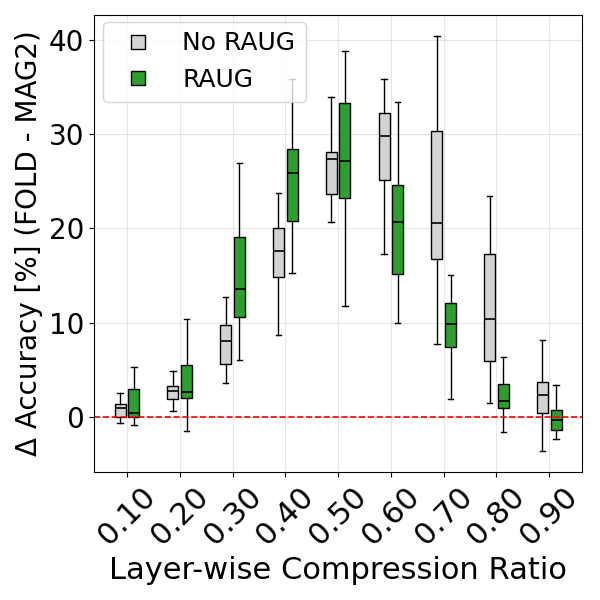}
    }
    \subfloat[][ViT-B/32, \magtwo vs \fold]{
        \includegraphics[height=0.24\textwidth]{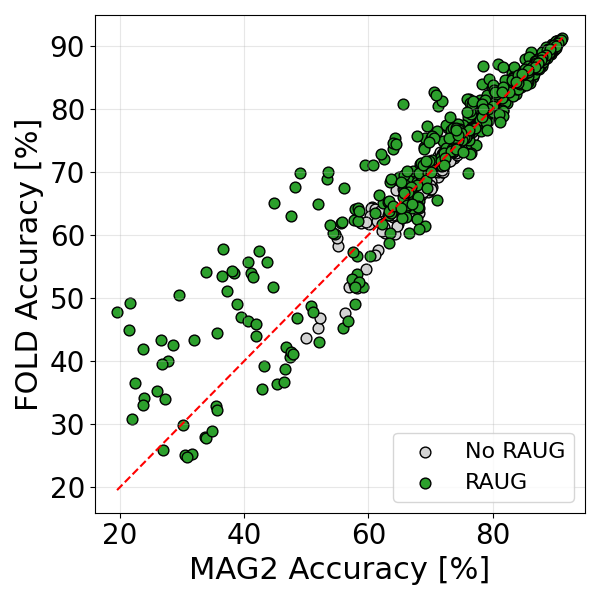}
        \includegraphics[height=0.24\textwidth]{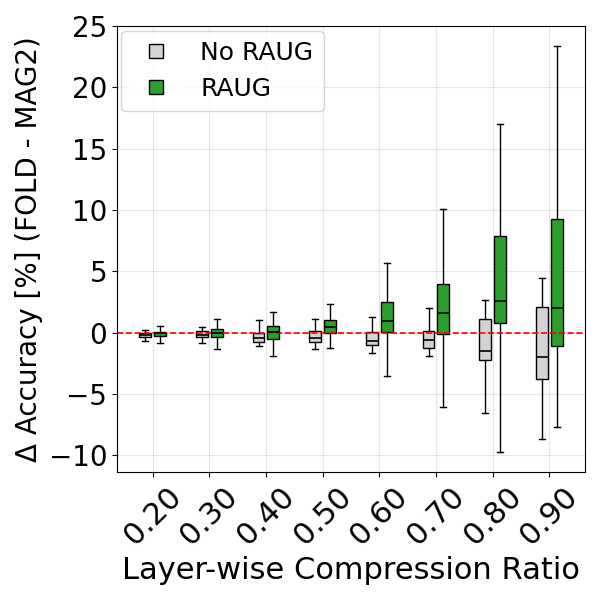}
    }
    \caption{\textbf{Random augmentations narrow the folding–pruning gap.} Post-compression accuracy on ResNet18 (CIFAR-10) trained without vs.\ with random augmentations: \textbf{(a)} Adam, \textbf{(b)} SGD, \textbf{(c)} PreActResNet, \textbf{(d)} ViT-B/32. The figure extends \Figref{fig:raug-effects} to \magtwo.
     }
    \label{fig:raug-effects-other}
\end{figure}

\begin{figure}[t]
    \centering
    \vskip -0.4cm
    \subfloat[][Adam, \magtwo vs \fold]{
        \includegraphics[height=0.24\textwidth]{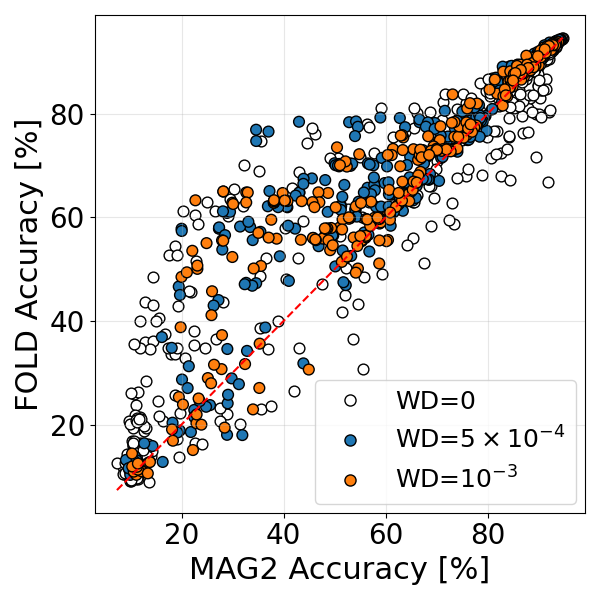}
        \includegraphics[height=0.24\textwidth]{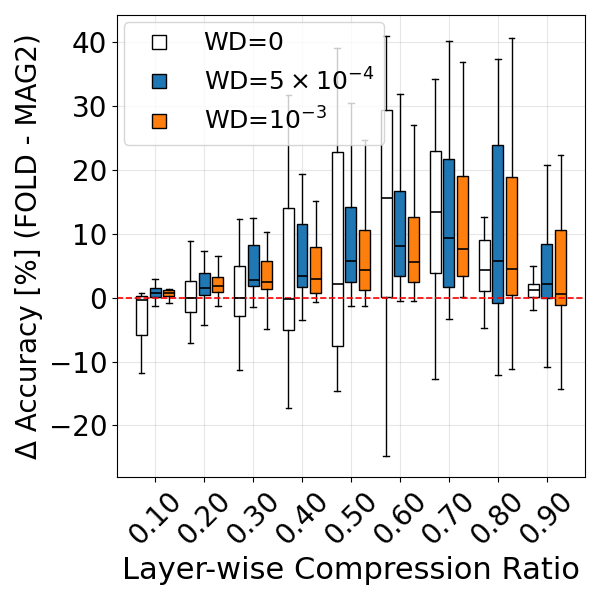}
    }
    \subfloat[][SGD, \magtwo vs \fold]{
        \includegraphics[height=0.24\textwidth]{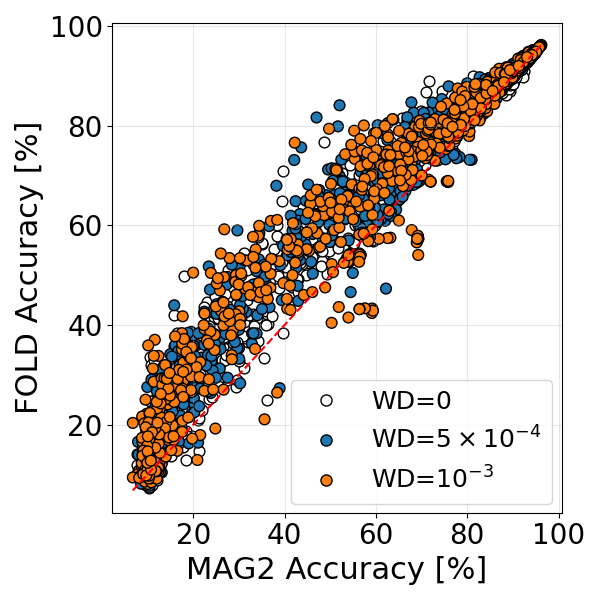}
        \includegraphics[height=0.24\textwidth]{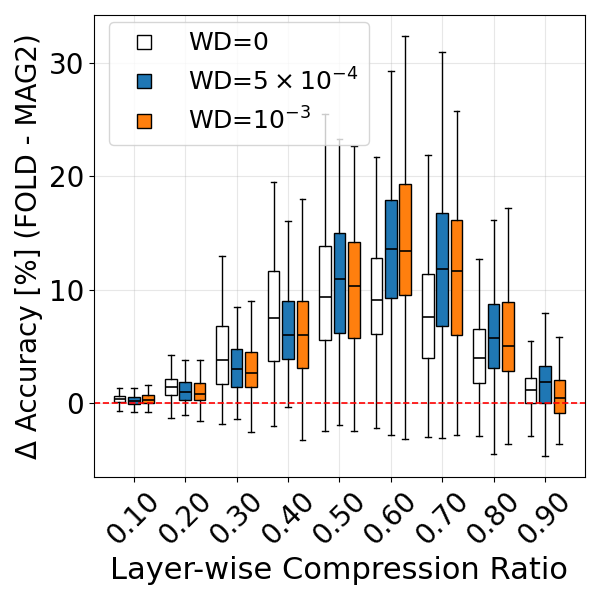}
    }
    \caption{\textbf{ResNet18: Weight Decay.} 
    Test accuracy of ResNet18 checkpoints trained with varying weight decay values. Weight decay does not diminish the advantage of \fold compared to \magtwo, especially for SGD-trained models.}
    \label{fig:wd-effects-other}
\end{figure}

\section{Additional Analyses: Sharpness, Runtime, and LLMs}
\label{appx:rebuttal}

Below we report additional evaluations. We extend our study by training and compressing 60M- and 130M-parameter LLaMA models on C4, and provide analyses of sharpness, and measure runtime overhead.

\begin{table*}[h]
\centering
\small
\setlength{\tabcolsep}{4pt}
\renewcommand{\arraystretch}{1.2}
\resizebox{1.00\textwidth}{!}{%
\begin{tabular}{cccc|cc|cc}
\hline
\rowcolor{lightblue}
\textbf{weight\_decay} & \textbf{warmup\_steps} & \textbf{max\_lr} &
\textbf{PPL$\downarrow$ 0\% sparsity} &
\textbf{PPL$\downarrow$ \magtwo (20\%)} & \textbf{PPL$\downarrow$ \fold (20\%)} &
\textbf{PPL$\downarrow$ \magtwo (50\%)} & \textbf{PPL$\downarrow$ \fold (50\%)} \\
\hline
0.01& 1100 & 0.001 & 23.90  & 39.88  & \textbf{39.48}  & \textbf{236.16} & 308.77 \\
0.01& 2200 & 0.001 & 23.99  & \textbf{38.75} & 39.61   & \textbf{259.79} & 469.25 \\
0.01& 3300 & 0.001 & 24.08  & \textbf{38.54} & 39.10   & \textbf{289.67} & 451.27 \\
0.0 & 1100 & 0.001 & 24.01  & 42.39  & \textbf{42.27}  & \textbf{270.31} & 477.70 \\
0.0 & 2200 & 0.001 & 24.12  & \textbf{40.01} & 41.53   & \textbf{239.25} & 489.48 \\
0.0 & 3300 & 0.001 & 24.19  & \textbf{38.72} & 40.31   & \textbf{277.10} & 531.09 \\
0.01& 1100 & 0.005 & 42.11  & 72.31  & \textbf{62.63}  & 536.53 & \textbf{298.65} \\
0.01& 2200 & 0.005 & 22.82  & 52.18  & \textbf{40.46}  & 824.69 & \textbf{333.59} \\
0.01& 3300 & 0.005 & \underline{22.66}  & 44.35  & \underline{\textbf{36.22}}  & 589.33 & \textbf{222.21} \\
0.0 & 1100 & 0.005 & 44.92  & 73.38  & \textbf{63.75}  & 414.17 & \textbf{261.23} \\
0.0 & 2200 & 0.005 & 23.32  & 57.62  & \textbf{43.04}  & 1616.74 & \textbf{342.64} \\
0.0 & 3300 & 0.005 & 23.00  & 46.87  & \textbf{39.11}  & 904.07 & \textbf{305.85} \\
0.01& 1100 & 0.01  & 300.95 & 302.26 & \textbf{301.87} & 401.28 & \textbf{361.10} \\
0.01& 2200 & 0.01  & 66.48  & 88.09  & \textbf{84.30}  & 398.16 & \textbf{252.99} \\
0.01& 3300 & 0.01  & 54.34  & 97.35  & \textbf{76.15}  & 440.71 & \textbf{229.42} \\
0.0 & 1100 & 0.01  & 282.11 & 282.48 & \textbf{282.38} & 345.74 & \textbf{329.58} \\
0.0 & 2200 & 0.01  & 140.20 & 169.78 & \textbf{149.58} & 352.14 & \textbf{234.69} \\
0.0 & 3300 & 0.01  & 86.18  & 118.05 & \textbf{100.14} & 339.37 & \underline{\textbf{179.43}} \\
\hline
\end{tabular}
}
\caption{\textbf{Evaluation of \fold and \magtwo on LLaMA-130M} (in addition to LLaMA-60M evaluations in \Tabref{tab:llama60m}). We train and evaluate 18 LLaMA-family models with 130M parameters on C4 while varying max\_lr, warmup steps, and weight decay. Columns show perplexity of the pretrained model (0\% sparsity) and perplexity after structured magnitude pruning and folding with 20\% and 50\% sparsity in FFN blocks. For higher learning rates, especially for the settings with the best achieved performance in each sparsity category (underlined), \fold consistently outperforms \magtwo (bold).}
\label{tab:llama130m}
\end{table*}

\begin{figure}[h]
     \centering
     \vskip -0.4cm
     \subfloat[][ResNet18, Adam, \fold vs \magone, \textbf{no} L1 regularization]{
        \includegraphics[height=0.24\textwidth]{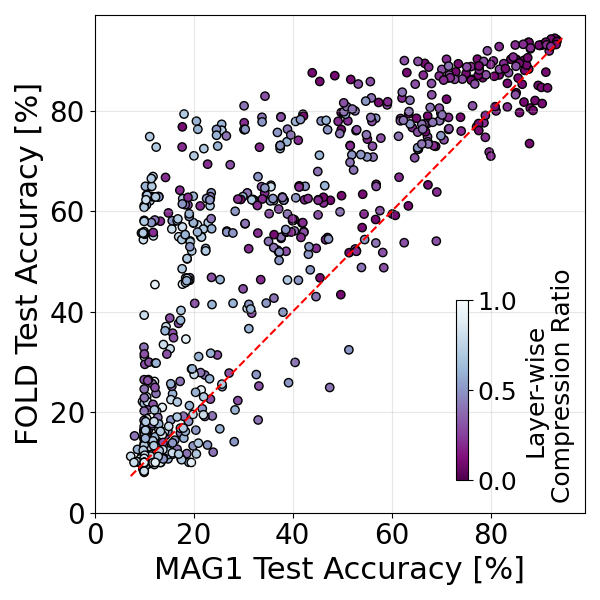}
        \includegraphics[height=0.24\textwidth]{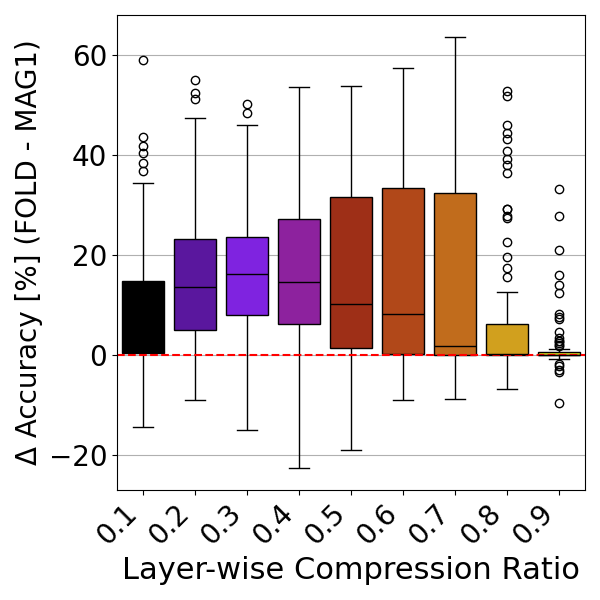}
    }
    \subfloat[][PreActResNet18, \fold vs \magone]{
        \includegraphics[height=0.24\textwidth]{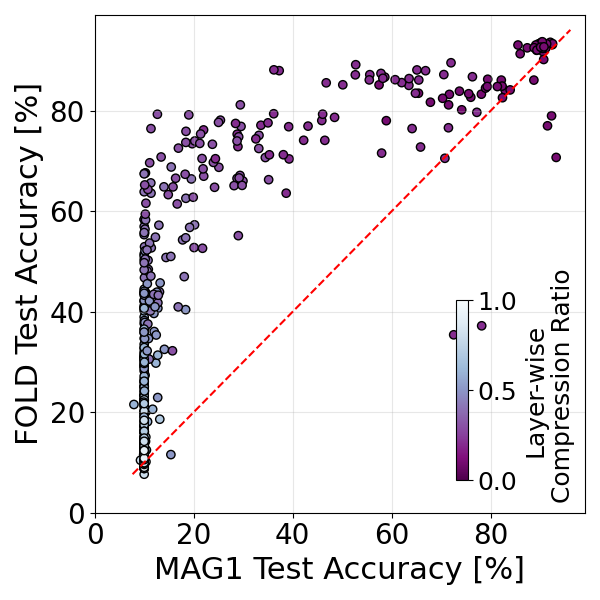}
        \includegraphics[height=0.24\textwidth]{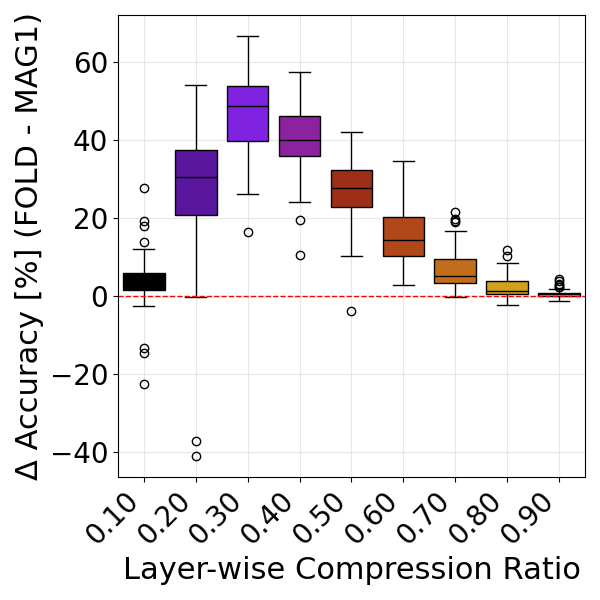}
    }

    \subfloat[][ResNet18, Adam, \fold vs \magtwo, \textbf{no} L1 regularization]{
        \includegraphics[height=0.24\textwidth]{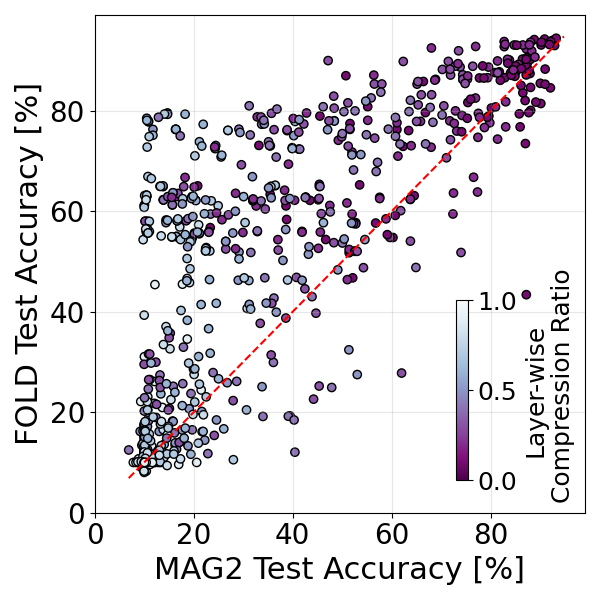}
        \includegraphics[height=0.24\textwidth]{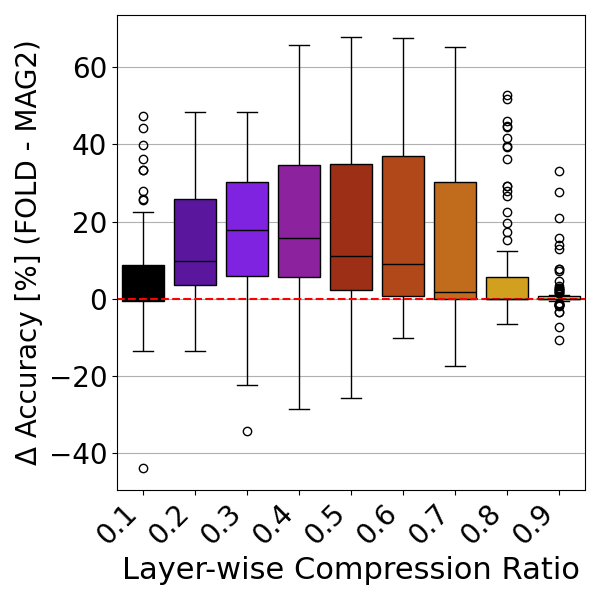}
    }
    \subfloat[][PreActResNet18, \fold vs \magtwo]{
        \includegraphics[height=0.24\textwidth]{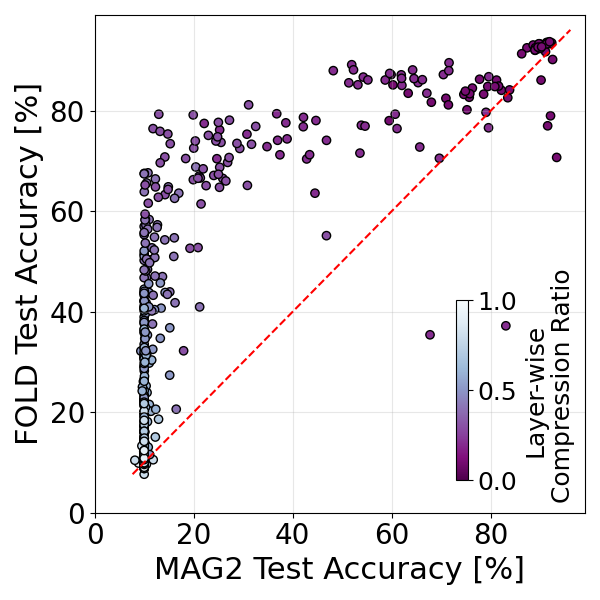}
        \includegraphics[height=0.24\textwidth]{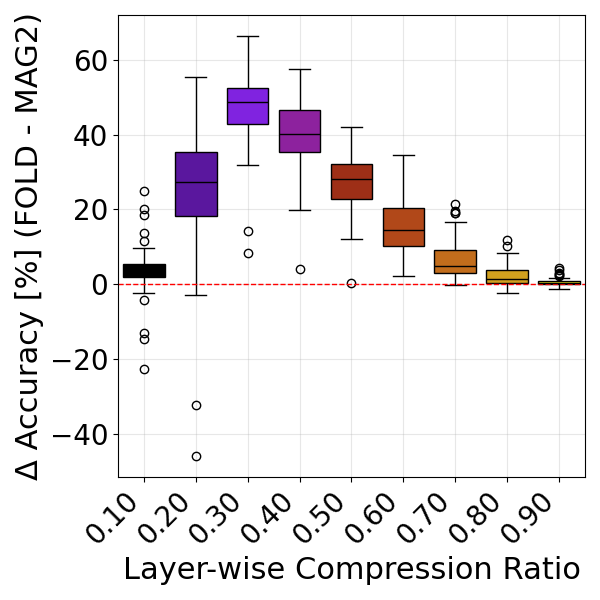}
    }
    \caption{\textbf{Folding vs. magnitude pruning before REPAIR.} The same setup as in \Figref{fig:accuracy_FOLD_vs_MAG_L1} and \Figref{fig:accuracy_FOLD_vs_MAG_L2} for CNNs (ResNet18 and PreActResNet18 on CIFAR-10), but the performance is compared for both pruning and folding before REPAIR.
    \textbf{Top row:} \magone, 
    \textbf{bottom row:} \magtwo. In both cases, folding shows stronger performance already before data-based REPAIR is applied.}
     \label{fig:accuracy_FOLD_vs_MAG_L2__before_repair}
\end{figure}

\begin{figure}[h]
     \centering
     \vskip -0.4cm
     \subfloat[][ResNet18, Adam]{
        \includegraphics[height=0.19\textwidth]{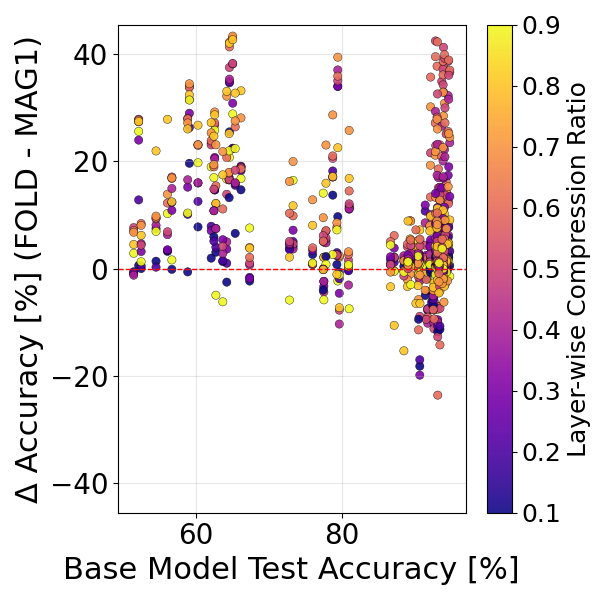}
    }
    \subfloat[][ResNet18, SGD]{
        \includegraphics[height=0.19\textwidth]{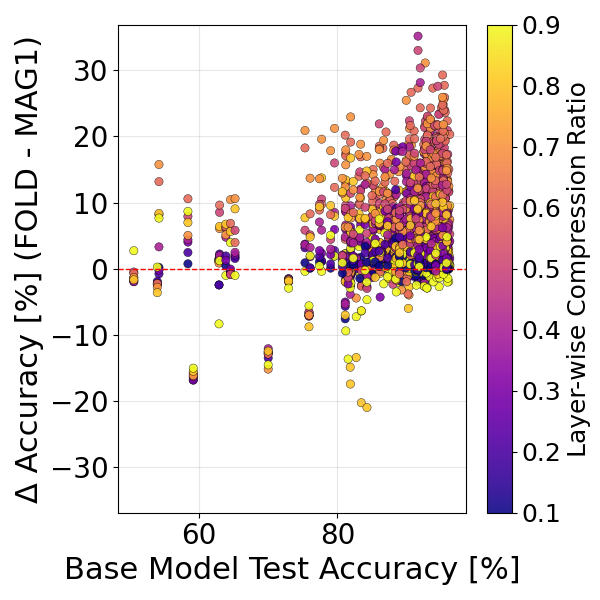}
    }
    \subfloat[][PreActResNet18]{
        \includegraphics[height=0.19\textwidth]{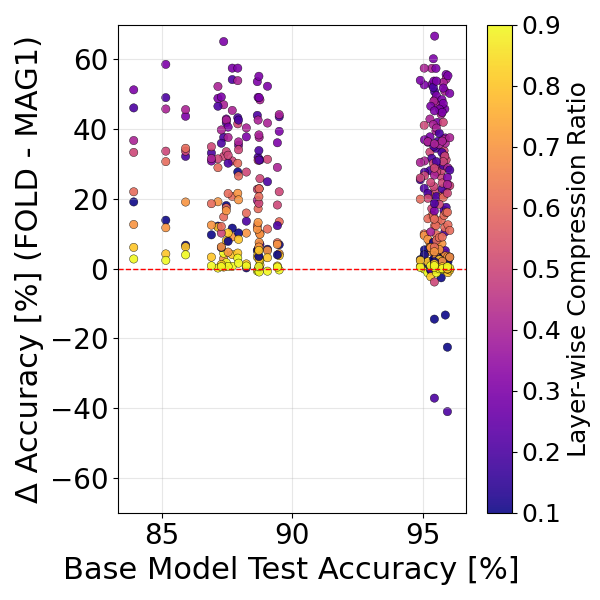}
    }
    \subfloat[][ViT-B/32]{
        \includegraphics[height=0.19\textwidth]{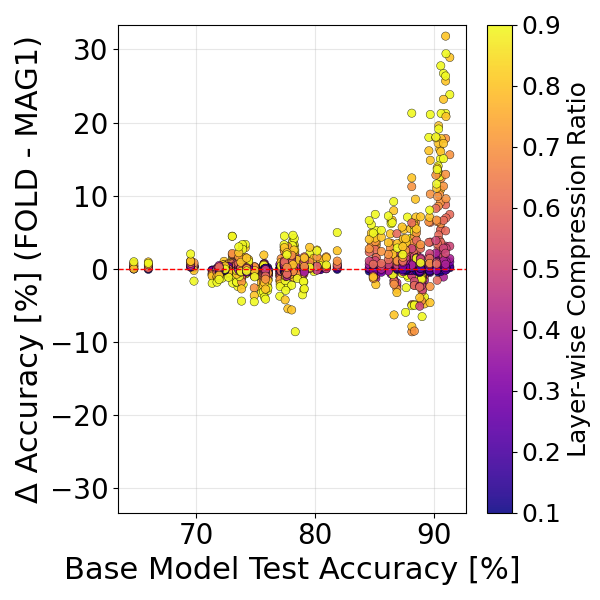}
    }
    \subfloat[][CLIP ViT-B/32]{
        \includegraphics[height=0.19\textwidth]{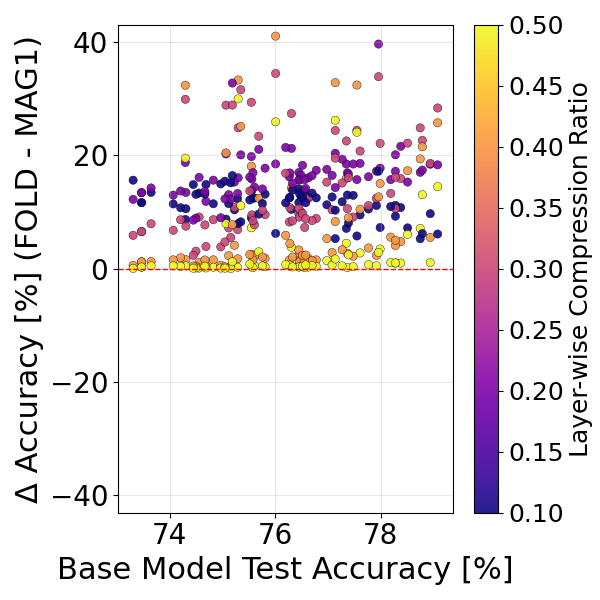}
    }

    \subfloat[][ResNet18, Adam]{
        \includegraphics[height=0.19\textwidth]{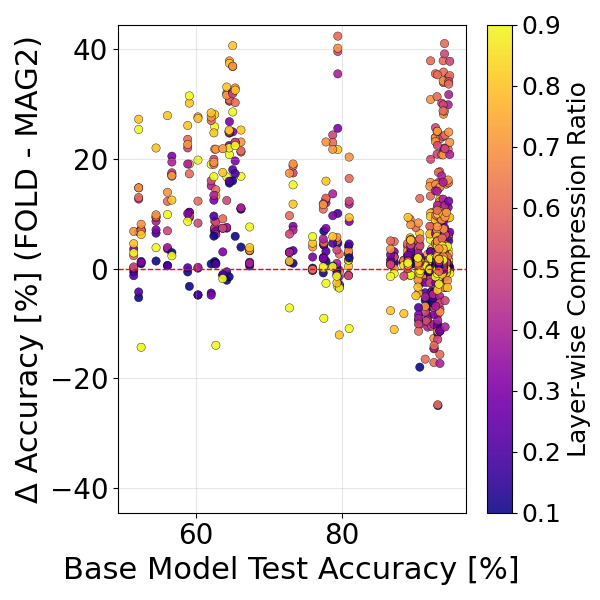}
    }
    \subfloat[][ResNet18, SGD]{
        \includegraphics[height=0.19\textwidth]{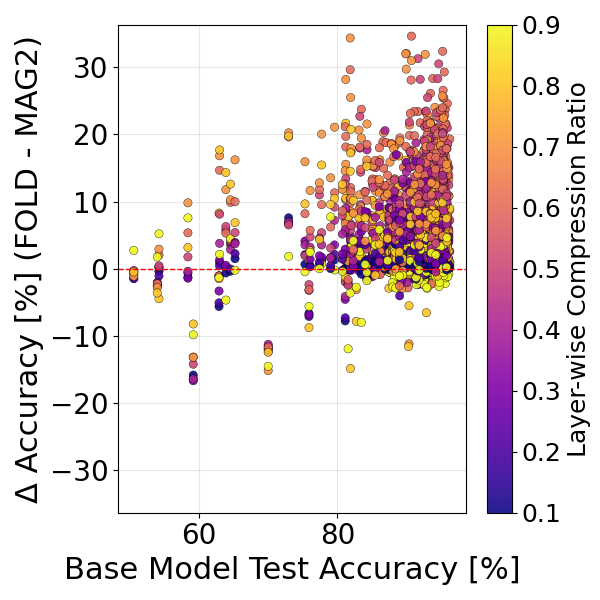}
    }
    \subfloat[][PreActResNet18]{
        \includegraphics[height=0.19\textwidth]{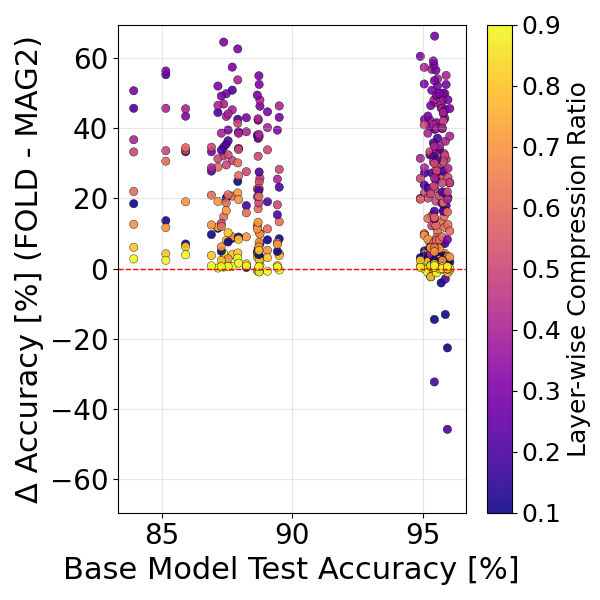}
    }
    \subfloat[][ViT-B/32]{
        \includegraphics[height=0.19\textwidth]{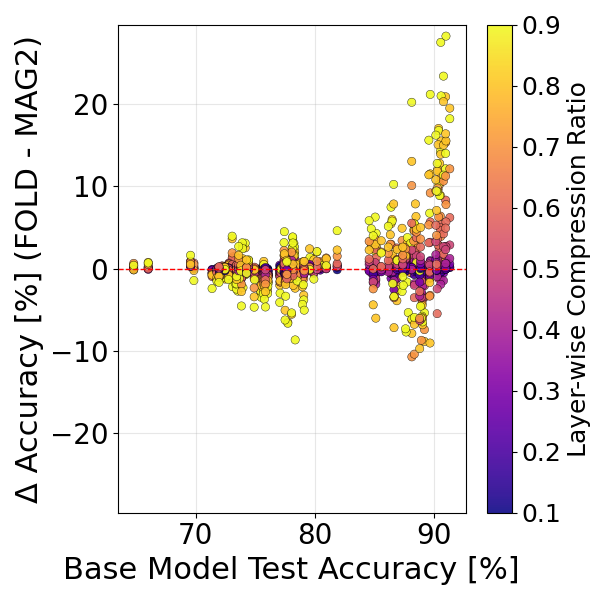}
    }
    \subfloat[][CLIP ViT-B/32]{
        \includegraphics[height=0.19\textwidth]{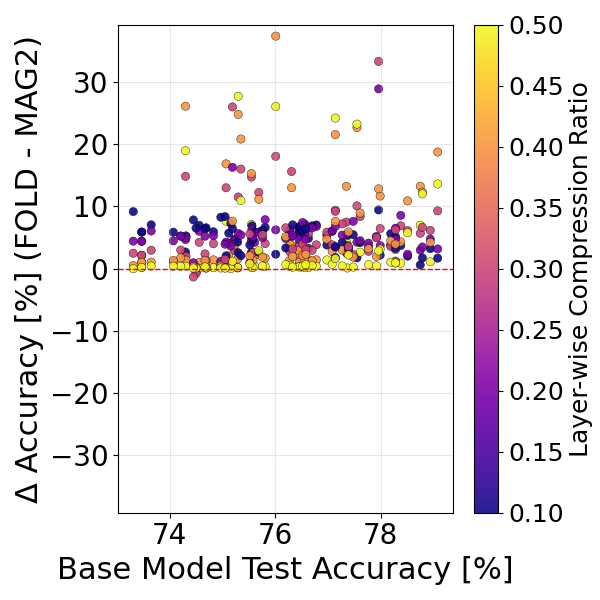}
    }
    \caption{\textbf{Uncompressed model accuracy vs. performance difference $\Delta \text{Accuracy}=\mathrm{Acc}{(\text{\fold})}-\mathrm{Acc}{(\text{\magprune})}$.} The same setup as in \Figref{fig:accuracy_FOLD_vs_MAG_L1} and \Figref{fig:accuracy_FOLD_vs_MAG_L2}.
    \textbf{Top row:} \magone,
    \textbf{Bottom row:} \magtwo. Model folding shows strong performance on models of different quality, with amplified effect on high-performing models (especially on ResNet18, SGD and ViT-B/32).}
     \label{fig:accuracy_compressed_vs_uncompressed}
\end{figure}

\subsection{Sharpness and Barrier Analysis}
\label{appx:sharpness}
We compute sharpness following the implementation by \citet{andriushchenko2023modernlookrelationshipsharpness}.  
Sharpness for CLIP is measured only on the final projection layer using $\sim$1000 images, while for PreActResNets it is computed over the full model and dataset.

Sharpness increases with compression ratio for all methods and architectures (\Figref{fig:sharpness_FOLD_vs_MAG}), reaching a peak before stronger compression pushes the model out of its original basin and into a flatter, lower-capacity region. This rise–then–fall pattern appears consistently in both PreActResNet and CLIP models, see \Figref{fig:sharpness_FOLD_vs_MAG}.

The correlation analysis in \Figref{fig:sharpness:corr} further supports this interpretation. Across the 200 compressed ResNet18 models, 50 compressed PreActResNet18s and 72 compressed CLIP models, \fold exhibits negative correlations between $\Delta$-sharpness and $\Delta$-accuracy ($\Delta\mathrm{Accuracy}=\mathrm{Acc}(\text{\fold})-\mathrm{Acc}(\text{\magprune})$). As shown in the scatter plots and correlation tables of \Figref{fig:sharpness:corr}, larger reductions in sharpness under \fold relative to \magprune are associated with larger accuracy gains. This relationship holds across the evaluated compression ratios, up to the point where one of the models leaves the original basin and sharpness becomes less informative.

In addition to the global sharpness trends discussed above, \Figref{fig:adam_sharpness_delta_vs_delta} and \Figref{fig:sgd_sharpness_delta_vs_delta} provide a more 
fine-grained view of how training hyperparameters influence the relationship between $\Delta$-sharpness and $\Delta$-accuracy under compression. For Adam-trained ResNet models (\Figref{fig:adam_sharpness_delta_vs_delta}), the scatter plots reveal a strong and stable negative correlation: whenever 
\fold produces lower sharpness than magnitude pruning, it also achieves higher accuracy across 
almost all pruning ratios. The structure of the point clouds, especially at high learning rates, shows that Adam’s adaptive scaling can induce highly anisotropic sharpness profiles, which in turn amplify the divergence between the compression trajectories of \fold and \magprune.

In contrast, SGD-trained models (\Figref{fig:sgd_sharpness_delta_vs_delta}) exhibit a weaker and more dispersed relationship 
between $\Delta$-sharpness and $\Delta$-accuracy, consistent with the flatter and more isotropic minima typically found by SGD. Under SGD, \fold often remains slightly flatter than 
magnitude pruning even when $\Delta$-sharpness $\approx 0$, explaining why \fold maintains a mild yet 
more weakly correlated accuracy advantage. The interaction with SAM and augmentation further differs across optimizers: SAM tightens the $\Delta$-sharpness distribution under SGD, 
stabilizing the performance gap in favor of \fold, while RandAug tends primarily to reduce variance without introducing strong directional trends.

These results highlight that the predictive power of sharpness for pruning outcomes is 
optimizer- and hyperparameter-dependent: sharpness differences are highly informative for Adam-trained networks but less so for SGD, even though \fold consistently follows a smoother and less disruptive compression path than magnitude-based pruning in both regimes.

These findings align with recent work linking compression and landscape geometry. AdaSAP \citep{bair2024adaptivesharpnessawarepruningrobust} treats pruning as a sharpness-aware process, and \citet{zhang2025sparseprunedeepnetwork} show that feasible pruning ratios depend on intrinsic flatness. Our results support this perspective: compression initially increases sharpness as degrees of freedom are removed within the same basin, but stronger compression forces the model into a flatter basin with reduced curvature. \fold follows this trajectory more smoothly, maintaining basin structure and yielding lower barriers than \magprune.

\begin{figure}[h]
     \centering
     \vskip -0.4cm
     \subfloat[][ResNet18, Adam, \fold]{
        \includegraphics[width=0.24\textwidth]{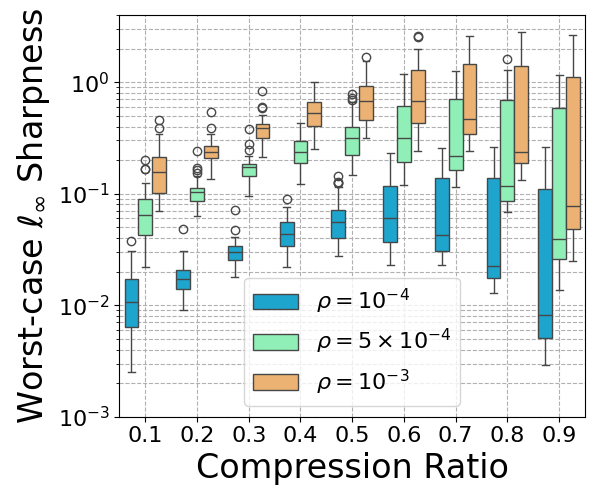}
    }
    \subfloat[][ResNet18, SGD, \fold]{
        \includegraphics[width=0.24\textwidth]{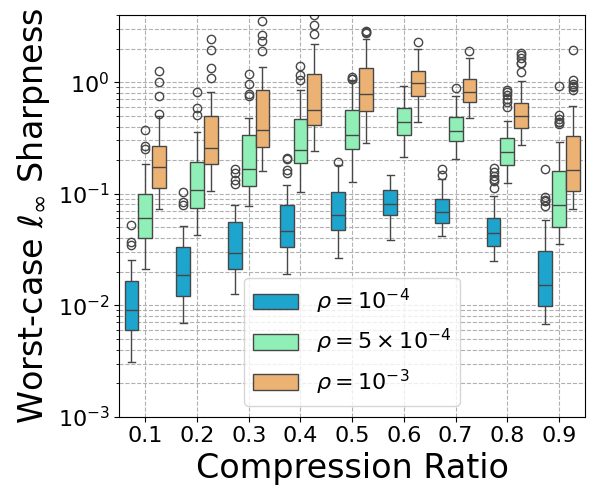}
    }
     \subfloat[][PreActResNet18, \fold]{
        \includegraphics[width=0.24\textwidth]{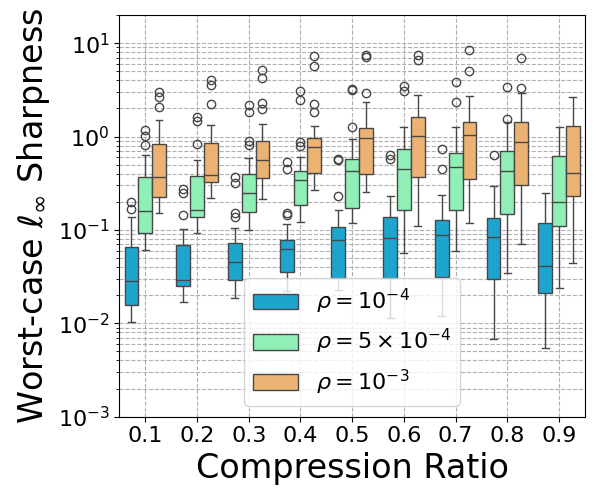}
    }
    \subfloat[][CLIP ViT-B/32, \fold]{
        \includegraphics[width=0.24\textwidth]{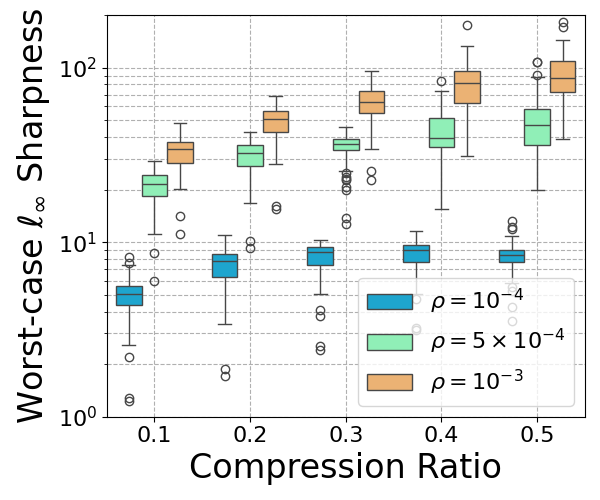}
    }

     \subfloat[][ResNet18, Adam, \magone]{
        \includegraphics[width=0.24\textwidth]{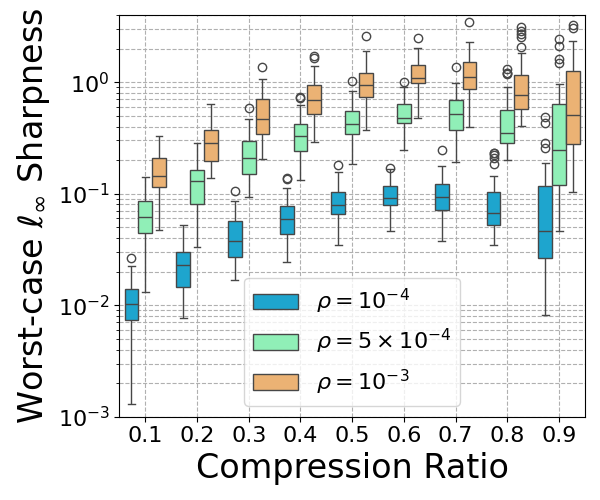}
    }
    \subfloat[][ResNet18, SGD, \magone]{
        \includegraphics[width=0.24\textwidth]{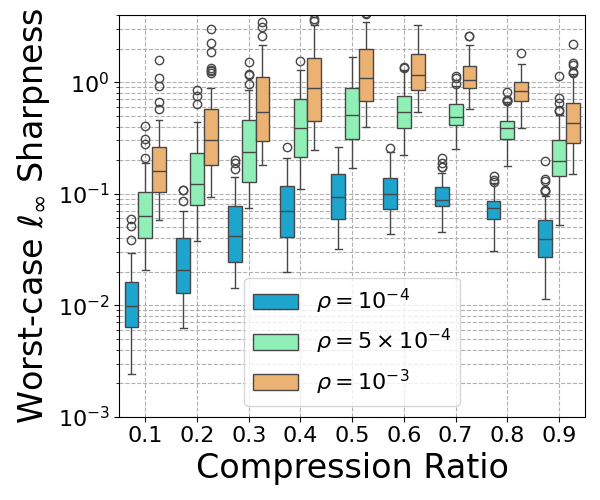}
    }
    \subfloat[][PreActResNet18, \magone]{
        \includegraphics[width=0.24\textwidth]{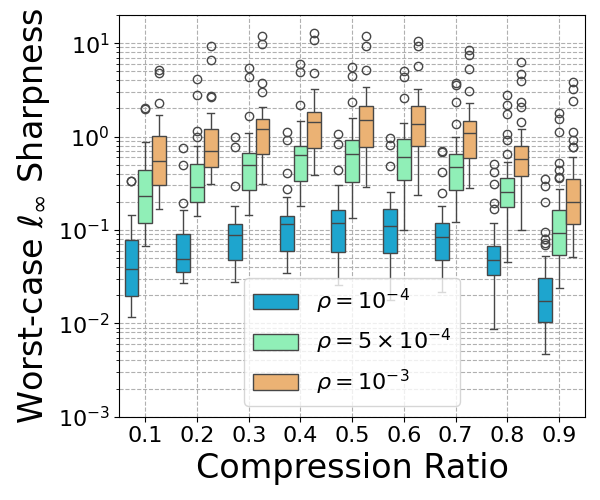}
    }
    \subfloat[][CLIP ViT-B/32, \magone]{
        \includegraphics[width=0.24\textwidth]{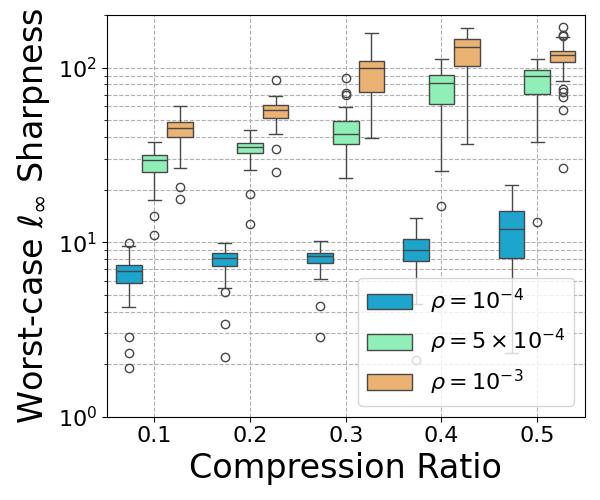}
    }
    \caption{\textbf{Worst-case $\ell_\infty$ sharpness as a function of compression ratio across architectures and pruning methods.}
    Each subplot reports the sharpness distribution for independently trained models at three perturbation radii 
    $(\rho = 10^{-4},\, 5\times10^{-4},\, 10^{-3})$. 
    Panels (a)--(d) show \fold sharpness for ResNet18 trained with Adam and SGD, PreActResNet18, and CLIP ViT-B/32, respectively. 
    Panels (e)--(h) show the corresponding results for \magone.
    Observed trends:
    (i) Sharpness generally increases with compression ratio up to moderate levels before flattening or dropping at extreme compression.
    (ii) \fold produces on-average lower sharpness than \magone.
    (iii) Transformer models (CLIP ViT-B/32) experience substantially sharper solutions under compression compared to residual networks.
    These patterns indicate that \fold maintains flatter loss landscapes across a wide range of settings, whereas \magprune more often drives models toward sharper and less stable minima.}
     \label{fig:sharpness_FOLD_vs_MAG}
\end{figure}

\begin{figure}[h]
\centering
\begin{minipage}[t]{0.24\linewidth}
    \centering
    \includegraphics[width=\linewidth]{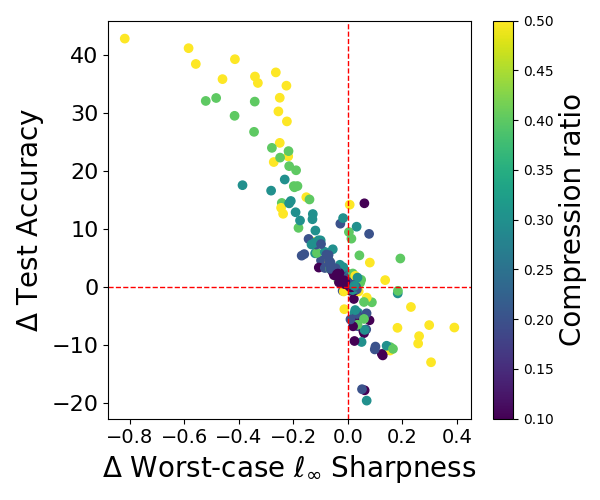}
\end{minipage}
\hfill
\begin{minipage}[t]{0.24\linewidth}
    \centering
    \includegraphics[width=\linewidth]{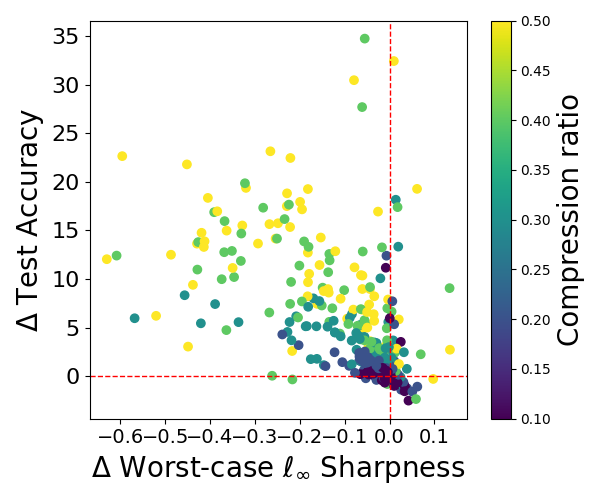}
\end{minipage}
\hfill
\begin{minipage}[t]{0.24\linewidth}
    \centering
    \includegraphics[width=\linewidth]{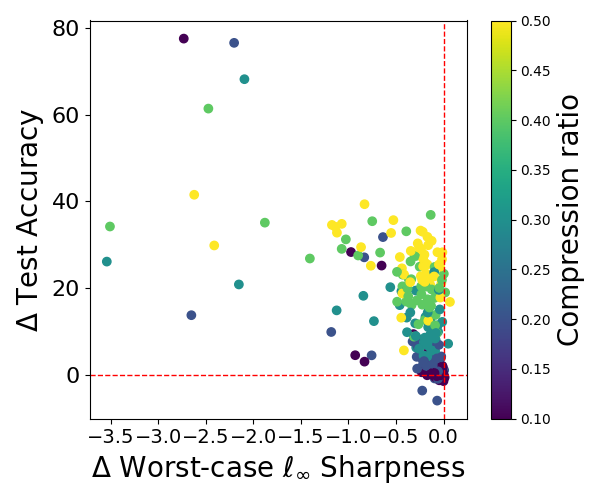}
\end{minipage}
\hfill
\begin{minipage}[t]{0.24\linewidth}
    \centering
    \includegraphics[width=\linewidth]{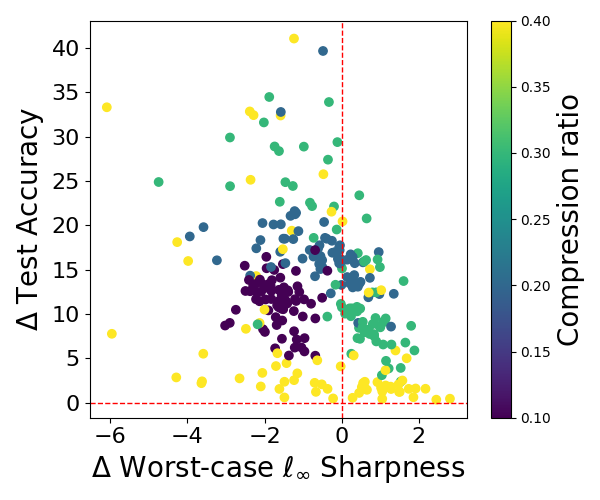}
\end{minipage}

\begin{minipage}[t]{0.24\linewidth}
\centering
\resizebox{1.00\textwidth}{!}{%
\begin{tabular}{c|cc}
\toprule
\rowcolor{lightblue}
\textbf{Ratio} & \textbf{Pearson} & \textbf{Spearman} \\
\midrule
10\% & -0.635 & -0.837 \\
20\% & -0.707 & -0.785 \\
30\% & -0.830 & -0.896 \\
40\% & -0.930 & -0.917 \\
50\% & -0.920 & -0.943 \\
60\% & -0.862 & -0.910 \\
70\% & -0.775 & -0.788 \\
80\% & -0.454 & -0.493 \\
90\% & 0.045  & 0.051  \\
\bottomrule
\end{tabular}
}
\caption*{\small ResNet18, Adam, \\\fold vs \magone}
\end{minipage}
\hfill
\begin{minipage}[t]{0.24\linewidth}
\centering
\resizebox{1.00\textwidth}{!}{%
\begin{tabular}{c|cc}
\toprule
\rowcolor{lightblue}
\textbf{Ratio} & \textbf{Pearson} & \textbf{Spearman} \\
\midrule
10\% & -0.219 & -0.437 \\
20\% & -0.488 & -0.546 \\
30\% & -0.505 & -0.654 \\
40\% & -0.628 & -0.710 \\
50\% & -0.535 & -0.665 \\
60\% & -0.369 & -0.403 \\
70\% & -0.399 & -0.419 \\
80\% & -0.277 & -0.163 \\
90\% & -0.168 & -0.120 \\
\bottomrule
\end{tabular}
}
\caption*{\small  ResNet18, SGD, \\\fold vs \magone}
\end{minipage}
\begin{minipage}[t]{0.24\linewidth}
\centering
\resizebox{1.00\textwidth}{!}{%
\begin{tabular}{c|cc}
\toprule
\rowcolor{lightblue}
\textbf{Ratio} & \textbf{Pearson} & \textbf{Spearman} \\
\midrule
10\% & -0.932 & -0.520 \\
20\% & -0.710 & -0.399 \\
30\% & -0.612 & -0.418 \\
40\% & -0.630 & -0.481 \\
50\% & -0.440 & -0.388 \\
60\% & -0.164 & -0.135 \\
70\% & 0.005  & 0.050  \\
80\% & 0.065  & 0.115  \\
90\% & 0.320  & 0.368  \\
\bottomrule
\end{tabular}
}
\caption*{\small PreActResNet18, \\\fold vs \magone}
\end{minipage}
\hfill
\begin{minipage}[t]{0.24\linewidth}
\centering
\resizebox{1.00\textwidth}{!}{%
\begin{tabular}{c|cc}
\toprule
\rowcolor{lightblue}
\textbf{Ratio} & \textbf{Pearson} & \textbf{Spearman} \\
\midrule
10\% & -0.359 & -0.400 \\
20\% & -0.270 & -0.222 \\
30\% & -0.520 & -0.462 \\
40\% & 0.148  & 0.058  \\
50\% & -0.139 & -0.383 \\
\bottomrule
\end{tabular}
}
\caption*{\small CLIP ViT-B/32, \\\fold vs \magone}
\end{minipage}
\caption{\textbf{Correlation between $\Delta$-sharpness and $\Delta$-accuracy across architectures and pruning baselines.}
Each column shows the relationship between pruning–induced differences in worst-case $\ell_\infty$ sharpness 
($\Delta$ sharpness = \fold $-$ \textsc{mag}) and differences in test accuracy ($\Delta$ accuracy = \fold $-$ \textsc{mag}), for the pair of pruning methods indicated below each plot. 
Color encodes the layer-wise compression ratio. 
The tables underneath each subplot report Pearson and Spearman correlations at every compression ratio, 
quantifying how predictive the sharpness difference is of the accuracy difference. Results are shown for \fold vs \magone for ResNet18 trained with Adam and SGD, PreActResNet18 and CLIP ViT-B/32. Statistics are computed over ~200 independently trained ResNets18 and 50 PreActResNet18 on CIFAR-10, and 72 CLIP ViT-B/32 models on ImageNet-1K. Correlations use the sharpness value at $\rho = 5\times10^{-4}$. Results at $\rho = 10^{-4}$ and $\rho = 10^{-3}$ are qualitatively very close.}
\label{fig:sharpness:corr}
\end{figure}

\begin{figure}[t]
     \centering
     \includegraphics[width=\textwidth]{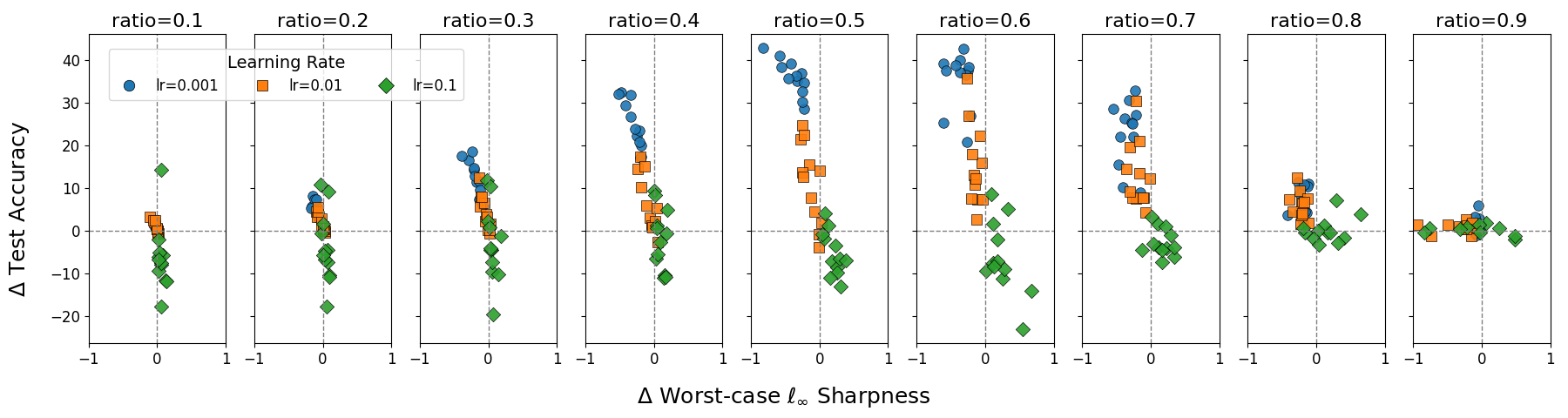}
        
    \includegraphics[width=\textwidth]{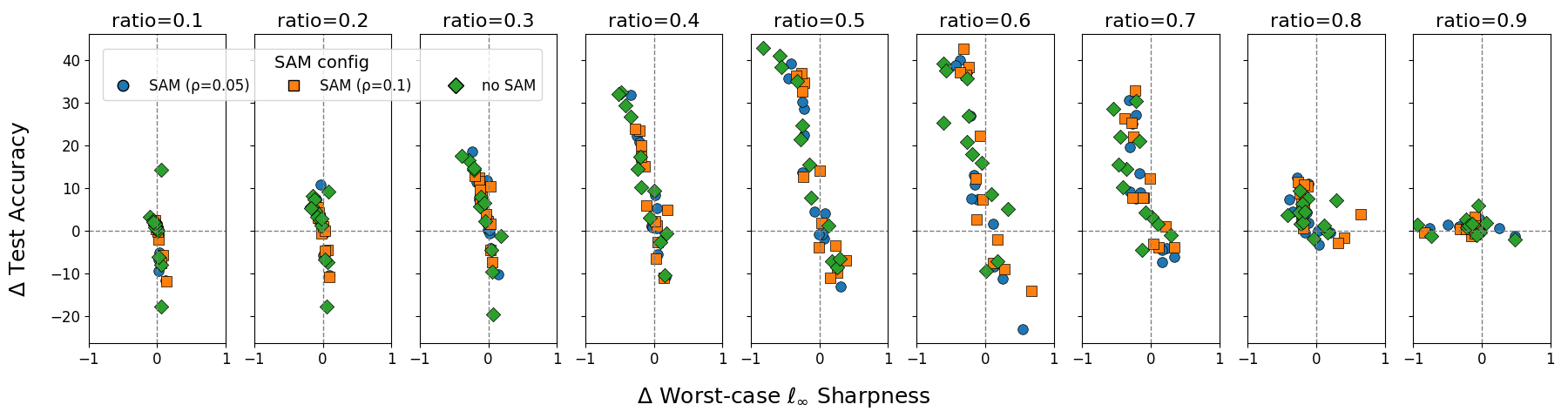}
        
    \includegraphics[width=\textwidth]{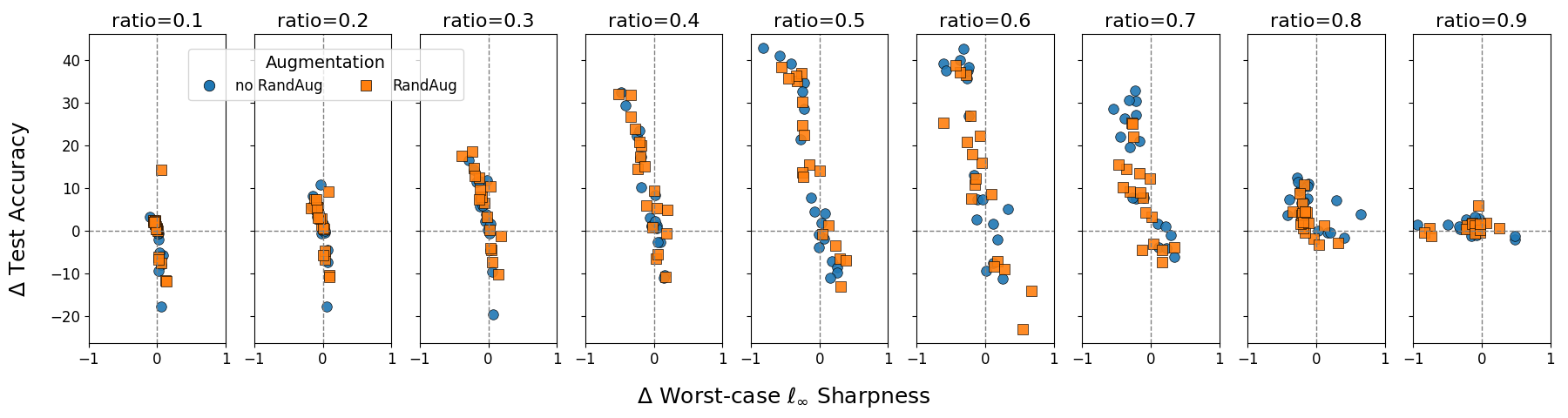}
    \caption{\textbf{Sharpness-accuracy trade-off between \fold and \magone for ResNet18 trained with Adam.} Each column corresponds to a layer-wise compression ratio (0.1--0.9), and the three rows group models by learning rate, SAM configuration (including~$\rho$), and RandAug usage. Points show $\Delta$ worst-case $\ell_\infty$ sharpness (\fold $-$ \magone) vs.\ $\Delta$ test accuracy (\fold $-$ \magone). 
    \textbf{Observations:} 
    (1) The difference in model sharpness strongly predicts the difference in performance between \fold and \magone across almost all pruning ratios. 
    (2) \emph{Learning rate:} higher learning rates lead to more dispersed sharpness changes. For models trained with Adam using high learning rates, \magone moves the model along a less sharp path compared to \fold, which struggles to catch up. However, the behavior flips for moderate and low learning rates.
    (3) \emph{SAM/}$\rho$: using SAM reduces the variability in the sharpness shift between methods, especially for larger $\rho$. $\Delta$ sharpness gets closer to zero.
    (4) \emph{RandAug:} augmentation show little specific visible trend.}
     \label{fig:adam_sharpness_delta_vs_delta}
\end{figure}

\begin{figure}[t]
     \centering
     \includegraphics[width=\textwidth]{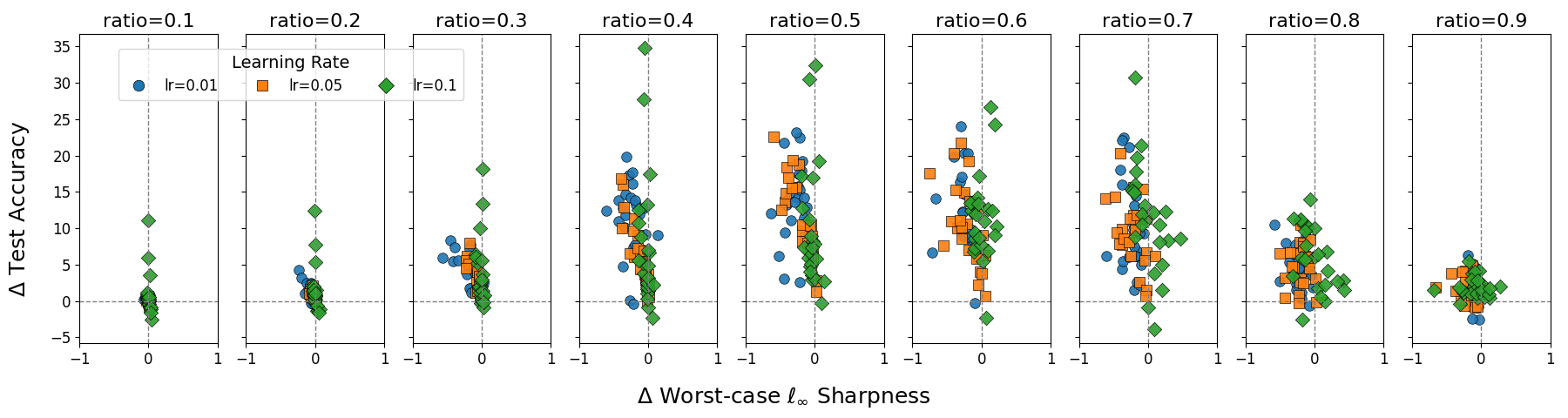}
        
    \includegraphics[width=\textwidth]{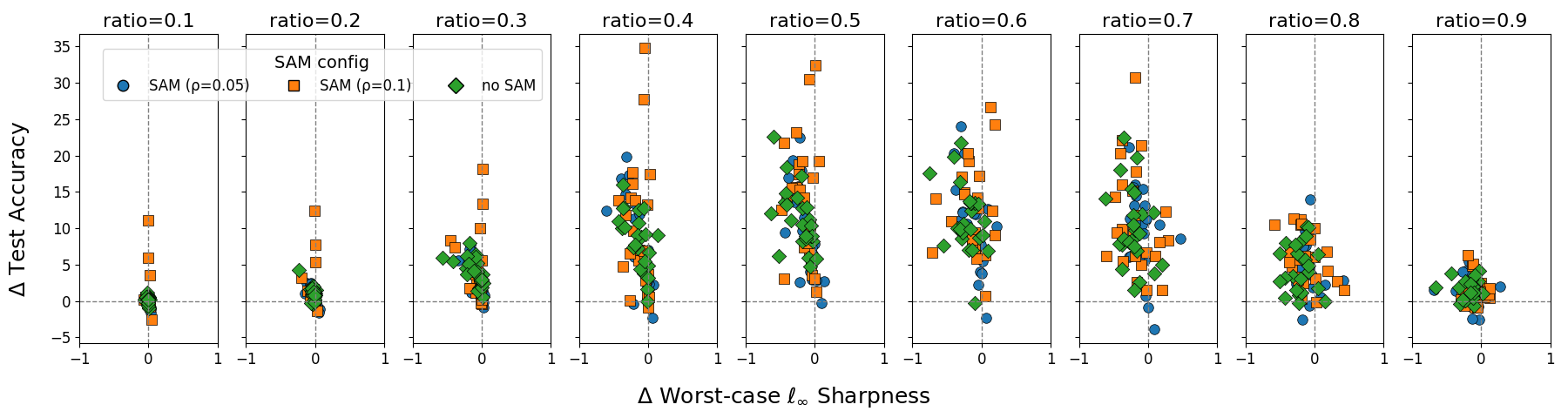}
        
    \includegraphics[width=\textwidth]{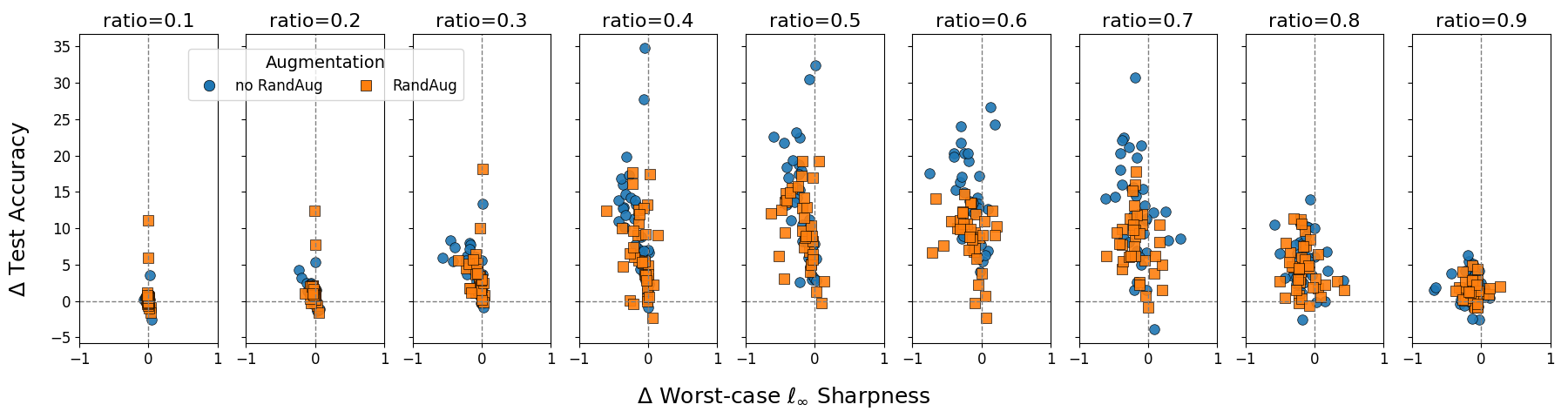}
    \caption{\textbf{Sharpness-accuracy trade-off between \fold and \magone for ResNet18 trained with SGD.}
    Each column corresponds to a layer-wise compression ratio (0.1--0.9), and the three rows group models by learning rate, SAM configuration (including~$\rho$), and RandAug usage. 
    Points show $\Delta$ worst-case $\ell_\infty$ sharpness (\fold $-$ \magone) vs.\ $\Delta$ test accuracy (\fold $-$ \magone).
    \textbf{Observations:}
    (1) For SGD, the relationship between $\Delta$ sharpness and $\Delta$ accuracy is visibly weaker and more scattered than for Adam, reflecting the flatter and more anisotropic minima found by SGD.  
    (2) \emph{Learning rate:} Unlike Adam, higher SGD learning rates do not systematically increase the sharpness gap between the methods. For most ratios, \fold tends to remain less sharp than \magone, producing positive $\Delta$ accuracy even when $\Delta$ sharpness is near zero.  
    (3) \emph{SAM/}$\rho$: SAM has a strong flattening effect under SGD—$\Delta$ sharpness clusters tightly around zero, and the accuracy advantage of \fold becomes more stable as $\rho$ increases.  
    (4) \emph{RandAug:} Augmentation increases robustness to pruning under SGD, reducing the spread in $\Delta$ accuracy and further weakening the sharpness--accuracy link. Overall, SGD-trained models exhibit a regime where \fold consistently follows a gentler sharpness trajectory than \magone, leading to a clearer accuracy advantage at moderate pruning ratios.}
     \label{fig:sgd_sharpness_delta_vs_delta}
\end{figure}

\subsection{Runtime overhead and equivalence of compressed models}
\label{subsec:runtime_overhead}

We profile both the compression procedures and the inference behavior of the resulting compressed models on a dedicated DGX A100 server equipped with dual-socket AMD EPYC 7742 CPUs (256 hardware threads) and 8× NVIDIA A100 80GB GPUs. All measurements use the \textsc{THOP} profiler\footnote{\url{https://github.com/ultralytics/thop}} and report compression time, peak memory during compression, per-batch latency, FLOPs, and peak forward-pass memory before and after compression.
For each architecture (PreActResNet18 and CLIP ViT-B/32), all compression methods generate the \emph{same} compressed network topology (identical channel counts and tensor shapes). Consequently, all methods yield identical FLOPs and nearly identical latencies, demonstrating that inference-time behavior is determined entirely by the resulting architecture, not by the choice of compression algorithm. \fold introduces a moderate one-off compression overhead, but its inference-time profile matches the other compressed models.

\begin{table}[h]
  \centering
  \resizebox{1.0\textwidth}{!}{%
  \begin{tabular}{lrrrrrrrr}
    \toprule
    \rowcolor{lightblue}
    Method & Params & Comp.~time [s] & Comp.~peak mem [MB] & Lat. [ms/batch] & Lat. [ms/img] & FLOPs [MFLOPs/img] & Fwd peak mem [MB] \\
    \midrule
    Original & 11{,}172{,}170 & --    & --      & 3.69 & 0.0288 & 557.65 & 214.30 \\
    \midrule
    \textsc{Fold}   & 4{,}008{,}346 & 9.48 & 157.47 & 3.17 & 0.0248 & 199.05 & 170.53 \\
    \textsc{Mag}    & 4{,}008{,}346 & 1.77 & 115.22 & 3.15 & 0.0246 & 199.05 & 169.82 \\
    \bottomrule
  \end{tabular}
  }
  \caption{Runtime characteristics of PreActResNet18 before and after compression ($64.1\%$ parameter reduction). Latency is measured for a full batch.
  FLOPs are reported per image. Comp.~time and comp.~peak mem refer to the
  overhead of running the compression method once.}
  \label{tab:runtime_preact}
\end{table}

Table~\ref{tab:runtime_preact} compares the original and compressed
PreActResNet18 at a layer-wise compression ratio of $0.4$ (\ie a
$64.1\%$ reduction in model parameters). Compression reduces FLOPs from
$557.65$~MFLOPs/image to $199.05$~MFLOPs/image (a $64.3\%$ reduction),
improves latency from $3.69$~ms/batch to roughly $3.15$~ms/batch, and lower peak forward-pass memory (from $214.30$~MB to about $170$~MB).

\begin{table}[h]
  \centering
  \resizebox{1.0\textwidth}{!}{%
  \begin{tabular}{lrrrrrrrr}
    \toprule
    \rowcolor{lightblue}
    Method & Params & Comp.~time [s] & Comp.~peak mem [MB] &
    Lat. [ms/batch] & Lat. [ms/img] & FLOPs [MFLOPs/img] & Fwd peak mem [MB] \\
    \midrule
    Original & 151{,}790{,}313 & --    & --      & 19.851 & 0.6203 & 2946.76 & 684.18 \\
    \midrule
    \textsc{Fold} & 140{,}447{,}253 & 92.833 & 681.23 & 17.343 & 0.5420 & 2379.98 & 636.97 \\
    \textsc{Mag}  & 140{,}447{,}253 & 2.627  & 625.61 & 17.372 & 0.5429 & 2379.98 & 637.88 \\
    \bottomrule
  \end{tabular}
  }
  \caption{Runtime characteristics of CLIP ViT-B/32 before and after compression ($7.47\%$ parameter reduction). Latency is measured for a full batch. FLOPs are reported per image. Comp.~time and Comp.~peak mem refer to the one-off overhead of running the compression method.}
  \label{tab:runtime_clip}
\end{table}

Table~\ref{tab:runtime_clip} shows the same evaluation for
CLIP ViT-B/32, where FFN blocks are compressed to with 20\% layer-wise compression ratio. Here, FLOPs decrease from $2946.76$~MFLOPs/image to
$2379.98$~MFLOPs/image (a $19.2\%$ reduction), and latency improves
from $19.8$~ms/batch to roughly $17.35$~ms/batch
(about $1.14\times$ speed-up). Again, \fold is the slowest method due to its iterative nature to compute $k$-means clusters, but the compressed models share the same FLOPs, memory and latency.

\begin{table}[h]
  \centering
  \resizebox{1.0\textwidth}{!}{%
  \begin{tabular}{lrrrrrrrrrrrrr}
    \toprule
    \rowcolor{lightblue}
    Layer &
    \multicolumn{1}{c}{Params\_fold} &
    \multicolumn{1}{c}{Params\_mag} &
    \multicolumn{1}{c}{$\Delta p$} &
    \multicolumn{1}{c}{FLOPs\_fold} &
    \multicolumn{1}{c}{FLOPs\_mag} &
    \multicolumn{1}{c}{$\Delta F$} &
    \multicolumn{1}{c}{Act\_fold} &
    \multicolumn{1}{c}{Act\_mag} &
    \multicolumn{1}{c}{$\Delta a$} &
    \multicolumn{1}{c}{NZ\_fold} &
    \multicolumn{1}{c}{NZ\_mag} &
    \multicolumn{1}{c}{$\Delta nz$} \\
    \midrule
    conv1                      &    1026 &    1026 &    0 &   1050624 &   1050624 &      0 &   38912 &   38912 &    0 &   38912 &   38912 &    0 \\
    layer1.0.conv1             &   12996 &   12996 &    0 &  13307904 &  13307904 &      0 &   38912 &   38912 &    0 &   38912 &   38912 &    0 \\
    layer1.0.conv2             &   12996 &   12996 &    0 &  13307904 &  13307904 &      0 &   38912 &   38912 &    0 &   38912 &   38912 &    0 \\
    layer1.1.conv1             &   12996 &   12996 &    0 &  13307904 &  13307904 &      0 &   38912 &   38912 &    0 &   38912 &   38912 &    0 \\
    layer1.1.conv2             &   12996 &   12996 &    0 &  13307904 &  13307904 &      0 &   38912 &   38912 &    0 &   38912 &   38912 &    0 \\
    layer2.0.conv1             &   25992 &   25992 &    0 &   6653952 &   6653952 &      0 &   19456 &   19456 &    0 &   19456 &   19456 &    0 \\
    layer2.0.conv2             &   51984 &   51984 &    0 &  13307904 &  13307904 &      0 &   19456 &   19456 &    0 &   19456 &   19456 &    0 \\
    layer2.0.shortcut.0        &    2888 &    2888 &    0 &    739328 &    739328 &      0 &   19456 &   19456 &    0 &   19456 &   19456 &    0 \\
    layer2.1.conv1             &   51984 &   51984 &    0 &  13307904 &  13307904 &      0 &   19456 &   19456 &    0 &   19456 &   19456 &    0 \\
    layer2.1.conv2             &   51984 &   51984 &    0 &  13307904 &  13307904 &      0 &   19456 &   19456 &    0 &   19456 &   19456 &    0 \\
    layer3.0.conv1             &  104652 &  104652 &    0 &   6697728 &   6697728 &      0 &    9792 &    9792 &    0 &    9792 &    9792 &    0 \\
    layer3.0.conv2             &  210681 &  210681 &    0 &  13483584 &  13483584 &      0 &    9792 &    9792 &    0 &    9792 &    9792 &    0 \\
    layer3.0.shortcut.0        &   11628 &   11628 &    0 &    744192 &    744192 &      0 &    9792 &    9792 &    0 &    9792 &    9792 &    0 \\
    layer3.1.conv1             &  210681 &  210681 &    0 &  13483584 &  13483584 &      0 &    9792 &    9792 &    0 &    9792 &    9792 &    0 \\
    layer3.1.conv2             &  210681 &  210681 &    0 &  13483584 &  13483584 &      0 &    9792 &    9792 &    0 &    9792 &    9792 &    0 \\
    layer4.0.conv1             &  422739 &  422739 &    0 &   6763824 &   6763824 &      0 &    4912 &    4912 &    0 &    4912 &    4912 &    0 \\
    layer4.0.conv2             &  848241 &  848241 &    0 &  13571856 &  13571856 &      0 &    4912 &    4912 &    0 &    4912 &    4912 &    0 \\
    layer4.0.shortcut.0        &   46971 &   46971 &    0 &    751536 &    751536 &      0 &    4912 &    4912 &    0 &    4912 &    4912 &    0 \\
    layer4.1.conv1             &  848241 &  848241 &    0 &  13571856 &  13571856 &      0 &    4912 &    4912 &    0 &    4912 &    4912 &    0 \\
    layer4.1.conv2             &  848241 &  848241 &    0 &  13571856 &  13571856 &      0 &    4912 &    4912 &    0 &    4912 &    4912 &    0 \\
    linear                     &    3080 &    3080 &    0 &      3070 &      3070 &      0 &      10 &      10 &    0 &      10 &      10 &    0 \\
    \midrule
    TOTALS                     & 4003678 & 4003678 &    0 & 197725902 & 197725902 &      0 &  365370 &  365370 &    0 &  365370 &  365370 &    0 \\
    \bottomrule
  \end{tabular}
  }
  \caption{Per-layer comparison of PreActResNet18 after \fold and \magtwo at compression ratio $0.4$. For each convolutional and linear layer we report parameters, per-image FLOPs, activation size, and the number of non-zero activations (effective activations). All per-layer differences are zero, confirming that parameters, FLOPs, activations, and effective activations are exactly matched between the two compressed models.}
  \label{tab:per_layer_preact}
\end{table}

\begin{table}[h]
  \vskip -1cm
  \centering
  \resizebox{0.95\textwidth}{!}{%
  \begin{tabular}{lrrrrrrrrrrrrr}
    \toprule
    \rowcolor{lightblue}
    Layer &
    \multicolumn{1}{c}{Params\_fold} &
    \multicolumn{1}{c}{Params\_mag} &
    \multicolumn{1}{c}{$\Delta p$} &
    \multicolumn{1}{c}{FLOPs\_fold} &
    \multicolumn{1}{c}{FLOPs\_mag} &
    \multicolumn{1}{c}{$\Delta F$} &
    \multicolumn{1}{c}{Act\_fold} &
    \multicolumn{1}{c}{Act\_mag} &
    \multicolumn{1}{c}{$\Delta a$} &
    \multicolumn{1}{c}{NZ\_fold} &
    \multicolumn{1}{c}{NZ\_mag} &
    \multicolumn{1}{c}{$\Delta nz$} \\
    \midrule
    classification\_head                                   &   513000 &   513000 &    0 &          0 &          0 &     0 &        0 &        0 &    0 &        0 &        0 &    0 \\
    transformer.resblocks.0.attn.out\_proj                 &   262656 &   262656 &    0 &          0 &          0 &     0 &        0 &        0 &    0 &        0 &        0 &    0 \\
    transformer.resblocks.0.mlp.c\_fc                      &  1050624 &  1050624 &    0 &          0 &          0 &     0 &        0 &        0 &    0 &        0 &        0 &    0 \\
    transformer.resblocks.0.mlp.c\_proj                    &  1049088 &  1049088 &    0 &          0 &          0 &     0 &        0 &        0 &    0 &        0 &        0 &    0 \\
    transformer.resblocks.1.attn.out\_proj                 &   262656 &   262656 &    0 &          0 &          0 &     0 &        0 &        0 &    0 &        0 &        0 &    0 \\
    transformer.resblocks.1.mlp.c\_fc                      &  1050624 &  1050624 &    0 &          0 &          0 &     0 &        0 &        0 &    0 &        0 &        0 &    0 \\
    transformer.resblocks.1.mlp.c\_proj                    &  1049088 &  1049088 &    0 &          0 &          0 &     0 &        0 &        0 &    0 &        0 &        0 &    0 \\
    transformer.resblocks.10.attn.out\_proj                &   262656 &   262656 &    0 &          0 &          0 &     0 &        0 &        0 &    0 &        0 &        0 &    0 \\
    transformer.resblocks.10.mlp.c\_fc                     &  1050624 &  1050624 &    0 &          0 &          0 &     0 &        0 &        0 &    0 &        0 &        0 &    0 \\
    transformer.resblocks.10.mlp.c\_proj                   &  1049088 &  1049088 &    0 &          0 &          0 &     0 &        0 &        0 &    0 &        0 &        0 &    0 \\
    transformer.resblocks.11.attn.out\_proj                &   262656 &   262656 &    0 &          0 &          0 &     0 &        0 &        0 &    0 &        0 &        0 &    0 \\
    transformer.resblocks.11.mlp.c\_fc                     &  1050624 &  1050624 &    0 &          0 &          0 &     0 &        0 &        0 &    0 &        0 &        0 &    0 \\
    transformer.resblocks.11.mlp.c\_proj                   &  1049088 &  1049088 &    0 &          0 &          0 &     0 &        0 &        0 &    0 &        0 &        0 &    0 \\
    transformer.resblocks.2.attn.out\_proj                 &   262656 &   262656 &    0 &          0 &          0 &     0 &        0 &        0 &    0 &        0 &        0 &    0 \\
    transformer.resblocks.2.mlp.c\_fc                      &  1050624 &  1050624 &    0 &          0 &          0 &     0 &        0 &        0 &    0 &        0 &        0 &    0 \\
    transformer.resblocks.2.mlp.c\_proj                    &  1049088 &  1049088 &    0 &          0 &          0 &     0 &        0 &        0 &    0 &        0 &        0 &    0 \\
    transformer.resblocks.3.attn.out\_proj                 &   262656 &   262656 &    0 &          0 &          0 &     0 &        0 &        0 &    0 &        0 &        0 &    0 \\
    transformer.resblocks.3.mlp.c\_fc                      &  1050624 &  1050624 &    0 &          0 &          0 &     0 &        0 &        0 &    0 &        0 &        0 &    0 \\
    transformer.resblocks.3.mlp.c\_proj                    &  1049088 &  1049088 &    0 &          0 &          0 &     0 &        0 &        0 &    0 &        0 &        0 &    0 \\
    transformer.resblocks.4.attn.out\_proj                 &   262656 &   262656 &    0 &          0 &          0 &     0 &        0 &        0 &    0 &        0 &        0 &    0 \\
    transformer.resblocks.4.mlp.c\_fc                      &  1050624 &  1050624 &    0 &          0 &          0 &     0 &        0 &        0 &    0 &        0 &        0 &    0 \\
    transformer.resblocks.4.mlp.c\_proj                    &  1049088 &  1049088 &    0 &          0 &          0 &     0 &        0 &        0 &    0 &        0 &        0 &    0 \\
    transformer.resblocks.5.attn.out\_proj                 &   262656 &   262656 &    0 &          0 &          0 &     0 &        0 &        0 &    0 &        0 &        0 &    0 \\
    transformer.resblocks.5.mlp.c\_fc                      &  1050624 &  1050624 &    0 &          0 &          0 &     0 &        0 &        0 &    0 &        0 &        0 &    0 \\
    transformer.resblocks.5.mlp.c\_proj                    &  1049088 &  1049088 &    0 &          0 &          0 &     0 &        0 &        0 &    0 &        0 &        0 &    0 \\
    transformer.resblocks.6.attn.out\_proj                 &   262656 &   262656 &    0 &          0 &          0 &     0 &        0 &        0 &    0 &        0 &        0 &    0 \\
    transformer.resblocks.6.mlp.c\_fc                      &  1050624 &  1050624 &    0 &          0 &          0 &     0 &        0 &        0 &    0 &        0 &        0 &    0 \\
    transformer.resblocks.6.mlp.c\_proj                    &  1049088 &  1049088 &    0 &          0 &          0 &     0 &        0 &        0 &    0 &        0 &        0 &    0 \\
    transformer.resblocks.7.attn.out\_proj                 &   262656 &   262656 &    0 &          0 &          0 &     0 &        0 &        0 &    0 &        0 &        0 &    0 \\
    transformer.resblocks.7.mlp.c\_fc                      &  1050624 &  1050624 &    0 &          0 &          0 &     0 &        0 &        0 &    0 &        0 &        0 &    0 \\
    transformer.resblocks.7.mlp.c\_proj                    &  1049088 &  1049088 &    0 &          0 &          0 &     0 &        0 &        0 &    0 &        0 &        0 &    0 \\
    transformer.resblocks.8.attn.out\_proj                 &   262656 &   262656 &    0 &          0 &          0 &     0 &        0 &        0 &    0 &        0 &        0 &    0 \\
    transformer.resblocks.8.mlp.c\_fc                      &  1050624 &  1050624 &    0 &          0 &          0 &     0 &        0 &        0 &    0 &        0 &        0 &    0 \\
    transformer.resblocks.8.mlp.c\_proj                    &  1049088 &  1049088 &    0 &          0 &          0 &     0 &        0 &        0 &    0 &        0 &        0 &    0 \\
    transformer.resblocks.9.attn.out\_proj                 &   262656 &   262656 &    0 &          0 &          0 &     0 &        0 &        0 &    0 &        0 &        0 &    0 \\
    transformer.resblocks.9.mlp.c\_fc                      &  1050624 &  1050624 &    0 &          0 &          0 &     0 &        0 &        0 &    0 &        0 &        0 &    0 \\
    transformer.resblocks.9.mlp.c\_proj                    &  1049088 &  1049088 &    0 &          0 &          0 &     0 &        0 &        0 &    0 &        0 &        0 &    0 \\
    visual.conv1                                            & 2359296 & 2359296 &    0 & 115605504 & 115605504 &     0 &    37632 &    37632 &    0 &    37632 &    37632 &    0 \\
    visual.transformer.resblocks.0.attn.out\_proj          &   590592 &   590592 &    0 &         0 &         0 &     0 &        0 &        0 &    0 &        0 &        0 &    0 \\
    visual.transformer.resblocks.0.mlp.c\_fc               &  1889433 &  1889433 &    0 &  94348800 &  94348800 &     0 &   122850 &   122850 &    0 &   122850 &   122850 &    0 \\
    visual.transformer.resblocks.0.mlp.c\_proj             &  1887744 &  1887744 &    0 &  94348800 &  94348800 &     0 &    38400 &    38400 &    0 &    38400 &    38400 &    0 \\
    visual.transformer.resblocks.1.attn.out\_proj          &   590592 &   590592 &    0 &         0 &         0 &     0 &        0 &        0 &    0 &        0 &        0 &    0 \\
    visual.transformer.resblocks.1.mlp.c\_fc               &  1889433 &  1889433 &    0 &  94348800 &  94348800 &     0 &   122850 &   122850 &    0 &   122850 &   122850 &    0 \\
    visual.transformer.resblocks.1.mlp.c\_proj             &  1887744 &  1887744 &    0 &  94348800 &  94348800 &     0 &    38400 &    38400 &    0 &    38400 &    38400 &    0 \\
    visual.transformer.resblocks.10.attn.out\_proj         &   590592 &   590592 &    0 &         0 &         0 &     0 &        0 &        0 &    0 &        0 &        0 &    0 \\
    visual.transformer.resblocks.10.mlp.c\_fc              &  1889433 &  1889433 &    0 &  94348800 &  94348800 &     0 &   122850 &   122850 &    0 &   122850 &   122850 &    0 \\
    visual.transformer.resblocks.10.mlp.c\_proj            &  1887744 &  1887744 &    0 &  94348800 &  94348800 &     0 &    38400 &    38400 &    0 &    38400 &    38400 &    0 \\
    visual.transformer.resblocks.11.attn.out\_proj         &   590592 &   590592 &    0 &         0 &         0 &     0 &        0 &        0 &    0 &        0 &        0 &    0 \\
    visual.transformer.resblocks.11.mlp.c\_fc              &  1889433 &  1889433 &    0 &  94348800 &  94348800 &     0 &   122850 &   122850 &    0 &   122850 &   122850 &    0 \\
    visual.transformer.resblocks.11.mlp.c\_proj            &  1887744 &  1887744 &    0 &  94348800 &  94348800 &     0 &    38400 &    38400 &    0 &    38400 &    38400 &    0 \\
    visual.transformer.resblocks.2.attn.out\_proj          &   590592 &   590592 &    0 &         0 &         0 &     0 &        0 &        0 &    0 &        0 &        0 &    0 \\
    visual.transformer.resblocks.2.mlp.c\_fc               &  1889433 &  1889433 &    0 &  94348800 &  94348800 &     0 &   122850 &   122850 &    0 &   122850 &   122850 &    0 \\
    visual.transformer.resblocks.2.mlp.c\_proj             &  1887744 &  1887744 &    0 &  94348800 &  94348800 &     0 &    38400 &    38400 &    0 &    38400 &    38400 &    0 \\
    visual.transformer.resblocks.3.attn.out\_proj          &   590592 &   590592 &    0 &         0 &         0 &     0 &        0 &        0 &    0 &        0 &        0 &    0 \\
    visual.transformer.resblocks.3.mlp.c\_fc               &  1889433 &  1889433 &    0 &  94348800 &  94348800 &     0 &   122850 &   122850 &    0 &   122850 &   122850 &    0 \\
    visual.transformer.resblocks.3.mlp.c\_proj             &  1887744 &  1887744 &    0 &  94348800 &  94348800 &     0 &    38400 &    38400 &    0 &    38400 &    38400 &    0 \\
    visual.transformer.resblocks.4.attn.out\_proj          &   590592 &   590592 &    0 &         0 &         0 &     0 &        0 &        0 &    0 &        0 &        0 &    0 \\
    visual.transformer.resblocks.4.mlp.c\_fc               &  1889433 &  1889433 &    0 &  94348800 &  94348800 &     0 &   122850 &   122850 &    0 &   122850 &   122850 &    0 \\
    visual.transformer.resblocks.4.mlp.c\_proj             &  1887744 &  1887744 &    0 &  94348800 &  94348800 &     0 &    38400 &    38400 &    0 &    38400 &    38400 &    0 \\
    visual.transformer.resblocks.5.attn.out\_proj          &   590592 &   590592 &    0 &         0 &         0 &     0 &        0 &        0 &    0 &        0 &        0 &    0 \\
    visual.transformer.resblocks.5.mlp.c\_fc               &  1889433 &  1889433 &    0 &  94348800 &  94348800 &     0 &   122850 &   122850 &    0 &   122850 &   122850 &    0 \\
    visual.transformer.resblocks.5.mlp.c\_proj             &  1887744 &  1887744 &    0 &  94348800 &  94348800 &     0 &    38400 &    38400 &    0 &    38400 &    38400 &    0 \\
    visual.transformer.resblocks.6.attn.out\_proj          &   590592 &   590592 &    0 &         0 &         0 &     0 &        0 &        0 &    0 &        0 &        0 &    0 \\
    visual.transformer.resblocks.6.mlp.c\_fc               &  1889433 &  1889433 &    0 &  94348800 &  94348800 &     0 &   122850 &   122850 &    0 &   122850 &   122850 &    0 \\
    visual.transformer.resblocks.6.mlp.c\_proj             &  1887744 &  1887744 &    0 &  94348800 &  94348800 &     0 &    38400 &    38400 &    0 &    38400 &    38400 &    0 \\
    visual.transformer.resblocks.7.attn.out\_proj          &   590592 &   590592 &    0 &         0 &         0 &     0 &        0 &        0 &    0 &        0 &        0 &    0 \\
    visual.transformer.resblocks.7.mlp.c\_fc               &  1889433 &  1889433 &    0 &  94348800 &  94348800 &     0 &   122850 &   122850 &    0 &   122850 &   122850 &    0 \\
    visual.transformer.resblocks.7.mlp.c\_proj             &  1887744 &  1887744 &    0 &  94348800 &  94348800 &     0 &    38400 &    38400 &    0 &    38400 &    38400 &    0 \\
    visual.transformer.resblocks.8.attn.out\_proj          &   590592 &   590592 &    0 &         0 &         0 &     0 &        0 &        0 &    0 &        0 &        0 &    0 \\
    visual.transformer.resblocks.8.mlp.c\_fc               &  1889433 &  1889433 &    0 &  94348800 &  94348800 &     0 &   122850 &   122850 &    0 &   122850 &   122850 &    0 \\
    visual.transformer.resblocks.8.mlp.c\_proj             &  1887744 &  1887744 &    0 &  94348800 &  94348800 &     0 &    38400 &    38400 &    0 &    38400 &    38400 &    0 \\
    visual.transformer.resblocks.9.attn.out\_proj          &   590592 &   590592 &    0 &         0 &         0 &     0 &        0 &        0 &    0 &        0 &        0 &    0 \\
    visual.transformer.resblocks.9.mlp.c\_fc               &  1889433 &  1889433 &    0 &  94348800 &  94348800 &     0 &   122850 &   122850 &    0 &   122850 &   122850 &    0 \\
    visual.transformer.resblocks.9.mlp.c\_proj             &  1887744 &  1887744 &    0 &  94348800 &  94348800 &     0 &    38400 &    38400 &    0 &    38400 &    38400 &    0 \\
    \midrule
    TOTALS                                                 & 83633940 & 83633940 &    0 & 2379976704 & 2379976704 &     0 & 1972632 & 1972632 &    0 & 1972632 & 1972632 &    0 \\
    \bottomrule
  \end{tabular}
  }
  \caption{\small{Per-layer comparison of CLIP ViT-B/32 after \fold and \magtwo at compression ratio $0.2$ (global parameter reduction $7.47\%$). The table reports all convolutional and linear layers in the vision transformer and classification head, including their parameters, per-image FLOPs, activation sizes, and effective activations. 
  The \texttt{transformer.resblocks.*} modules belong to CLIP’s text encoder. Because the ImageNet-1k fine-tuned variant evaluates only the vision encoder and classification head, the text encoder is not part of the forward graph. \textsc{THOP} therefore records zero FLOPs and zero activations for these layers, while their parameters remain included in the model.
  Note that the per-layer totals (\(\approx 8.36\times10^7\) parameters) are smaller than the full model parameter count (\(\approx 1.40\times10^8\)) because this table excludes components without FLOPs, such as token embeddings, positional embeddings, and LayerNorm parameters, which are included in the global counts but not part of the per-layer FLOP/activation analysis. 
  Several projection layers inside the attention blocks show zero FLOPs because CLIP implements attention using fused operations; these operations are profiled at the block level by \textsc{THOP} rather than attributed to the individual Linear submodules.
  All per-layer differences are zero, showing that \fold and \magtwo produce structurally identical compressed models on every layer affected by compression.}}

  \label{tab:per_layer_clip}
\end{table}

\subsection{Impact of one rank slack and singleton folding}
\label{subsec:one_rank_slack}

In this section we separate two effects in our pruning vs. folding comparison: (i) the influence of the one-rank slack in 
Theorems~\ref{thm:folding-existence}–\ref{thm:folding-optimal}, and 
(ii) the intrinsic difference between pruning and folding as projection operators. We therefore contrast the gain from increasing the pruning rank from $k$ to $k{+}1$ (blue curves in \Figref{fig:adam_one_rank_slack:rank}) with the gain from replacing pruning by folding at the same nominal rank (orange curves).  
This isolates the contribution of rank from the contribution of the projection geometry.

\Figref{fig:adam_one_rank_slack:rank} shows both effects: 
for each weight matrix $\mathbf{W}$ (in every layer of ViT or ResNet18), it plots the relative Frobenius error change when the retained rank increases by one (blue) versus when the method changes from pruning to folding at fixed rank (orange), as a function of~$k$.

\[
\Delta_{\text{rank}}(k) \;=\; 
\frac{\|\mathbf{W} - \mathbf{W}_p^{(k)}\|_F \;-\; \|\mathbf{W} - \mathbf{W}_p^{(k+1)}\|_F}{\|\mathbf{W}\|_F}.
\]
This quantity measures the improvement obtained when increasing the retained rank from $k$ to $k+1$ within magnitude pruning \magtwo. It isolates the rank slack effect. Across all examined layers, the improvement from a single additional retained channel is small, especially in deeper layers.

\[
\Delta_{\text{method}}(k) \;=\;
\frac{\|\mathbf{W} - \mathbf{W}_p^{(k)}\|_F \;-\; \|\mathbf{W} - \mathbf{W}^{\star(k)}_f\|_F}{\|\mathbf{W}\|_F}.
\]
This measures the gain obtained by switching from structured magnitude pruning to optimal folding at the same rank $k$. The improvements are one to two orders of magnitude larger than the corresponding $\Delta_{\text{rank}}$ values for nearly all layers.

The results empirically support the clarification presented in the rebuttal: although Theorems~\ref{thm:folding-existence}-\ref{thm:folding-optimal} compare pruning at rank $k$ to folding at rank $k{+}1$, the contribution of the rank difference is negligible in practice. The blue curves show that $\|\mathbf{W} - \mathbf{W}_p^{(k)}\|_F$ changes only minimally when $k \!\to\! k{+}1$, while the orange curves demonstrate that folding provides a substantially tighter approximation than pruning at comparable compression levels. This confirms that the practical advantage of folding arises primarily from the richer family of cluster-based projections rather than the added rank.

\begin{figure}[h]
     \centering
     \includegraphics[width=.88\textwidth]{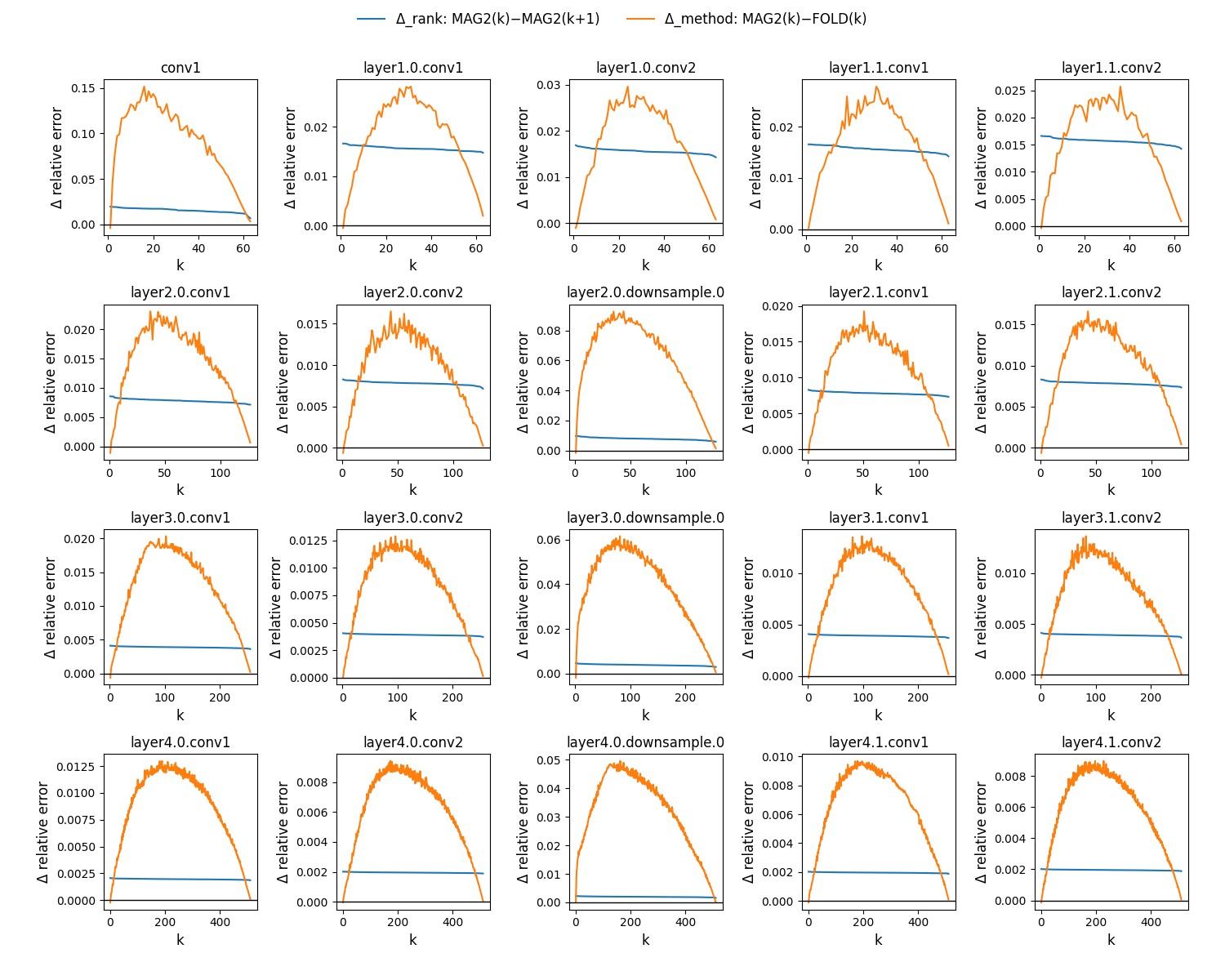}

     \includegraphics[width=.68\textwidth]{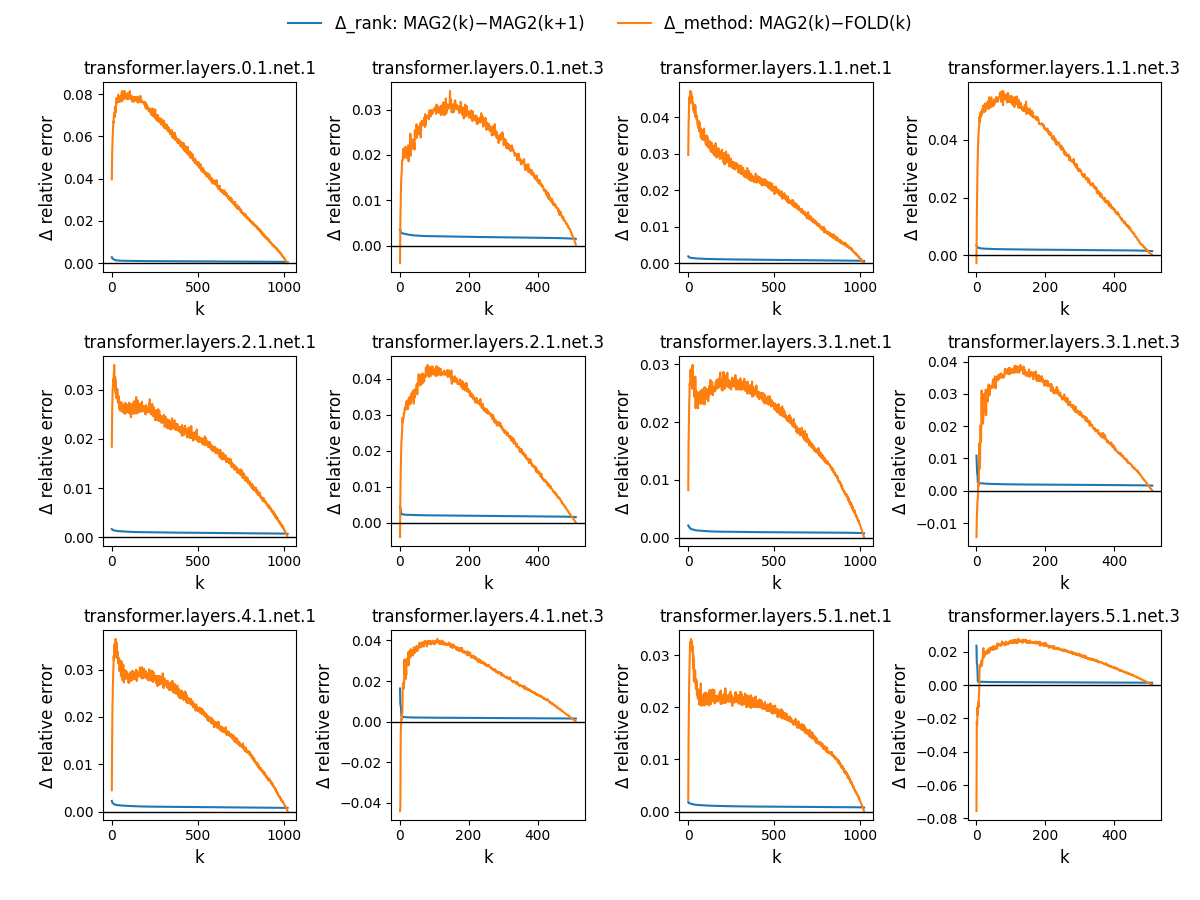}
    \caption{\small{\textbf{Effect of increasing the retained rank by one is significantly lower than changing the compression method from \magprune to \fold.} Comparison of (i) the effect of increasing the retained rank by one (blue curves) and (ii) the effect of switching from \magtwo to \fold at the same nominal rank (orange curves). Each panel corresponds to a single weight matrix $\mathbf{W}$ in ResNet18 convolutional layers (\textbf{top}) and ViT-B/32 FFN layers (\textbf{bottom}) and shows the relative Frobenius error difference $\Delta$ as a function of the retained rank $k$.}}
     \label{fig:adam_one_rank_slack:rank}
\end{figure}

\Figref{fig:adam_one_rank_slack:opt} reports the relative squared error for pruning, special folding $\mathbf{W}_f'$, and optimal folding $\mathbf{W}_f^*$ across all layers of ResNet18 and ViT-B/32 on CIFAR10, evaluated at multiple keep ratios. Three consistent phenomena appear:

\begin{itemize}
    \item \textbf{Pruning always yields the largest error.}
    For every layer and every keep ratio, the pruning curves lie above both folding curves. This confirms empirically that pruning introduces the largest distortion of the original weight matrix.

    \item \textbf{Special folding $\mathbf{W}_f'$ strictly improves over pruning.}
    The construction used in Theorem~\ref{thm:folding-existence}, obtained by merging all pruned rows, leads to smaller reconstruction error for all layers and all keep ratios. This empirically validates the first inequality $\mathrm{error}(\mathbf{W}_p) \ge \mathrm{error}(\mathbf{W}_f')$.

    \item \textbf{Optimal folding $\mathbf{W}_f^*$ achieves the smallest error.}
    The $k$-means solution consistently attains the lowest error, verifying the second inequality $\mathrm{error}(\mathbf{W}_f') \ge \mathrm{error}(\mathbf{W}_f^*)$.
\end{itemize}

Importantly, the gap between pruning and both folding methods is 
larger than the very small difference induced by adding 
a single additional cluster (\ie increasing the rank from $k$ to $k+1$). This supports the clarification made in the rebuttal: 
the practical advantage of folding does not stem from the $+1$ change in rank, but from the richer family of cluster-based projections that folding can realize.

The observed ordering, \ie $\mathrm{error}(\mathbf{W}_p) > \mathrm{error}(\mathbf{W}_f') > \mathrm{error}(\mathbf{W}_f^*)$, holds uniformly across all layers of ResNet18 as well as FFN layers of ViT-B/32. This indicates that the theoretical inequalities are not only valid in principle but also manifest strongly and consistently in real trained models.

\begin{figure}[h]
     \centering
     \includegraphics[width=.85\textwidth]{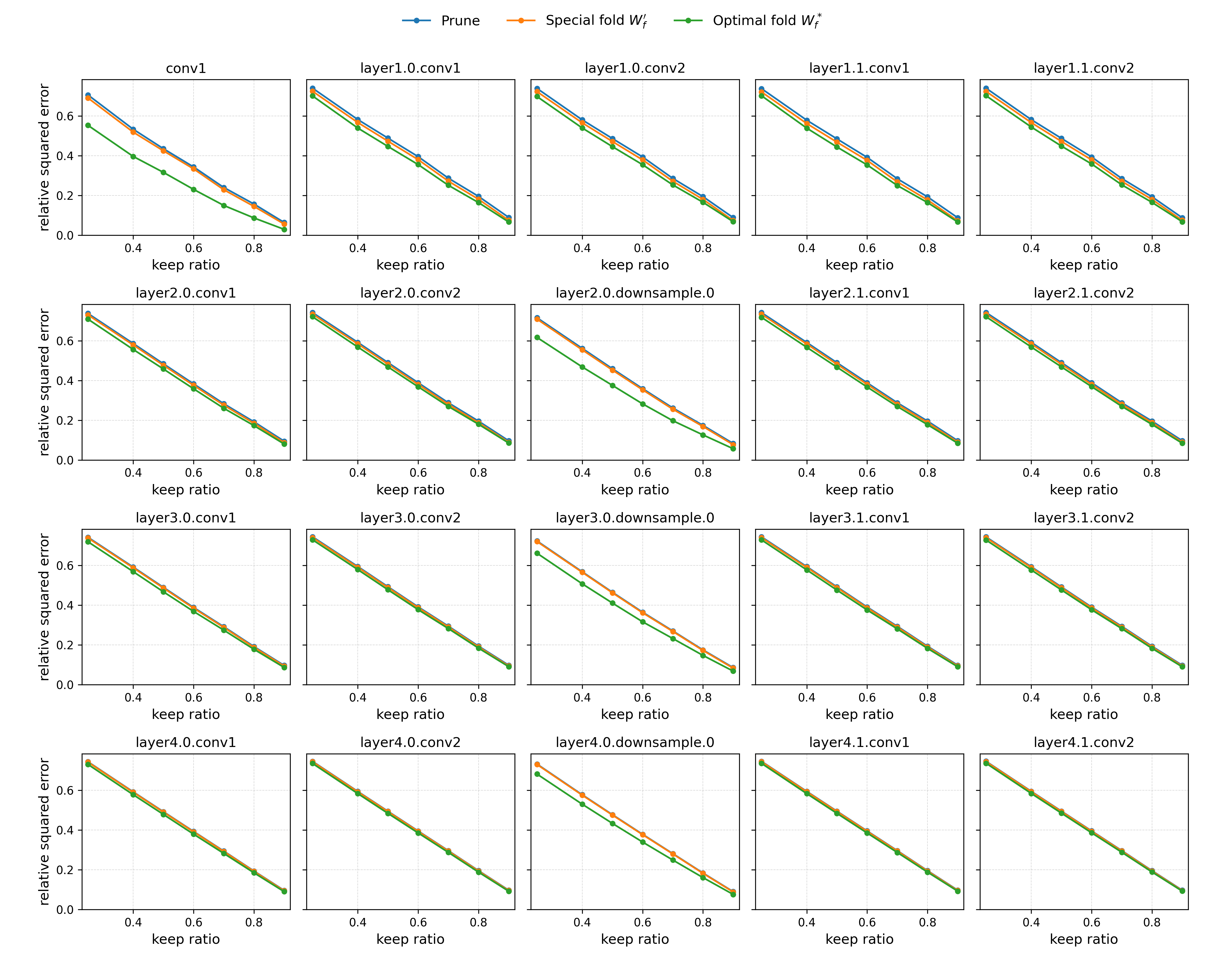}

     \includegraphics[width=.65\textwidth]{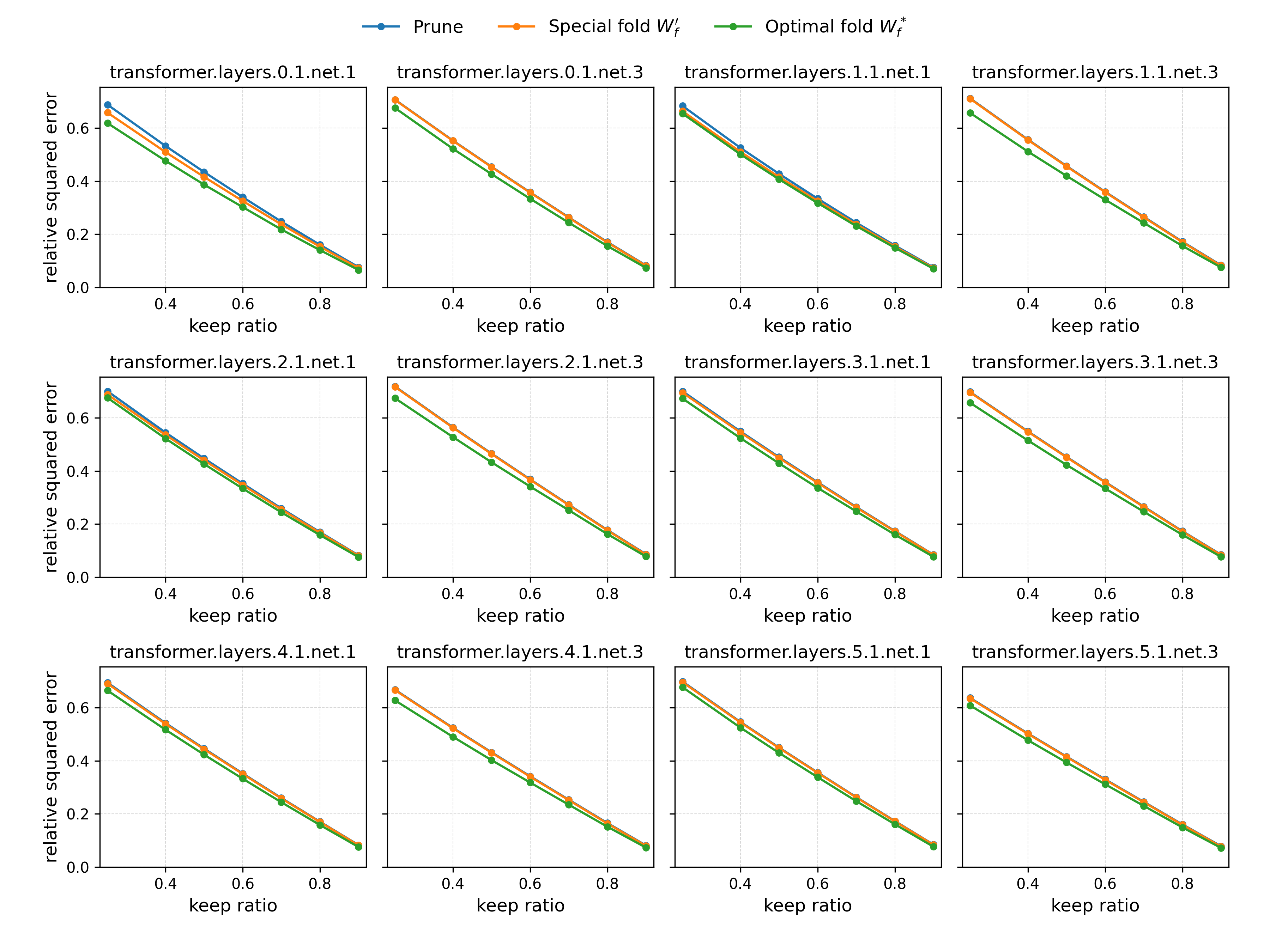}
    \caption{\small{\textbf{Relative reconstruction error of pruning, special folding $\mathbf{W}_f'$ (from the proof of Theorem~\ref{thm:folding-existence}), and optimal folding $\mathbf{W}_f^*$ for all layers of ResNet18 (\textbf{top}) and ViT-B/32 (\textbf{bottom}).}
    For each layer, we report the normalized squared Frobenius error  $\|\mathbf{W} - \mathbf{W}_\bullet\|_F^2 / \|\mathbf{W}\|_F^2$ at several keep ratios $k_p/m$, where $\bullet \in \{\magtwo, \text{singleton}, \fold\}$. \magtwo (blue) denotes structured magnitude pruning with $k_p$ retained rows. The special fold $\mathbf{W}_f'$ (orange) merges all pruned rows into a single extra cluster ($k_f = k_p + 1$). The optimal fold $\mathbf{W}_f^*$ (green) is the $k$-means solution with $k_f$ clusters.
    Across all layers, $\mathbf{W}_f'$ consistently outperforms pruning, and $\mathbf{W}_f^*$ yields the smallest error, empirically validating $\mathrm{error}(\mathbf{W}_p) \ge \mathrm{error}(\mathbf{W}_f') \ge \mathrm{error}(\mathbf{W}_f^*)$.}}
     \label{fig:adam_one_rank_slack:opt}
\end{figure}

Across all layers of ResNet18 and ViT-B/32, the gain from increasing the retained rank by one is consistently negligible, while the gain from replacing pruning with folding is one to two orders of 
magnitude larger (\Figref{fig:adam_one_rank_slack:rank}). Thus, the empirical advantage of folding is not a byproduct of the $k \mapsto k{+}1$ rank slack, but stems from the richer family of cluster-based projections that folding can realize. This is further corroborated by the layer-wise reconstruction errors in \Figref{fig:adam_one_rank_slack:opt}, where 
$\mathrm{error}(\mathbf{W}_p) > \mathrm{error}(\mathbf{W}_f') > \mathrm{error}(\mathbf{W}_f^*)$ holds uniformly. These results confirm that folding’s superior approximation properties are structural rather than an artifact of rank.

The theory controls the loss via the parameter-Lipschitz bound 
$|L(\mathbf{W}) - L(\mathbf{W}_\bullet)| \le \kappa \|\mathbf{W} - \mathbf{W}_\bullet\|_F$.  In regimes where $\kappa$ is moderate, \eg flat solutions obtained with smaller learning rates or SAM-folding’s smaller Frobenius error reliably translates into smaller loss degradation. However, in sharp minima such as those produced by Adam at large learning rates (see \Figref{fig:adam_sharpness_delta_vs_delta} and \Figref{fig:sgd_sharpness_delta_vs_delta}), the effective local $\kappa$ becomes extremely large. In this setting, even tiny parameter perturbations cause large loss changes, and the ordering of 
Frobenius errors no longer predicts the ordering of accuracies. Thus, the discovered failure cases of folding are not contradictions of the theory but instances where the Lipschitz assumption required for loss control breaks down due to extreme curvature.

\section{Use of Large Language Models}
\label{appx:use_llms}
We used ChatGPT~\footnote{ChatGPT / GPT-5: \url{https://chatgpt.com}} for sentence-level grammar correction and improvement, drafting trivial plotting snippets to produce figures from logs, and code readability edits. All ideas, proofs, experiments, and analyses are ours.

\end{document}